\documentclass[12pt]{article}
\usepackage{times}

\usepackage{booktabs} 
\usepackage[shortlabels]{enumitem}
\usepackage{cite}
\usepackage{geometry}
\usepackage{appendix}
\usepackage{relsize}%
\usepackage{subcaption}

 \usepackage{amsmath,amssymb,amsfonts,amsthm}
\usepackage{bbm}

\usepackage{algpseudocode}
\usepackage{algorithm}
\usepackage{custom_macros}
\usepackage{graphicx}

\usepackage{multicol}
\usepackage[bookmarks=true]{hyperref}
\usepackage{xcolor}
\usepackage{tikz}
 \usepackage{tikz-3dplot}
 \usetikzlibrary{shapes.multipart}
 \usetikzlibrary{shapes.symbols}
\usepackage{pgfplots}

\usepackage{standalone}

\usepackage{subcaption}
\usepackage{comment}

\usepackage{etoolbox}
\usepackage{float}
\usepackage[belowskip=-15pt,aboveskip=0pt]{caption}

\usepackage{fancyhdr}
\newlength\FHoffset
\fancyheadoffset{\FHoffset}

\newlength\FHleft
\newlength\FHright

\newbox\FHline
\setbox\FHline=\hbox{\hsize=\paperwidth%
  \hspace*{\FHleft}%
  \rule{\dimexpr\headwidth-\FHleft-\FHright\relax}{\headrulewidth}\hspace*{\FHright}%
}

\allowdisplaybreaks

\title{\textbf{\Large Efficient Randomized Subspace Embeddings for Distributed Optimization under a Communication Budget}}

\author{Rajarshi Saha, Mert Pilanci, and Andrea J. Goldsmith
\thanks{Rajarshi Saha and Mert Pilanci are with the Department of Electrical Engineering, Stanford University, Stanford, CA 94305, USA (Emails: rajsaha@stanford.edu, pilanci@stanford.edu).
Andrea J. Goldsmith is with the Department of Electrical and Computer Engineering, Princeton University, Princeton, NJ 08544, USA. (Email: goldsmith@princeton.edu).
}
\thanks{
This work was partially supported by the Office of Naval Research under grant ONR N00014-18-1-2191, National Science
Foundation under grants DMS-2134248 and ECCS-2037304, Army Research Office Early Career Award W911NF-21-1-0242, Intel Research, Huawei Technologies, Facebook Research, Adobe Research, and Stanford SystemX Alliance.}}
\date{\today}

\begin{document}
\maketitle

\begin{abstract}
\vspace{4mm}

We study first-order optimization algorithms under the constraint that the descent direction is quantized using a pre-specified budget of $R$-bits per dimension, where $R \in (0 ,\infty)$.
We propose computationally efficient optimization algorithms with convergence rates matching the information-theoretic performance lower bounds for: (i) Smooth and Strongly-Convex objectives with access to an Exact Gradient oracle, as well as (ii) General Convex and Non-Smooth objectives with access to a Noisy Subgradient oracle.
The crux of these algorithms is a polynomial complexity source coding scheme that embeds a vector into a random subspace before quantizing it.
These embeddings are such that with high probability, their projection along any of the canonical directions of the transform space is small.
As a consequence, quantizing these embeddings followed by an inverse transform to the original space yields a source coding method with optimal covering efficiency while utilizing just $R$-bits per dimension.
Our algorithms guarantee optimality for arbitrary values of the bit-budget $R$, which includes both the sub-linear budget regime ($R < 1$), as well as the high-budget regime ($R \geq 1$), while requiring $O\left(n^2\right)$ multiplications, where $n$ is the dimension.
We also propose an efficient relaxation of this coding scheme using Hadamard subspaces that requires a near-linear time, i.e., $O\left(n \log n\right)$ additions.
Furthermore, we show that the utility of our proposed embeddings can be extended to significantly improve the performance of gradient sparsification schemes.
Numerical simulations validate our theoretical claims.
Our implementations are available at \href{https://github.com/rajarshisaha95/DistOptConstrComm}{here}.\\


\end{abstract}

\clearpage

\section{Introduction}
\label{sec:introduction}

Distributed optimization algorithms that leverage edge computation of remote devices have proved promising for training large-scale machine learning models \cite{bekkerman_2011, kairouz2019advances}.
To solve an optimization problem
$\minimize_{\xv \in \Xcal \subseteq \Real^n}f(\xv)$ in the parameter-server framework \cite{li_parameter_server_2014}, the \textit{server} maintains an iterate $\xv_t$ at time $t$, which is an estimate of the minimizer $\xv_f^* = \argminimize_{\xv \in \Xcal} f(\xv)$.
We consider a setting in which at every iteration, the worker(s) receives the current iterate $\xv_t$ from the server, computes $\nabla f(\xv_t)$, and communicates information about the computed gradient back to the server, allowing it to take a descent step in that direction.
This process is repeated till the sequence of iterates $\{\xv_t\}_{t = 1, 2, \ldots}$ converges.

Communication overhead is invariably the primary bottleneck of such distributed systems.
When distributed optimization algorithms are implemented over severely communication constrained environments, for instance in wireless networks \cite{mmamiri_2020, saha_jsac_2021}, the bandwidth of the channel (over which the workers send information to the server) is an expensive resource, and is often a pre-specified constraint beyond the algorithm designers' control.
In this work, we resolve the question: \textit{How to design optimal quantization schemes when the worker(s) is constrained to communicate its (sub) gradient information to the server with a strict budget of $R$ bits per dimension?}
We consider the following settings:
\begin{enumerate}[\rm (i)]
    \item When the objective function $f(\xv)$ is $L$-smooth and $\mu$-strongly convex.
    \item When $f(\xv)$ is a general convex function that is not necessarily smooth.
\end{enumerate}

We first consider a \textit{single-worker single-server} setting and extend it later to multiple workers.
Let $\Pi_R$ denote the set of all \textit{optimization protocols} with access to an exact first-order oracle for which the information exchange between the worker and the server at every iteration is limited to $R$-bits per dimension.
For setting (i), it is possible to achieve a linear convergence of the iterates $\{\xv_t\}$ to $\xv_f^*$.
With this in mind, to quantify the performance of any protocol $\pi \in \Pi_R$, consider the asymptotic worst-case linear convergence rate:
\begin{equation}
    \label{eq:worst_case_performance_smooth_strongly_convex}
    C(\pi) = \limsup_{T \to \infty} \sup_{f \in \Fcal_{\mu,L,D}} \left(\frac{\norm{\xv_T(\pi) - \xv_f^*}_2}{D}\right)^{\frac{1}{T}}.
\end{equation}
Here, $\Fcal_{\mu,L,D}$ is the class of $L$-smooth and $\mu$-strongly convex functions which satisfy $\norm{\xv_f^*}_2 \leq D$ ($D \geq 0$), and $\xv_T(\pi)$ denotes the output of the protocol $\pi \in \Pi_R$ after $T$ iterations.
Note that \eqref{eq:worst_case_performance_smooth_strongly_convex} implies $\pi$ achieves a convergence of the iterates $\{\xv_t\}$ to $\xv_f^*$ with a guarantee $\norm{\xv_T(\pi) - \xv_f^*}_2 \lesssim C(\pi)^TD$ for $C(\pi) \leq 1$.
This implies that algorithms in the class $\Pi_R$ require
\[O\Paren{\frac{\log\Paren{D / \epsilon}}{\log\Paren{C(\pi)^{-1}}}}\]
iterations to achieve a suboptimality gap of $\epsilon$.
Since a smaller value of $C(\pi)$ is desirable for faster convergence, we can characterize the set of protocols $\Pi_R$ according to the following \textit{minimax rate}:
\begin{align}
    \label{eq:minimax_defn_smooth}
    C(R) \triangleq \inf_{\pi \in \Pi_R} C(\pi) = \inf_{\pi \in \Pi_R} \limsup_{T \to \infty} \sup_{f \in \Fcal_{\mu,L,D}} \left(\frac{\norm{\xv_T(\pi) - \xv_f^*}_2}{D}\right)^{\frac{1}{T}}.
\end{align}
It has been shown \cite{lin2020achieving} that the information-theoretic lower bound on \eqref{eq:minimax_defn_smooth} can be obtained as $C(R) \geq \max\{\sigma, 2^{-R}\}$ where $\sigma = \frac{L -\mu}{L +\mu}$ is the convergence rate in the absence of any bit-budget constraints.
This provides a theoretical limit to the performance of \textbf{any} protocol in $\Pi_R$.
In this work, we propose an algorithm \textbf{\textsc{DGD-DEF}} that achieves this lower bound to within constant factors while requiring only $O(n^3)$ (or $O(n^2)$, if additional information is available) multiplications.
We also propose a relaxed, near-linear time version that requires only $O\Paren{n \log n}$ additions, saving significantly on the computation requirement while attaining a performance that is a mild $O\Paren{\sqrt{\log n}}$ factor away from the lower bound.
To the best of our knowledge, \textbf{\textsc{DGD-DEF}} is the first polynomial complexity algorithm whose performance matches the lower bound of $C(R) \geq \max\{\sigma, 2^{-R}\}$.

On the other hand for setting (ii), when $f(\xv)$ is a general convex function (not necessarily smooth) and we have access to a noisy subgradient oracle \cite{mayekar_2020}, it is not possible to achieve a linear convergence.
In this case, to measure the performance of any protocol $\pi$, we consider the \textit{expected suboptimality gap},
\begin{equation}
    \label{eq:worst_case_performance_general_convex}
    \Ecal(\pi) = \sup_{(f, \Ocal)} \hspace{1mm} \mathbb{E}[f(\xv_T(\pi))] - f(\xv^*),
\end{equation}
and study how it scales with the number of iterations $T$.
Here, $\xv_T(\pi)$ is the output of protocol $\pi$, and we consider the worst-case performance over all objectives $f: \Xcal \to \Real$ with compact, convex domain $\Xcal \subseteq \Real^n$ satisfying $\sup_{\xv, \yv \in \Xcal}\norm{\xv-\yv}_2 \leq D$, and stochastic subgradient oracles $\Ocal$ whose outputs are uniformly bounded by some parameter $B$.
We consider algorithms $\pi$ from the class $\Pi_{T,R}$ of all optimization protocols that execute at most $T$ iterations with a communication budget of $R$-bits per dimension per iteration.
This set of protocols $\Pi_{T,R}$ can be characterized according to the following \textit{minimax expected suboptimality gap},
\begin{equation}
    \label{eq:minimax_defn_non_smooth}
    \Ecal(T,R) \triangleq \inf_{\pi \in \Pi_{T,R}} \Ecal(\pi) = \inf_{\pi \in \Pi_{T,R}} \sup_{(f, \Ocal)} \hspace{1mm} \mathbb{E}[f(\xv_T(\pi))] - f(\xv^*).
\end{equation}
A lower bound on \eqref{eq:minimax_defn_non_smooth} can be obtained \cite{mayekar_2020} as $\Ecal(T,R) \geq \frac{cDB}{\sqrt{T\cdot\min\{1,R\}}}$.
We propose \textbf{\textsc{DQ-PSGD}} and its relaxed version, that respectively, attain this lower bound to within constant and mild logarithmic factors while requiring $O(n^2) $ multiplications and $O(n \log n)$ additions.

A minimax optimal protocol (or algorithm) $\pi^*$ in either setting requires designing an optimal source coding scheme that quantizes the gradient information efficiently.
A \textbf{source coding scheme} is a pair of mappings $(\Esf,\Dsf)$, where the \textit{encoding} $\Esf: \Real^n \to \{0,1\}^{nR}$ is done by the worker to quantize the information it wants to send to the server.
The \textit{decoding} map $\Dsf:\{0,1\}^{nR} \to \Real^n$ recovers an estimate of the input to the encoder and is implemented at the server.
In this work, we present \textbf{Democratic Source Coding (DSC)}, an efficient \textbf{polynomial-time} \textbf{fixed-length} vector quantization scheme, which when used with suitably designed first-order optimization algorithms, can achieve the respective lower bounds on the minimax rates \eqref{eq:minimax_defn_smooth} and \eqref{eq:minimax_defn_non_smooth} to \textbf{within constant factors}, establishing the minimax optimality of the algorithms.

An alternative way to look at the bit-budget constrained optimization problem is to consider the minimum threshold budget $R_{thr}$ required to attain the convergence rate achievable in the absence of any budget constraint.
For setting (i), the lower bound of $C(R) \geq \max\{\sigma, 2^{-R}\}$ implies that we cannot hope to achieve the convergence rate of unquantized setting if $R < \log\Paren{\frac{1}{\sigma}}$.
Whereas naive quantizers \cite{alistarh_NeurIPS_2017_qsgd} would require $R_{thr} \gtrsim \log\Paren{\frac{\sqrt{n}}{\sigma}}$ bits to attain the unquantized convergence rate, \textbf{\textsc{DGD-DEF}} gets rid of the dimension dependence, and requires just $R_{thr} = O\Paren{\log\Paren{\frac{1}{\sigma}}}$ bits.
Moreover, it achieves this while entailing only a polynomial complexity of $O\Paren{n^2}$ as opposed to Roger's quantizer \cite{rogers_1963} used in \cite{lin2020achieving} that demands exponential complexity.
For setting (ii), the lower bound of $\Ecal(T,R) \geq \frac{cDB}{\sqrt{T\cdot \min\{1, R\}}}$ implies that for the sub-linear budget regime, i.e., when $R < 1$, we can expect the suboptimality gap to scale as $\frac{1}{\sqrt{T}}$ as long as we have a constant bit-budget, i.e., $R_{thr} = O(1)$.
\textbf{\textsc{DQ-PSGD}} ensures this while requiring only $R_{thr} = O(1)$ bits, which establishes its optimality, as opposed to $R_{thr} = O\Paren{\sqrt{n}}$ for naive quantizers or $R_{thr} = O\Paren{\log \log n}$ for RATQ \cite{mayekar_2020}.

\subsection{Our Contributions}
\label{subsec:our_contributions}

In this work, we consider algorithms that attain the information-theoretic lower bounds to the minimax performance metrics of bit-budget constrained optimization.
Existing works \cite{lin2020achieving, mayekar_2020} have characterized the precise lower bounds to \eqref{eq:minimax_defn_smooth}, \eqref{eq:minimax_defn_non_smooth}, and we provide optimal algorithms that achieve these minimax lower bounds to within constant factors while requiring $O(n^2)$ computation.
Our contributions are as follows:
\begin{enumerate}[(i)]
    \item We first propose \textbf{Democratic Source Coding (DSC)} which use Kashin embeddings \cite{lyubarskii_2010} to compress a vector in $\Real^n$ subject to a constraint of $R$ bits per dimension.
    \textbf{DSC} is a polynomial-time source coding scheme and its error is independent of the dimension $n$; a crucial property for efficiently compressing high-dimensional vectors.
    
    \item For strongly convex smooth objectives, we propose \textbf{\textsc{DGD-DEF}}: \textbf{D}istributed \textbf{G}radient \textbf{D}escent with \textbf{D}emocratically \textbf{E}ncoded \textbf{F}eedback, an algorithm that uses \textbf{DSC} to quantize the feedback-corrected gradients and show that it achieves the lower bound on \eqref{eq:minimax_defn_smooth}.
    
    \item For general convex non-smooth objectives, we propose \textbf{\textsc{DQ-PSGD}}: \textbf{D}emocratically \textbf{Q}uantized \textbf{P}rojected \textbf{S}tochastic sub\textbf{G}radient \textbf{D}escent that achieves the lower bound on \eqref{eq:minimax_defn_non_smooth}.
    
    \item Since even the $O(n^2)$ complexity of \textbf{\textsc{DSC}} can be computationally demanding for large $n$, we further propose a computationally simpler relaxation, referred to as \textbf{\textsc{NDSC}}: \textbf{N}ear \textbf{D}emocratic \textbf{S}ource \textbf{C}oding, which achieves optimality to within a mild logarithmic factor.
    We observe that in simulations, \textbf{\textsc{NDSC}} performs at par with \textbf{\textsc{DSC}}.
    
    \item Finally, in \S \ref{subsec:extension_to_multiple_workers}, we show how our algorithms can be extended to multi-worker setups.
    We also show that \textbf{DSC} or \textbf{NDSC} consistently improve the performance when used in conjunction with other existing compression strategies (\S \ref{sec:numerical_simulations} and Appendix \ref{app:extension_to_general_compression_schemes}).
\end{enumerate}

\subsection{Significance and Related work}
\label{subsec:significance_and_related_work}

\textbf{Communication-Constrained Distributed Optimization.}
Much work has been done in recent years to address the communication bottleneck in distributed optimization.
\textit{Variable-length} coding schemes were proposed in \cite{alistarh_NeurIPS_2017_qsgd}.
The bit-requirement of these quantization schemes are optimal in expectation, but their worst-case performance is not.
Our work considers \textit{fixed-length} quantizers for the setting where precision constraints are imposed as a pre-specified bit-budget of $R$-bits that needs to be strictly respected even for worst case inputs.
The problem of distributed optimization under bit-budget constraints is considered in \cite{lin2020achieving, mayekar_2020}.
\cite{lin2020achieving} considers smooth and strongly convex objectives, and derive a lower bound on the minimax convergence rate defined in \eqref{eq:minimax_defn_smooth}, along with a matching upper bound.
However, their upper bounding algorithm has exponential complexity and hence, practically infeasible; whereas, our proposed algorithm \textbf{\textbf{\textsc{DGD-DEF}}}, which uses \textbf{DSC} for quantization has polynomial complexity and achieves the minimax lower bound to within constant factors.
For the setting of general convex and non-smooth objectives, \cite{mayekar_2020} provides a lower bound to the minimax suboptimality gap defined in \eqref{eq:minimax_defn_non_smooth}.
Using their proposed quantizer \textit{RATQ}, they also give an upper bound which characterizes the minimum bit-budget required to attain this minimax optimal lower bound to within an iterated logarithmic factor in $d$. 
Compared to this, \textbf{\textsc{DQ-PSGD}} uses $R + o_n(1)$ bits per dimension, and attains a suboptimality gap within constant factors of the minimax lower bound.
Here, $R$ is specified as a constraint and is beyond the algorithm designer's control and $o_n(1)$ is a term that goes to zero as $n \to \infty$.
A fixed length nearly optimal coding scheme that employs \textit{random rotations} was used in \cite{suresh_distributed_mean_estimation, abdi_fakri_aaai}.
Orthonormal transforms for random rotations were also used in \cite{hadad_erez_2016_IEEE_TSP}.
However, their goal was to design quantizers that achieve low statistical correlation between signal and quantization error rather than minimizing the $\ell_2$ quantization error, which is more relevant for quantizing gradients in distributed optimization, when the distribution of quantizer input is not known.
In our work, we also propose a computationally simpler relaxation of \textbf{DSC}, namely \textbf{NDSC} that achieves the minimax lower bounds to within a logarithmic factor.
We note that our proposed \textit{near-democratic embeddings} boil down to \textit{random rotations} when square orthonormal transforms are used, i.e., \textbf{NDSC} is a generalization of random rotations.
When Hadamard transforms are considered, these works assume that the dimension $n$ is such that a Hadamard matrix can be constructed.
However, it might not necessarily be true, and naive heuristics like partitioning the vector or zero-padding in order to make the dimension equal to the nearest power of $2$ can be suboptimal.
\textbf{NDSC} performs better than random rotations in such cases.
Another popular strategy to reduce the communication requirement is \textit{gradient sparsification} that reduces the dimension of the vector being exchanged.
Our coding strategies, \textbf{DSC} and \textbf{NDSC} can be used in conjunction with these sparsification methods.
We provide a comparison of our work with existing quantization and sparsification strategies in Table \ref{tab:compression_schemes_comparison}.

\begin{table*}[t]
\begin{center}
\begin{small}
\begin{sc}
\renewcommand{\arraystretch}{1.1}
\begin{tabular}{|l|c|c|c|r|}
\toprule
Compression Scheme & No. of Bits  & Error & Complexity \\
\midrule
Sign quantization \cite{bernstein_icml_2018, karimireddy_2019_error_feedback} & $O(n)$ & $O(n)$ & $O(n)$ \\
QSGD \cite{alistarh_NeurIPS_2017_qsgd} & $O(2^R(2^R + \sqrt{n}))$ & $\min \hspace{-1mm}\left\{\hspace{-1mm}\sqrt{n}2^{-R}\hspace{-1mm}, n 2^{-2R}\right\}$ & $O(n)$\\
Ternary quantization \cite{terngrad_neurips_2017} & $O(n\log_23)$ & $O(n)$ & $O(n)$ \\
vqSGD-Gaussian \cite{gandikota2020vqsgd} & $O\left(c\right), c > \log n$ & $O\left(\frac{n}{c}\right)$ & $O(\exp(c))$\\
vqSGD-Cross Polytope \cite{gandikota2020vqsgd} & $O\left(\log n\right)$ & $O\left(n\right)$ & $O(n)$ \\
Top-$k$ sparsification \cite{stich_2018_sparsifiedSGD} & $O\Paren{k + \log_2\binom{n}{k}}$ & $(n-k)/n$ & $O(k\hspace{-1mm}+\hspace{-1mm}(n\hspace{-1mm}-\hspace{-1mm}k)\log_2k)$ \\
Random sparsification \cite{wangni_2018} & $O\Paren{k + \log_2\binom{n}{k}}$ & $O\left(n/k\right)$ & $O(n)$ \\
Sim-Q+ \cite{mayekar_SimQ+} & $O(3n)$ & $O(1)$ & $O(n^2)$ \\
\textbf{DSC} (\textbf{\textit{Ours}}) & $nR + O(1)$ & $O\left(2^{-2R/\lambda}\right)$ & $O(n^2)$\\
\textbf{NDSC} (\textbf{\textit{Ours}}) & $nR + O(1)$ & $O(2^{-2R/\lambda}\log n)$ & $O(n\log n)$ \\
\bottomrule
\end{tabular}
\end{sc}
\end{small}
\end{center}
\caption{\textsc{Comparison of various compression schemes}}
\label{tab:compression_schemes_comparison}
\end{table*}

\textbf{Kashin Embeddings and Random Matrix Theory.}
\textit{Kashin embeddings} were studied in the random matrix theory literature \cite{lyubarskii_2010, Ka_in_1977} for their relation to convex geometry and vector quantization.
From a high level perspective, \textit{Kashin embedding} of a vector $\yv \in \Real^n$, is a vector $\xv \in \Real^N$ $(N \geq n)$, which has the property that the components of $\xv$ are similar to each other in magnitude, i.e., for all $i \in [N]$, $|x_i| = \Theta(1/\sqrt{N})$ with high probability (w.h.p.).
Even if components of $\yv$ may be arbitrarily varying in magnitude, Kashin embeddings have an effect of evenly distributing this variation across different components of $\xv$.
Subsequently, applying lossy compression schemes (eg. quantization) to the democratic embedding $\xv$ instead of the original vector $\yv$ incurs less error.
The inverse embedding map $\Sv: \Real^N \to \Real^n$ is linear, i.e., $\yv = \Sv\xv$.
Usually, $\Sv$ is randomly generated and the properties of the democratic embeddings are very closely related to \textit{Restricted Isometry Property (RIP)} parameters of $\Sv$ \cite{candes_tao_RIP}.
We study different classes of random matrices (subgaussian, orthonormal, and Hadamard) and the pros and cons of using them for constructing respective \textbf{DSC} schemes.
The efficacy of Kashin embeddings for various learning problems have been studied in \cite{chen_2020_breaking_trilemma, caldas2019expanding, safaryan2021uncertainty, saha_2022_model_compression}.
In our work, we go even further in using them for designing general source coding schemes (both stochastic and deterministic) and show that they yield minimax optimal optimization algorithms.
Kashin embedding of a vector is not unique, and \cite{lyubarskii_2010} proposed an iterative-projection type algorithm to compute a Kashin embedding.
However, their algorithm requires explicit knowledge of RIP parameters of $\Sv$, which is not readily available.
To this end, \cite{studer2015democratic} introduced the notion of \textbf{democratic embeddings (DE)}.
\textbf{DE} of a vector $\yv \in \Real^n$ is a Kashin embedding too, and is obtained by solving a linear program.
We propose a simple relaxation of this linear program, and show that its solution yields a \textbf{near-democratic embedding}.

\section{Democratic Embeddings}
\label{sec:democratic_embeddings}

Consider a wide matrix $\Sv \in \Real^{n \times N}$, where $n \leq N$.
For any given vector $\yv \in \Real^n$, the system of equations $\yv = \Sv \xv$ is under-determined in $\xv\in \Real^N$, with the set $\Scal = \{ \xv \in \Real^N \hspace{1mm} | \hspace{1mm} \yv = \Sv\xv\}$ as the solution space.
The vector $\xv^* \in \Scal$ which has the minimum $\ell_{\infty}$-norm in this solution space is referred to as the \textbf{Democratic Embedding} of $\yv$ with respect to $\Sv$.
In other words, $\xv^*$ is obtained by solving,
\begin{equation}
\label{eq:l_inf_minimization_problem}
    \minimize_{\xv \in \Real^N} \norm{\xv}_{\infty} \hspace{2mm} \text{subject to} \hspace{2mm} \yv = \Sv\xv.
\end{equation}
The constraint set $\Scal$ can be relaxed to a larger set $\Scal' = \{\xv \in \Real^n \hspace{1mm} | \hspace{1mm} \norm{\yv - \Sv\xv}_2 \leq \epsilon\}$ as in \cite{studer2015democratic}. In the rest of our work, we consider $\epsilon = 0$, i.e., exact representations.
In order to characterize the solution of \eqref{eq:l_inf_minimization_problem} (cf. Lemma \ref{lem:democratic_embedding_property}), we review certain definitions from \cite{lyubarskii_2010, studer2015democratic}.

\begin{definition} \textbf{(Frame)}
\label{def:frame}
A matrix $\Sv \in \Real^{n \times N}$ with $n \leq N$ is called a \textbf{frame} if $A\norm{\yv}_2^2 \leq \norm{\Sv^\top\yv}_2^2 \leq B\norm{\yv}_2^2$ holds for any vector $\yv \in \Real^n$ with $0 < A \leq B < \infty$, where $A$ and $B$ are called \textbf{lower} and \textbf{upper frame bounds} respectively.
\end{definition}

\begin{definition}\textbf{(Uncertainty principle (UP))}
\label{def:uncertainty_principle}
A frame $\Sv \in \Real^{n \times N}$ satisfies the Uncertainty Principle with parameters $\eta, \delta$, with $\eta > 0, \delta \in (0,1)$ if $\norm{\Sv\xv}_2 \leq \eta\norm{\xv}_2$ holds for all (sparse) vectors $\xv \in \Real^N$ satisfying $\norm{\xv}_0 \leq \delta N$, where $\norm{\xv}_0$ denotes the number of non-zero elements in $\xv$.
\end{definition}

\begin{lemma}\textbf{[Democratic embeddings]} \cite{studer2015democratic}
\label{lem:democratic_embedding_property}
Let $\Sv \in \Real^{n \times N}$ be a frame with bounds $A,B$ (cf. Def. \ref{def:frame}) that satisfies  the uncertainty principle (UP) (cf. Def. \ref{def:uncertainty_principle}) with parameters $\eta, \delta$ such that $A > \eta\sqrt{B}$.
Then for any vector $\yv \in \Real^n$, the solution $\xv_d$ of \eqref{eq:l_inf_minimization_problem} satisfies
\begin{equation}
    \frac{K_l}{\sqrt{N}}\norm{\yv}_2 \leq \norm{\xv_d}_{\infty} \leq \frac{K_u}{\sqrt{N}}\norm{\yv}_2,
\end{equation}
where $ K_l = \frac{1}{\sqrt{B}}$ and $K_u = \frac{\eta}{\left(A-\eta\sqrt{B}\right)\sqrt{\delta}}$ are called \textbf{lower} and \textbf{upper Kashin constants} respectively.
\end{lemma}

We are interested in \textbf{Parseval frames} ($A = B = 1$), i.e., they satisfy $\Sv\Sv^{\top} = \Iv_n$ (where $\Iv_n \in \Real^{n \times n}$ is the identity matrix), implying $K_l = 1$ and $K_u = \eta(1-\eta)^{-1}\delta^{-1/2}$.
$K_u$ depends only on the choice of $\Sv$ and nothing else.
Lemma \ref{lem:democratic_embedding_property} shows that none of the coordinates of the democratic embedding is too large, and the information content of $\yv$ is distributed evenly.

The value of upper Kashin constant $K_u$ depends on the choice of frame construction $\Sv$, as well as its aspect ratio $\lambda = N/n$.
\cite{lyubarskii_2010, studer2015democratic} show that if $\Sv$ is a \textit{random Haar orthonormal matrix}, then $K = K\left(\lambda\right)$, where $\lambda > 1$ can be arbitrarily close to $1$.
Such frames can be obtained by generating a random $N \times N$ orthonormal matrix sampled from the Haar distribution, and randomly selecting $n$ of its rows.
Since choosing $\lambda$ is up to us, Lemma \ref{lem:democratic_embedding_property} implies that for random orthonormal frames, democratic embeddings satisfy $\norm{\xv_d}_{\infty} = \Theta(1/\sqrt{N})$ w.h.p.
As we will see in \S \ref{sec:democratic_source_coding}, for large $n$ (or equivalently large $N$ since $N \geq n$), this remarkably improves the robustness of our proposed compression schemes.
A comprehensive comparison of different choices for $\Sv$ is given in Appendix \ref{app:frames_comparison}.

\subsection{Near-Democratic Embeddings}
\label{subsec:near_democratic_embeddings}

Although the linear program \eqref{eq:l_inf_minimization_problem} can be solved with $O(n^3)$ multiplications using \textit{simplex} or \textit{Newton's method}, it can still be computationally intensive.
A \textit{projected gradient descent type} algorithm with $O(n^2)$ complexity was presented in \cite{lyubarskii_2010}, but implementing it requires explicit knowledge of $\eta, \delta$ which is not readily available.
We propose a simpler relaxation of \eqref{eq:l_inf_minimization_problem}:
\begin{equation}
\label{eq:l2_minimization_problem}
    \minimize_{\xv \in \Real^N} \norm{\xv}_2^2 \hspace{2mm} \text{subject to} \hspace{2mm} \yv = \Sv \xv.
\end{equation}

The solution of the $\ell_2$-minimization \eqref{eq:l2_minimization_problem} can be found in closed form (ref. Appendix \ref{app:near_democratic_closed_form}) as:
\begin{equation}
    \label{eq:near_democratic_closed_form}
    \xv_{nd} = \Sv^{\dagger}\yv = \Sv^{\top}\left(\Sv\Sv^{\top}\right)^{-1}\yv \in \Real^n\,,
\end{equation}
where $(\cdot)^{\dagger}$ (defined as above) is the pseudo-inverse.
For Parseval frames $\Sv$, this boils down to $\xv_{nd} = \Sv^{\top}\yv$.
We refer to $\xv_{nd} = \Sv^{\dagger}\yv$ as the \textbf{Near-Democratic} embedding of $\yv \in \Real^n$ with respect to $\Sv \in \Real^{n \times N}$, and show that the solution $\xv_{nd}$ of \eqref{eq:l2_minimization_problem} satisfies $\norm{\xv_{nd}}_{\infty} = O((\sqrt{\log N}/\sqrt{N})\norm{\yv}_2)$ w.h.p.
The additional $\sqrt{\log N}$ factor instead of the constant $K_u$ is a very modest price to pay compared to the computational savings, even for dimensions as large as $N \sim 10^6$.
Note that as $\lambda$ approaches $1$, the solution space $\Scal$ of $\eqref{eq:l_inf_minimization_problem}$ becomes smaller, and for $\lambda = 1$, the solutions of \eqref{eq:l_inf_minimization_problem} and \eqref{eq:l2_minimization_problem} coincide.
Lemma \ref{lemma:near_democratic_dynamic_range_random_orthonormal} characterizes our result explicitly.
A random orthonormal matrix is obtained by generating random Gaussian matrix $\Gv \in \Real^{N \times N}$ with i.i.d. entries, $\Gv_{ij} \stackrel{iid}{\sim} \Ncal(0,1)$, computing its singular-value decomposition $\Gv = \Uv \mathbf{\Sigma} \Vv^{\top}$, letting $\Stv = \Uv\Vv^{\top}$ and generating $\Sv \in \Real^{n \times N}$ by randomly selecting $n$ rows of $\Stv$, i.e., $\Sv = \Pv \Stv$ where $\Pv \in \Real^{n \times N}$ is a sampling matrix obtained by randomly selecting $n$ rows of $\Iv_N$.

\begin{lemma} \textbf{(Near-Democratic Embeddings with Random Orthonormal Frames)}
\label{lemma:near_democratic_dynamic_range_random_orthonormal}
For a random orthonormal frame $\Sv \in \Real^{n \times N}$ generated as described above, with probability (w.p.) at least $1 - \frac{1}{2N}$, the solution of \eqref{eq:l2_minimization_problem} satisfies:
\begin{equation}
    \norm{\xv_{nd}}_{\infty} \leq 2\sqrt{\frac{\lambda \log(2N)}{N}}\norm{\yv}_2.
\end{equation}
\end{lemma}
The proof of Lemma \ref{lemma:near_democratic_dynamic_range_random_orthonormal} is delegated to App. \ref{app:proof_NDE_Random_Orthonormal}.
It utilizes the observation that each coordinate of $\Sv^\top \yv \in \Real^N$ is isotropically distributed, and subsequently exploits measure concentration.
Random orthonormal matrices prove quite beneficial in this regard. 
Nevertheless, computing the near-democratic embeddings $\xv_{nd} = \Sv^{\top}\yv$, for random orthonormal frames still requires $O\left(n^2\right)$ time, and moreover, even storing $\Sv$, comprising of $32$-bit floating-point entries can be memory intensive.
To address this, we further propose a randomized Hadamard construction for $\Sv$.
Storing a randomized Hadamard matrix amounts to only storing the signs, and near-democratic embeddings using such matrices can be computed in near-linear time. 
Consider the $N \times N$ Hadamard matrix $\Hv$ whose entries are normalized, i.e., $\Hv_{ij} \hspace{-1mm} = \hspace{-1mm} \pm 1/\sqrt{N}$, $\Hv \hspace{-1mm} = \hspace{-1mm} \Hv^{\top}$, and $\Hv \Hv^{\top}\hspace{-2mm} = \hspace{-1mm} \Iv_N$.
Let $\Dv \in \Real^{N \times N}$ be a diagonal matrix whose entries are randomly chosen to be $\pm 1$ with equal probability.
Let $\Pv \in \Real^{n \times N}$ be the sampling matrix as before.
We define our frame to be $\Sv = \Pv\Dv\Hv \in \Real^{n \times N}$.
Note that $\Sv \Sv^{\top} = \Pv \Dv \Hv \Hv^{\top}\Dv\Pv^{\top} = \Pv \Dv^2 \Pv^{\top} = \Pv \Pv^{\top} = \Iv_n$.
i.e.,, our randomized Hadamard construction is a Parseval frame.
Storing the $1$-bit signs is enough to store the matrix $\Sv = \Pv \Dv \Hv$ in the memory.
Furthermore, the near-democratic embedding $\xv_{nd} = \Sv^{\top}\yv = \Hv \Dv \Pv^{\top}\yv$ can be computed with just $O(n \log n)$ additions, subtractions and sign-flips as $\Sv_{ij} = \pm 1/\sqrt{N}$. 
Unlike random orthonormal matrices, it does not require any explicit floating-point multiplications. 
Lemma \ref{lemma:near_democratic_dynamic_range_randomized_Hadamard} characterizes the $\norm{\cdot}_{\infty}$ of the solution of \eqref{eq:l2_minimization_problem} with $\Sv = \Pv \Dv \Hv$.

\begin{lemma}\textbf{(Near-Democratic Embeddings with Randomized Hadamard Frames)}
\label{lemma:near_democratic_dynamic_range_randomized_Hadamard}
For $\Sv = \Pv \Dv \Hv$ $\in \Real^{n \times N}$, with $\Pv, \Dv, \Hv$ defined as above, with probability at least $1 - \frac{1}{2N}$, the solution of \eqref{eq:l2_minimization_problem} satisfies:
\begin{equation}
    \norm{\xv_{nd}}_{\infty} \leq 2\sqrt{\frac{\log\left(2N\right)}{N}}\norm{\yv}_2.
\end{equation}
\end{lemma}
Its proof is provided in App. \ref{app:proof_NDE_Randomized_Hadamard}, and upper bounds the tail probability of each coordinate of $\Sv^{\top}\yv$ using a Chernoff-type argument, followed by a union bound.
In \S \ref{sec:democratic_source_coding}, we employ our democratic and near-democratic embeddings for source coding and show that they respectively yield efficient optimal and near-optimal vector quantizers.

\section{Democratic Source Coding}
\label{sec:democratic_source_coding}

We introduce our proposed random-embedding based quantization algorithms in \S \ref{subsec:proposed_quantization_strategy} and derive upper bounds on the $\ell_2$ quantization errors, which are relevant for the convergence analysis of our proposed algorithms in \S \ref{sec:proposed_optimization_algorithms}.
Furthermore, in \S \ref{subsec:optimal_covering_efficiency_of_DSC}, we discuss covering efficiency, which is an alternative notion of characterizing the efficiency of vector quantizers.

We first start with the definition of uniform scalar quantizer.
Denote the $\ell_{\infty}$-ball of radius $r$ centered at the origin of $\Real^N$ by $\Bcal_{\infty}^N(r)$.
Finite length source coding schemes map its inputs to a discrete set of finite cardinality.
An $R$-\textbf{bit uniform scalar quantizer} is a mapping $\Qsf(\cdot):\Bcal_{\infty}^N(1) \to S$ with $S \subset \Real^N$ and $|S| \leq 2^{\lfloor nR \rfloor}$.
With a bit-budget of $R$-bits per dimension, the $M = 2^R$ quantization points $\{v_i\}_{i=1}^{M}$ along any dimension are given by $v_i = -1 + (2i - 1)\Delta/2$, for $i = 1, \ldots, M$, where $\Delta = 2/M$ is the \textbf{resolution}.
$\Qsf(\xv)$ for $\xv \in \Bcal_{\infty}^{N}(1)$ is defined as $\Qsf(\xv) = \left[ x_1', \ldots, x_N' \right]^{\top}\hspace{-1mm}; \hspace{2mm} x_j' \triangleq \argminimize_{y \in \{v_1, \ldots, v_M\}}|y - x_j|$.
The maximum possible quantization error is given by,
\begin{equation}
\label{eq:uniform_quantizer_error}
    d = \sup_{\xv \in \Bcal_{\infty}^{N}(1)}\norm{\xv - \Qsf(\xv)}_2 \leq \frac{\Delta}{2}\sqrt{N}.
\end{equation}

\subsection{Proposed Quantization Strategy}
\label{subsec:proposed_quantization_strategy}

Given a frame $\Sv \in \Real^{n \times N}$, for any $\yv \in \Real^n$, denote its \textit{democratic} and \textit{near-democratic embeddings} (i.e., the solutions of \eqref{eq:l_inf_minimization_problem} and \eqref{eq:l2_minimization_problem} respectively) by $\xv_d$ and $\xv_{nd}$, both in $\Real^N$.
The \textbf{democratic} and \textbf{near-democratic encoders} are mappings $\Esf_d(\cdot), \Esf_{nd}(\cdot):\Real^n \to S \subset \Real^N$, $|S| \leq 2^{\lfloor nR \rfloor}$ defined as:
\begin{equation}
    \Esf_d(\yv) = \Qsf\left(\frac{\xv_d}{\norm{\xv_d}_{\infty}}\right), \hspace{1mm} \Esf_{nd}(\yv) = \Qsf\left(\frac{\xv_{nd}}{\norm{\xv_{nd}}_{\infty}}\right).
\end{equation}

$\Esf_d(\cdot)$ and $\Esf_{nd}(\cdot)$ are quantized outputs and sent over the channel from \textit{source} (worker) to the \textit{destination} (parameter server).
The corresponding \textbf{decoder} is the same for both, and is defined as the mapping $\Dsf(\cdot):S \to \Real^n$, $\Dsf(\xv') = \norm{\xv}_{\infty}\Sv \xv'$, where $\xv$ is either $\xv_d$ or $\xv_{nd}$, and $\xv' = \Esf_d(\yv)$ or $\Esf_{nd}(\yv)$. 
Normalization by $\norm{\xv}_{\infty}$ is needed to ensure that the input to $\Qsf(\cdot)$ lies in $\Bcal_{\infty}^{N}(1)$.
In the following Thm. \ref{prop:quantization_error_DSC_and_NDSC} we show the independence/weak-logarithmic dependence of \textbf{DSC} and \textbf{NDSC}.
For simplicity of exposition, here we have assumed that the scalar magnitude $\norm{\xv}_{\infty}$ is known exactly at the receiver.
We can quantize $\norm{\xv}_{\infty}$ using a constant number of bits.
In that case, the total number of bits required to quantize the vector is $nR + O(1)$, which implies $R + \frac{O(1)}{n}$ bits per dimension, which $\to R$ as $n \to \infty$.
In App. \ref{app:quantizing_the_linf_norm}, we show that this just introduces a small additive constant quantization error and all the results still hold true.

\begin{theorem}\textbf{(Quantization error of \textbf{DSC} and \textbf{NDSC})}
\label{prop:quantization_error_DSC_and_NDSC}
Given $\Sv \in \Real^{n \times N}$ and an $R$-bit uniform scalar quantizer $\Qsf(\cdot)$, for any $\yv \in \Real^n$, let $\Qsf_d(\yv) = \Dsf\left(\Esf_d(\yv)\right)$ and $\Qsf_{nd}(\yv) = \Dsf\left(\Esf_{nd}(\yv)\right)$. Then, with probability at least $1 - e^{-\Omega(n)}$,
\begin{align}
    \norm{\yv - \Qsf_d\left(\yv\right)}_2 \leq 2^{\left(1 - \frac{R}{\lambda}\right)}K_u\norm{\yv}_2,
\end{align}
and with probability at least $1 - 1/\Omega(n)$\footnote{The exact expression depends on the choice of $\Sv$ (with its UP parameters) and is given in Appendix \ref{app:frames_comparison}},
\begin{align}
    \norm{\yv - \Qsf_{nd}\left(\yv\right)}_2 \leq 2^{\left(2 - \frac{R}{\lambda}\right)}\sqrt{\log(2N)}\norm{\yv}_2.
\end{align}
\end{theorem}
The proof of Thm. \ref{prop:quantization_error_DSC_and_NDSC} is a direct consequence of Lemmas \ref{lem:democratic_embedding_property}, \ref{lemma:near_democratic_dynamic_range_random_orthonormal} and \ref{lemma:near_democratic_dynamic_range_randomized_Hadamard} and is provided in App. \ref{app:proof_of_quantization_error_DSC_and_NDSC}.
In the above Thm. \ref{prop:quantization_error_DSC_and_NDSC} we consider a randomized Hadamard frame for near-democratic representation (i.e., Lemma \ref{lemma:near_democratic_dynamic_range_randomized_Hadamard}). 
For random orthonormal frames, a $\sqrt{\lambda\log(2N)}$ factor appears instead of $\sqrt{\log(2N)}$.
For $\lambda = 1$, Thm. \ref{prop:quantization_error_DSC_and_NDSC} holds for both classes of frames.
Choosing $\lambda = 1$ is possible for random orthonormal frames, but not in the case of Hadamard frames for which the dimension $N$ must be such that Hadamard matrix can be constructed.
Democratic and near-democratic embeddings provide a unified way of looking at basis transforms for quantization, and can be applied with any general compression scheme.

\subsection{Optimal Covering Efficiency of Democratic Source Coding}
\label{subsec:optimal_covering_efficiency_of_DSC}

The notion of \textbf{covering efficiency} is a measure of how close a fixed-length quantizer is to being optimal.
Quantizer efficiency is related to how effectively a Euclidean ball of unit radius can be covered with a finite number of smaller balls \cite{dumer_2007, gray_quantization, wyner_packing}.
We review certain definitions to precisely characterize this notion.
Let $\Bcal_2^n(a)$ denote the Euclidean ball of radius $a$ centered at the origin.
The \textbf{dynamic range} (r) of an $R$-bit quantizer $\Qsf: \Rcal \to \Rcal' \subset \Real^n$, $|\Rcal'| \leq 2^{\lfloor nR \rfloor}$ is defined to be the radius of the largest Euclidean ball which fits inside the domain of $\Qsf$, i.e., $r \triangleq \sup\{a\hspace{1mm}|\hspace{1mm} \Bcal_2^n(a) \subseteq \Rcal\}$.
The \textbf{covering radius} of $\Qsf$ is defined as the maximum possible quantization error when any $\xv \in \Bcal_2^n(r)$ is quantized to its nearest neighbor, i.e., $d(\Qsf) \triangleq \inf\{d > 0 \hspace{1mm}|\hspace{1mm} \forall \xv \hspace{1mm} \in \Bcal_2^n(r), \norm{\xv - \Qsf(\xv)}_2 \leq d\}$.
The \textbf{covering efficiency} $\rho_n(\Qsf)$ of $\Qsf:\Rcal \to \Rcal'\subset \Real^n$ is defined as:
\begin{equation}
\label{def:covering_efficiency}
    \rho_n(\Qsf) = \left(|\Rcal'|\frac{\text{vol}\left(\Bcal_2^n(d(\Qsf)\right)}{\text{vol}\left(\Bcal_2^n(r)\right)}\right)^{\frac{1}{n}} = |\Rcal'|^{\frac{1}{n}}\frac{d(\Qsf)}{r}.
\end{equation}
If we consider Euclidean balls of radius $d(\Qsf)$ around each quantization point, the total volume of these balls must cover $\Bcal_2^n(r)$.
Covering efficiency formalizes how well this covering is and $\rho_n(\Qsf) \geq 1$ is a natural lower bound.
\cite{lin2020achieving} notes that for Roger's quantizer \cite{rogers_1963}, $\rho_n \to 1$ as $n \to \infty$, and is hence asymptotically optimal. 
However, it is practically infeasible for large $n$ as it cannot be implemented in polynomial time.
Popular quantization schemes \cite{alistarh_NeurIPS_2017_qsgd} use uniform scalar quantizers that have $\rho_n = \sqrt{n}$, which grows significantly far away from the lower bound of $1$ for large $n$ and are quite suboptimal.
The following Lemma \ref{lem:covering_efficiency_near_optimal_source_coding} quantifies the efficiency of our proposed quantization scheme.
Proof is a direct consequence of Thm. \ref{prop:quantization_error_DSC_and_NDSC} and is given in Appendix \ref{app:proof_covering_efficiency}.

\begin{lemma}(\textbf{Covering Efficiency of (Near) Democratic Source Coding})
\label{lem:covering_efficiency_near_optimal_source_coding}
For the (near) democratic source coding schemes described in \S \ref{subsec:proposed_quantization_strategy}, with probability at least $1 - \frac{1}{2N}$, the covering efficiencies are given by
\begin{align*}
    \rho_d = 2^{1 + R\left(1 - \frac{1}{\lambda}\right)}K_u, \hspace{1mm} \text{and} \hspace{1mm} \rho_{nd} = 2^{2+R\left(1 - \frac{1}{\lambda}\right)}\sqrt{\log(2N)},
\end{align*}
where $\lambda = N/n$ is the aspect ratio of the frame $\Sv \in \Real^{n \times N}$, and $K_u$ is its upper Kashin constant.
\end{lemma}
Lemma \ref{lem:covering_efficiency_near_optimal_source_coding} shows that when compared to naive uniform scalar quantizers, \textbf{DSC} and \textbf{NDSC} have remarkably better covering efficiency for large $n$, since it is either dimension independent or has a weak logarithmic dependence.
In the next section, we will show that this gives us independence/weak-logarithmic dependence on dimension for distributed optimization under bit-budget constraints.

\section{Proposed Optimization Algorithms}
\label{sec:proposed_optimization_algorithms}
\subsection{Smooth and Strongly Convex with Exact Gradient Oracle}
\label{subsec:smooth_and_strongly_convex_objectives}

Consider the class of $L$-smooth and $\mu$-strongly convex objective functions that satisfy $\norm{\xv_f^*} \leq D$ for some known $D \geq 0$, where $\xv_f^* = \argminimize_{\xv \in \Real^n}f(\xv)$.
For a starting point $\xhv_0 \in \Real^n$ and step-size $\alpha > 0$, we consider $\Pi_R$ to be the class of $R$-bit \textbf{Quantized Gradient Descent (QGD)} algorithms that iterate the descent rule $\xhv_{t+1} \leftarrow \xhv_t - \alpha\qv_t$, where the descent direction $\qv_t$ is a function of the computed gradients up until iteration $t$ \cite{lin2020achieving}.
Due to a bit-budget constraint, $\qv_t$ can only take values from a finite set of cardinality $2^{\lfloor nR \rfloor}$.
We consider the class of algorithms $\Pi_R$ to be those in which the worker determines the point $\zv_t$ at which the gradient is evaluated, and the quantizer input $\uv_t$, taking into account error feedback so that, $\zv_t \in \xhv_t + \text{span}\{\ev_0, \ldots, \ev_{t-1}\}$ and $\uv_t \in \nabla f\Paren{\zv_t} + \text{span}\{\ev_0, \ldots, \ev_t-1\}$, where $\ev_i \triangleq \qv_i - \uv_i ,i = 0, \ldots, t-1$ are the past quantization errors.
From Thm. IV.1 of \cite{lin2020achieving}, for the class $\Pi_R$ of $R$-bit QGD algorithms as described above, the minimax rate over the function class $\Fcal_{\mu,L,D}$ defined in \eqref{eq:minimax_defn_smooth} is lower bounded by $C(R) \geq \max\left\{\sigma, 2^{-R}\right\}$, where $\sigma \triangleq \frac{L - \mu}{L + \mu}$.
Here, $\sigma$ is the worst-case linear convergence rate of unquantized gradient descent over the same function class \cite{nesterov_2014}.
$C(R)$ has a sharp transition at a threshold budget $R_* = \log(1/\sigma)$.
\cite{lin2020achieving} shows that for their proposed algorithm, using scalar quantizers yields a convergence rate of $\leq \max\left\{\sigma, \sqrt{n}2^{-R}\right\}$.
This means that we require $R \geq \log(\sqrt{n}/\sigma)$ to achieve the convergence rate of unquantized GD, which is far from $R_*$ for large $n$.
We propose \textbf{\textsc{DGD-DEF}}: \textit{\textbf{D}istributed \textbf{G}radient \textbf{D}escent with \textbf{D}emocratically \textbf{E}ncoded \textbf{F}eedback} (Alg. \ref{alg:DGD-DEF}) which resolves this issue.
Here, $\Esf(\cdot)$ can be either $\Esf_d$ or $\Esf_{nd}$.
\textbf{\textsc{DGD-DEF}} is essentially a modification of the algorithm in \cite{lin2020achieving}, with the quantization scheme replaced by our coding scheme(s).
Thm. \ref{prop:DGD_DEF_convergence_rate} characterizes the convergence rate of \textbf{\textsc{DGD-DEF}}.
Its proof is similar to \cite[Thm. 7]{lin2020achieving} and is deferred to App. \ref{app:proof_DGDDEF_convergence_guarantee}.
\begin{theorem}(\textbf{\textsc{DGD-DEF} convergence guarantee})
\label{prop:DGD_DEF_convergence_rate}
For an objective $f \in \Fcal_{\mu, L, D}$, a bit-budget of $R$-bits per dimension, with high probability, \textbf{\textsc{DGD-DEF}} (Alg. \ref{alg:DGD-DEF}) with step-size $\alpha \leq \alpha^* \triangleq \frac{2}{L + \mu}$, employing a frame $\Sv \in \Real^{n \times N}$ for \textbf{DSC} or \textbf{NDSC} achieves
\begin{equation*}
    \norm{\xhv_T - \xv^*}_2 \leq 
    \begin{cases}
    \max\left\{\nu, \beta\right\}^T\left(1 + \beta\frac{\alpha L}{|\beta - \nu|}\right)\hspace{-0.5mm}D, \hspace{1mm} \text{ if } \hspace{2mm} \nu \neq \beta,\\
    \nu^T\left(1 + \alpha L T\right)\hspace{-0.5mm}D, \hspace{24mm} \text{otherwise},
    \end{cases}
\end{equation*}
where $\beta$ is the normalized error as in Thm. \ref{prop:quantization_error_DSC_and_NDSC}, i.e., $\beta \triangleq 2^{\left(1 - R/\lambda\right)}K_u$ if $\Esf = \Esf_d$, and $\beta \triangleq 2^{\left(2 - R/\lambda\right)}\sqrt{\log(2N)}$ if $\Esf = \Esf_{nd}$, and $\nu \triangleq \left(1 - (\alpha^* L \mu)\alpha\right)^{1/2}$ is the convergence rate of unquantized gradient descent with stepsize $\alpha$.
\end{theorem}
With $\alpha = \alpha^*$, $\limsup_{T \to \infty} \sup_{f \in \Fcal_{\mu,L,D}}\left(\norm{\xv_T - \xv_f^*}/D\right)^{1/T} = \max\left\{\sigma, 2^{-R}\beta\right\}$.
For \textbf{DSC}, $\beta = O(1)$ w.h.p., implying \textbf{\textsc{DGD-DEF}} achieves the lower bound of $\max\{\sigma, 2^{-R}\}$ to within constant factors, and since $\beta = O(\sqrt{\log n})$ for \textbf{NDSC} w.h.p., it is just a weak logarithmic factor away, which is better than $\sqrt{n}$ scaling of uniform scalar quantizers.
In other words, the threshold budget $R_{thr} = \log\left(\beta/\sigma\right)$ is much less than $\log\left(\sqrt{n}/\sigma\right)$ for large $n$.
Furthermore, compared to \cite{lin2020achieving}, which used Roger's quantizer \cite{rogers_1963} (exponential complexity), the worst-case complexity of \textbf{\textsc{DGD-DEF}} is polynomial w.r.t. dimension, i.e., $O(n^3)$ or $O(n^2)$.

{\centering
\begin{minipage}{.6\linewidth}
\begin{algorithm}[H]
    \centering
    \begin{algorithmic}
    \caption{\textbf{\textsc{DGD-DEF}}}
    \label{alg:DGD-DEF}
        \State {\bfseries Initialize:} $\xhv_0 \leftarrow \mathbf{0}$ and $\ev_{-1} \leftarrow \mathbf{0}$
        \For{$t=0$ to $T-1$}
        \State {\bfseries Worker:}
        \State $\zv_t \leftarrow \xhv_t + \alpha\ev_{t-1}$ \hfill (gradient access point)
        \State $\uv_t \leftarrow \nabla f(\zv_t) - \ev_{t-1}$ \hfill (error feedback)
        \State $\vv_t =\Esf\left(\uv_t\right)$ \hfill (source encoding)
        \State $\ev_t \leftarrow \Dsf(\vv_t) - \uv_t$ \hfill (error for next step)
        \State {\bfseries Server:}
        \State $\qv_t = \Dsf(\vv_t)$ \hfill (source decoding)
        \State $\xhv_{t+1} \leftarrow \xhv_t - \alpha\qv_t$ \hfill (gradient descent step)
        \EndFor
        \State {\bfseries Output:} $\xhv_T$
    \end{algorithmic}
\end{algorithm}
\end{minipage}
\par
}

{\centering
\begin{minipage}{.6\linewidth}
\begin{algorithm}[H]
    \centering
    \begin{algorithmic}
    \caption{\textbf{\textsc{DQ-PSGD}}}
    \label{alg:DQ-PSGD}
       \State {\bfseries Initialize:} $\xhv_0 \in \Xcal$, $\alpha \in \Real_+$ and $T$ 
        \For{$t=0$ to $T-1$}
        \State {\bfseries Worker:}
        \State $\ghv_t = \ghv(\xhv_t)$ \hfill (noisy subgradient)
        \State $\vv_t = \Esf_{Dith}(\ghv_t)$ \hfill (source encoding)
        \State {\bfseries Server:}
        \State $\qv_t = \Dsf_{Dith}(\vv_t)$ \hfill (source decoding)
        \State $\underline{\xhv}_{t+1} \leftarrow \xhv_t - \alpha\qv_t$ \hfill (subgradient step)
        \State $\xhv_{t+1} = \Gamma_{\Xcal}\left(\underline{\xhv}_{t+1}\right)$ \hfill (projection step)
        \EndFor
    \State {\bfseries Output:} $\xv_T = \frac{1}{T}\sum_{t=1}^{T}\xhv_t$ 
    \end{algorithmic}
\end{algorithm}
\end{minipage}
\par
}

\subsection{General Convex and Non-Smooth Objectives with Stochastic Subgradient Oracle}
\label{subsec:general_convex_and_non_smooth}

Consider $f$ to be convex, but not necessarily smooth.
The stochastic subgradient oracle output $\ghv(\xv)$ for any input query $\xv \in \Xcal$ is assumed to be \textit{unbiased}, i.e.,
\[\mathbb{E}[\ghv(\xv) \vert \xv] \in \partial f(\xv),\]
and \textit{uniformly bounded}, i.e.,
\[\norm{\ghv(\xv)}_2 \leq B \;\; \text{ for some } B > 0\]
In this case, an \textbf{$\mathbf{R}$-bit quantizer} is defined to be a set of (possibly randomized) mappings $(\Qsf^e, \Qsf^d)$ with the encoder $\Qsf^e:\Real^n \to \{0,1\}^{nR}$ and the decoder $\Qsf^d:\{0,1\}^{nR} \to \Real^n$.
To design the source coding scheme for a stochastic subgradient oracle, we consider the class of \textbf{gain-shape quantizers}.
For any vector input $\yv \in \Real^n$, gain-shape quantizers are of the form,
\[\Qsf(\yv) \triangleq \Qsf_G(\norm{\yv}_2)\cdot\Qsf_S(\yv/\norm{\yv}_2),\]
where $\Qsf_G:\Real \to \Real$ and $\Qsf_S:\Real^n \to \Real^n$ quantize the magnitude and shape separately, and multiply the estimates to obtain the quantized output.
We consider a uniformly dithered variant of \textbf{DSC} which we denote as $(\Esf_{Dith}, \Dsf_{Dith})$ in Alg. \ref{alg:DQ-PSGD} (cf. App. \ref{app:proof_DQPSGD_convergence_guarantee} for detailed description of this quantizer design) for $\Qsf_S$, and propose \textbf{\textsc{DQ-PSGD}}: \textbf{D}emocratically \textbf{Q}uantized \textbf{P}rojected \textbf{S}tochastic sub\textbf{G}radient \textbf{D}escent (Alg. \ref{alg:DQ-PSGD}).
We use a dithered version of $\textbf{\textsc{DSC}}$ instead of the nearest neighbor scheme of \S \ref{sec:democratic_source_coding} because for stochastic oracles, it enables us to attain the optimal minimax rate even without error-feedback.
Thm. \ref{prop:DQ_PSGD_convergence_rate} characterizes the expected suboptimality gap of \textbf{\textsc{DQ-PSGD}}.
Its proof is similar to \cite[Corollary 3.4]{mayekar_2020} (ref. App. \ref{app:proof_DQPSGD_convergence_guarantee}).

\begin{theorem}(\textbf{\textsc{DQ-PSGD} convergence guarantee})
\label{prop:DQ_PSGD_convergence_rate} For any general objective $f$ with access to the oracle $\Qsf \circ \Ocal$ which outputs quantized noisy subgradients $\Qsf(\ghv(\xv))$, where $\Qsf$ employs \textbf{DSC} for the shape quantizer, with a step-size choice of $\alpha = \frac{D}{BK_u}\sqrt{\frac{\min\{R, 1\}}{T}}$,
the worst case expected suboptimality gap of the output $\xv_T$ of \textbf{\textsc{DQ-PSGD}} after $T$ iterations is

\begin{equation}
    \sup_{(f,\Ocal)} \hspace{1mm} \mathbb{E}f(\xv_T) - f(\xv^*) \leq \frac{K_u DB}{\sqrt{T\cdot\min\{1,R\}}}.
\end{equation}
\end{theorem}

Since $K_u = O(1)$ w.h.p., Thm. \ref{prop:DQ_PSGD_convergence_rate} shows that \textbf{\textsc{DQ-PSGD}} achieves the minimax lower bound \cite[Thm. 2.3, 3.1]{mayekar_2020}, using only $R + o_n(1)$ bits per dimension.
The additional $o_n(1)$ bits is for transmitting the scalar magnitude.
Compared to \cite{mayekar_2020}, \textbf{\textsc{DQ-PSGD}} attains the minimax optimal $O(1/\sqrt{T})$ rate without additional logarithmic multiplicative factors in the bit-budget requirement.
A similar result with a weak logarithmic dependence on $n$ can be derived for \textbf{NDSC}.
Appendix \ref{app:extension_to_general_compression_schemes} shows that \textbf{DSC} \& \textbf{NDSC} improve performance when used in conjunction with existing general compression schemes, such as random sparsification.

\subsection{Extension to multiple workers}
\label{subsec:extension_to_multiple_workers}

To extend our algorithm to a setup with multiple workers, consider the following optimization problem over $m$ workers and a parameter-server (PS):
\begin{equation}
    \xv^* \triangleq \argminimize_{\xv \in \Xcal}f(\xv) \equiv \argminimize_{\xv \in \Xcal}\frac{1}{m}\sum_{i=1}^{m}f_i(\xv).
\end{equation}

{\centering
\begin{minipage}{.7\linewidth}
        \begin{algorithm}[H]
        \centering
        \begin{algorithmic}
        \caption{\textbf{\textsc{DQ-PSGD} (Multiple workers)}}
        \label{alg:dqpsgd_multiple_workers}
            \State {\bfseries Initialize:} $\xhv_0 \in \Xcal$ (at the PS), $\alpha \in \Real_+$ and $T$.
            \For{$t=0$ to $T-1$}
            \State \hspace{-2mm} {\bfseries Server:} Broadcasts $\widehat{\xv_t}$ to all workers $\text{Wk}_i$, $i \in [m]$.
            \vspace{1mm}
            \For{$i=1$ to $n$ at $\text{Wk}_i$}
            \State Compute $\widehat{\gv}^i_t = \widehat{\gv}^i(\xhv_t)$ \hfill (noisy subgradient)
            \State Encode $\vv^i_t = \Esf_{Dith}(\widehat{\gv}^i_t)$ \hfill (source encoding)
            \State $\text{Wk}_i$ sends $\vv^i_t$ to the PS. \hfill (Communication)
            \EndFor
            \vspace{1mm}
            \State \hspace{-2mm} {\bfseries Server:}
            \State \hspace{-2mm} $\qv^i_t = \Dsf_{Dith}(\vv^i_t)$ for all $i \in [n]$ \hfill (source decoding)
            \State $\qv_t = \frac{1}{n}\sum_{i=1}^{n}\qv^i_t$ \hfill (consensus step)
            \State \hspace{-2mm} $\underline{\xhv}_{t+1} \leftarrow \xhv_t - \alpha\qv_t$ \hfill (subgradient step)
            \State \hspace{-2mm} $\xhv_{t+1} = \Gamma_{\Xcal}\left(\underline{\xhv}_{t+1}\right)$ \hfill (projection step)
            \EndFor
        \State {\bfseries Output:} $\xv_T = \frac{1}{T}\sum_{t=1}^{T}\xhv_t$ 
        \end{algorithmic}
        \end{algorithm}
\end{minipage}
\par
}

\vspace{5mm}
Here, the objective $f(\xv):\Xcal \to \Real$ is the sum of multiple $f_i(\xv):\Xcal \to \Real$, each known privately to a corresponding node $i$.
Node $i$ can compute the gradient $\nabla f_i(\xv)$ (or a subgradient, $\gv^i(\xv) \in \partial f_i(\xv)$) for any $\xv \in \Xcal$, and communicate it to the PS.
For general convex, non-smooth functions with stochastic subgradient oracle, Alg. \ref{alg:DQ-PSGD} can be extended to multiple workers by incorporating an additional consensus step at the PS.
The pseudocode is provided in Alg. \ref{alg:dqpsgd_multiple_workers}.
We analyze this setting in more detail in Appendix \ref{supp_sec:extension_to_multiple_workers}.
With a budget of $R$-bits per dimension per worker, we show that using a na\"ive quantizer, the worst-case convergence rate scales as:
\begin{equation*}
    \sup_{f, \Ocal} \mathbb{E}f(\xv_T) - f(\xv^*) \lesssim O\Paren{\frac{1}{\sqrt{mT}}\cdot\frac{\sqrt{n}B}{\Paren{2^R - 1}}}.
\end{equation*}
The linear dependence of this convergence rate on the dimension $n$ can be detrimental for high dimensional problems.
With our proposed source coding schemes, we can get rid of this and get rates $O\Paren{\frac{1}{\sqrt{mT}}\cdot\frac{K_u}{(2^R - 1)}}$ and $O\Paren{\frac{1}{\sqrt{mT}}\cdot \frac{\sqrt{\log n}}{(2^R - 1)}}$ with \textbf{DSC} and \textbf{NDSC} respectively.

\section{Numerical Simulations}
\label{sec:numerical_simulations}

We validate our theoretical claims with numerical simulations.
Fig. \ref{fig:compression_methods_map} plots the normalized compression errors i.e., $\mathbb{E}\left[\norm{\Qsf(\yv) - \yv}\right]_2/\norm{\yv}_2$ for different compression schemes with and without the presence of near-democratic source coding.
The vectors $\yv \in \Real^{1000}$ chosen for compression are generated from a standard Gaussian distribution, and then raised to the power of $3$ element-wise, and averaged over $50$ realizations.
This ensures a heavier tail, and hence the entries of $\yv$ have very different magnitudes.
In the legend, \textbf{SD} denotes \textit{Standard Dithering} \cite{alistarh_NeurIPS_2017_qsgd}, \textbf{Top-K} denotes Top-K sparsification \cite{alistarh_topk_neurips_2018}, and \textbf{NDH, NDO} are abbreviations for Near-Democratic Hadamard/Orthogonal, specifying the type of randomized frame chosen for our coding scheme.
Note that for $n = 1000$ dimensions, solving \eqref{eq:l_inf_minimization_problem} to compute the democratic representation using standard optimization packages like CVX \cite{cvx} is computationally demanding. 
Hence, we used \cite{lyubarskii_2010}'s algorithm to compute Kashin representations, which require explicit knowledge of UP parameters $\eta, \delta$.
For the two plots labelled \textbf{Kashin} (with random orthonormal frame), we choose $\lambda = 1.5 \text{ and } 1.8$, which implies availability of $R/\lambda$ bits per dimension to quantize.
Due to the fixed bit-budget, the desired effect of even distribution of information in Kashin representation, is offset by the poorer quantization resolution per coordinate, which results in no net benefit (if not worse).
For this reason, in our near democratic representation with orthonormal frame, we choose $\lambda = 1$.
We observe that $\lambda$ is desired to be as close to $1$ as possible, and for Hadamard frame, we let $N = 2^{\lceil\log_2 n\rceil} = 1024$.

\begin{figure}[h]
  \begin{subfigure}[t]{.5\textwidth}
    \centering
    \includegraphics[width=\linewidth]{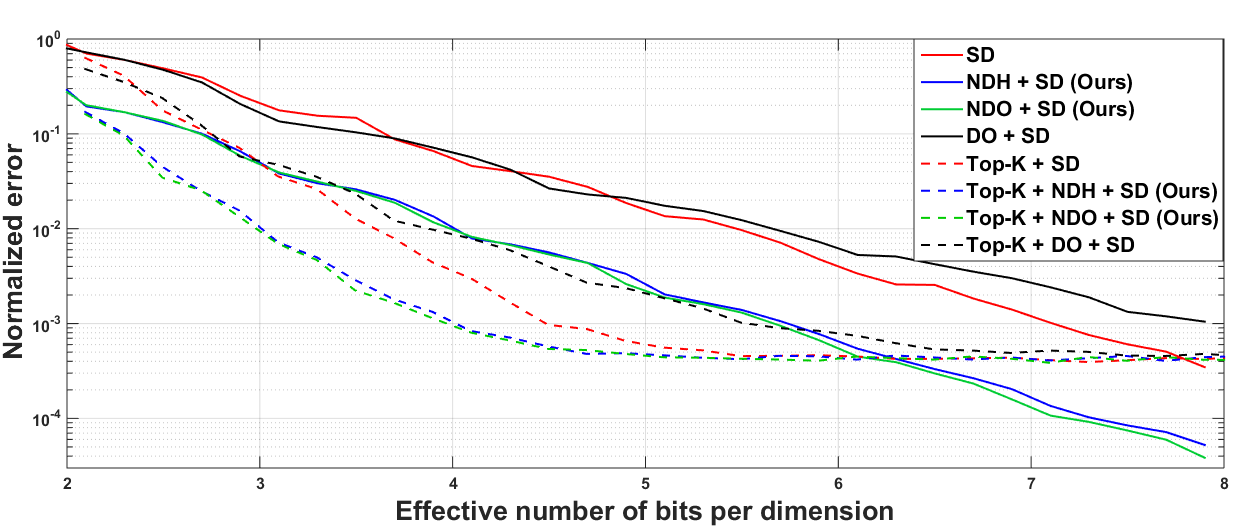}
    \caption{Comparison of different compression methods with and without near-democratic embedding}
    \label{fig:compression_methods_map}
  \end{subfigure}
  \hfill
  \begin{subfigure}[t]{.5\textwidth}
    \centering
    \includegraphics[width=\linewidth]{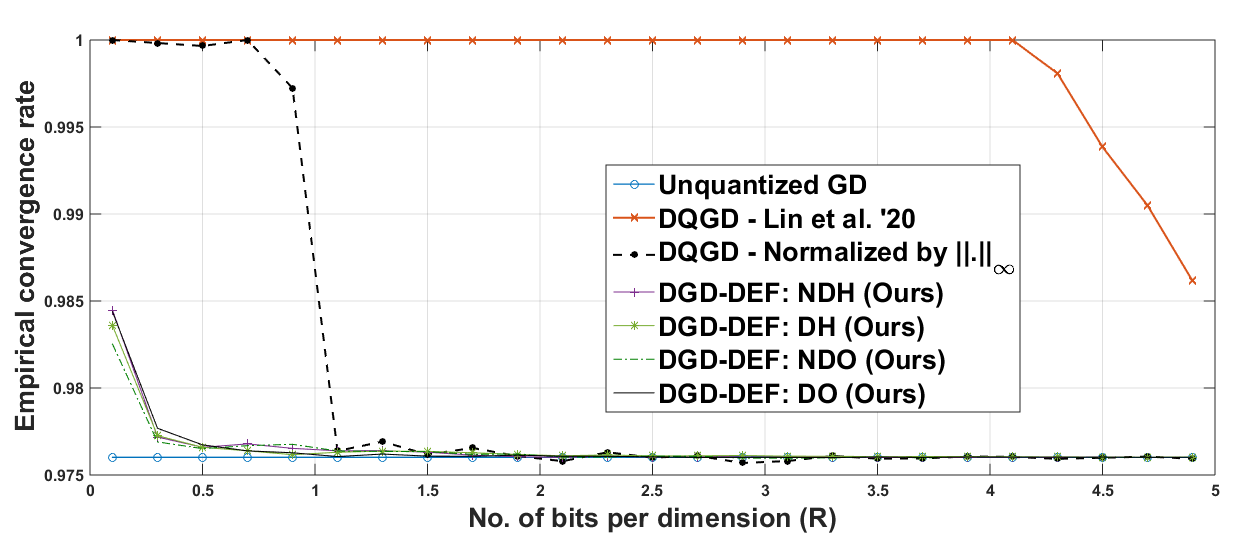}
    \caption{Variation of empirical convergence rate of \textbf{\textsc{DGD-DEF}} with bit-budget per dimension (R)}
    \label{fig:DGD-DEF_comparison}
  \end{subfigure}

  \begin{subfigure}[t]{.5\textwidth}
    \centering
    \vspace{6mm}
    \includegraphics[width=\linewidth]{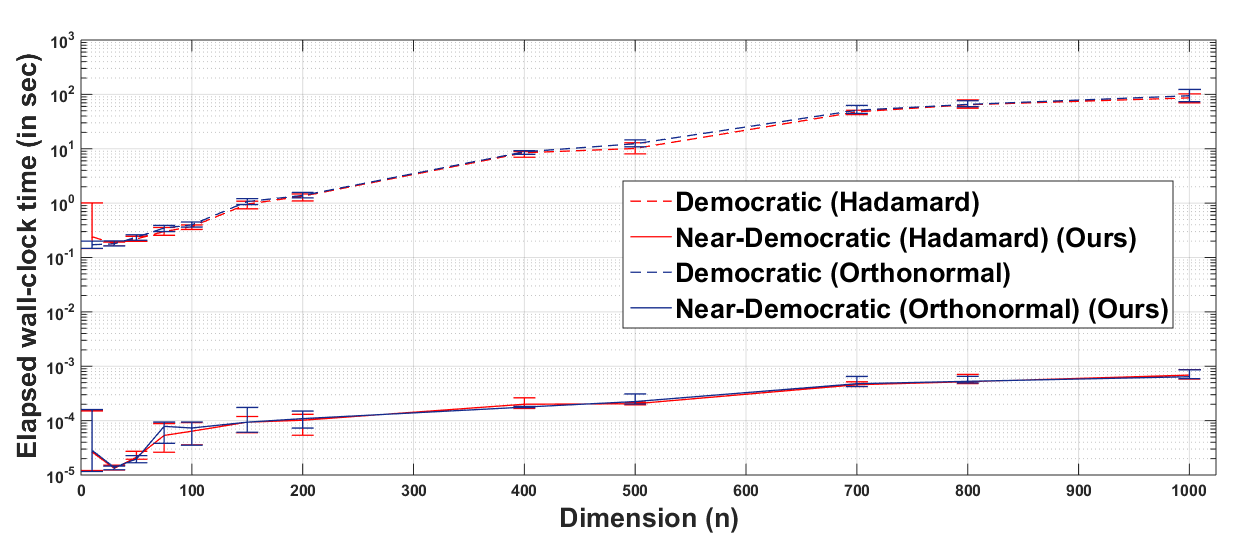}
    \caption{Wall clock times for computing near-democratic vs. democratic representations}
    \label{fig:wallclock_time}
  \end{subfigure}
  \hfill
  \begin{subfigure}[t]{.5\textwidth}
    \centering
    \vspace{6mm}
    \includegraphics[width=\linewidth]{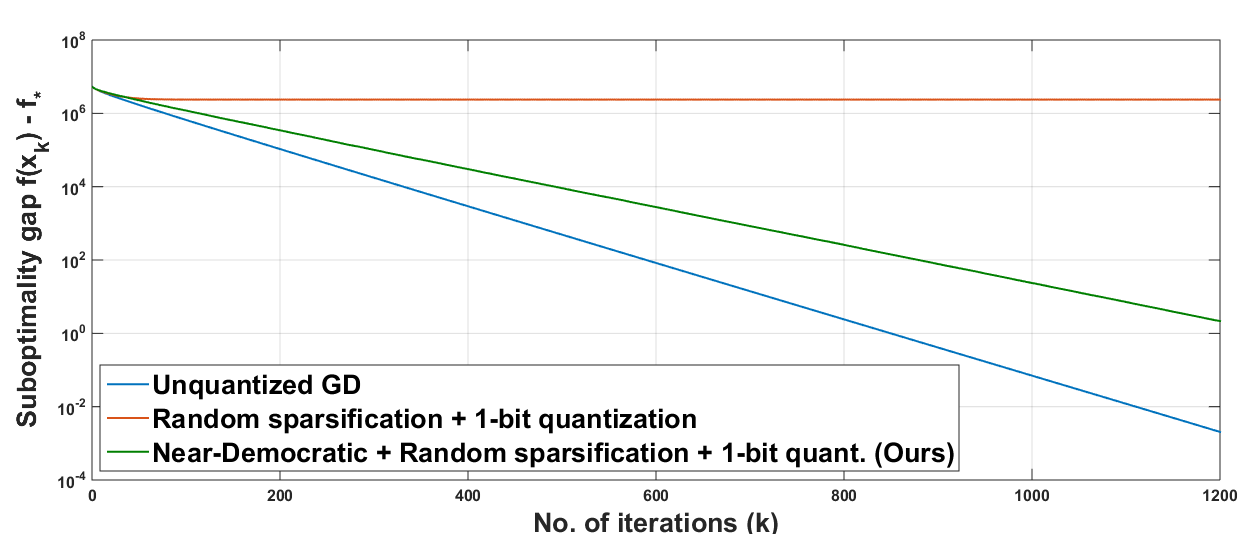}
    \caption{$\ell_2$-regularized least squares regression on \textsc{MNIST} dataset using sparsified GD}
    \label{fig:mnist_GD}
  \end{subfigure}
  \vspace{5mm}
  \caption{Simulations on smooth and strongly convex objectives}
\end{figure}

\vspace{4mm}
Fig. \ref{fig:DGD-DEF_comparison} compares the empirical convergence rate, defined as $\left.\norm{\xhv_T - \xv_f^*}_2 \middle/\norm{\xhv_0 - \xv_f^*}_2\right.$ versus the bit-budget constraint, i.e., $R$ bits per dimension, for solving the least squares problem $\minimize_{\xv \in \Real^n}\frac{1}{2}\norm{\yv - \Av \xv}_2^2$, where $n = 116$, and the entries of $\Av$ are drawn from Gaussian-cubed as before.
If the algorithm does not converge, the empirical rate is clipped at $1$.
\textit{Unquantized GD} has a constant rate equal to $\frac{L - \mu}{L +\mu}$ \cite{nesterov_2014}.
\textit{DQGD} proposed in \cite{lin2020achieving} used a predefined sequence of dynamic ranges, and \textit{nearest-neighbor} scalar quantization.
In comparison, we normalize the input to the quantizer by $\norm{\cdot}_{\infty}$ norm, and since it is a scalar quantity, we assume that it is transmitted with infinite precision.
A more comprehensive justification for sending scalars can be found in App. \ref{app:quantizing_the_linf_norm}.
We observe that Near-Democratic Embeddings (\textbf{NDE}) perform at par with Democratic Embeddings (\textbf{DE}), and both ensure convergence at very low bit-budgets.
Sometimes, it may even perform better because \textbf{NDE} allows us to choose $N = n$, and hence as seen before, no resolution is lost due to the fixed bit-budget.
Moreover, the computational advantage of \textbf{NDE} is evident from Fig. \ref{fig:wallclock_time} where we plot the wall-clock time (in seconds) (averaged over 10 realizations) vs. dimension to find these embeddings.
The \textbf{DE}'s are obtained by solving \eqref{eq:l_inf_minimization_problem} using CVX and the \textbf{NDE}'s are obtained from the closed form expression $\xv = \Sv^{\top}\yv$.
Here, for each $n$, the value of $N$ is chosen to be $N = 2^{\lceil \log_2 n \rceil}$.
This plot was obtained on a \textit{Dell Vostro} with an \textit{Intel i5 1.60GHz processor} running \textit{\textit{MATLAB R2014b}}.
Finally, in Fig. \ref{fig:mnist_GD}, we solve the $\ell_2$-regularized least squares problem for the MNIST dataset \cite{lecun_mnist}.
We use gradient descent where the gradients are compressed, first by \textit{random sparsification} followed by an aggressive $1$-bit quantization for the retained coordinates, so that effectively $R = 0.5$ bits are used per dimension.
We note that \textbf{NDE}'s using random orthonormal frames converge for $R = 0.5$, whereas the vanilla compression scheme fails.
For least-square simulations, we use the step-size $\alpha^*$ given by Thm. \ref{prop:DGD_DEF_convergence_rate}.

\begin{figure}
  \begin{subfigure}[t]{.5\textwidth}
    \centering
    \includegraphics[width=\linewidth]{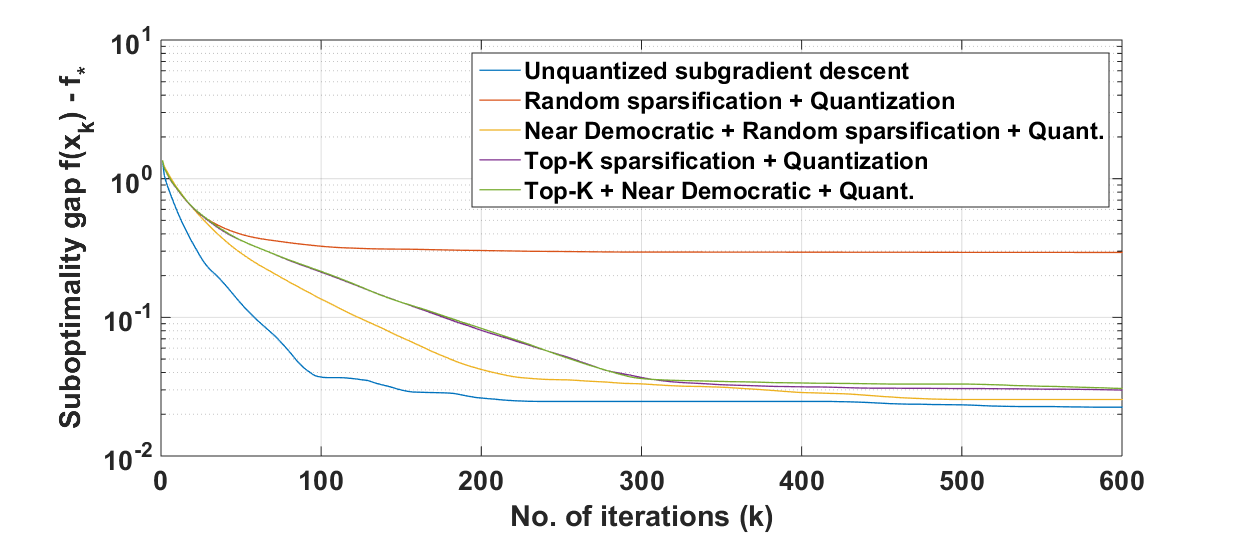}
    \caption{Suboptimality gap: SVM Gaussian data}
    \label{fig:SVM_suboptimality_gap}
  \end{subfigure}
  \hfill
  \begin{subfigure}[t]{.49\textwidth}
    \centering
    \includegraphics[width=\linewidth]{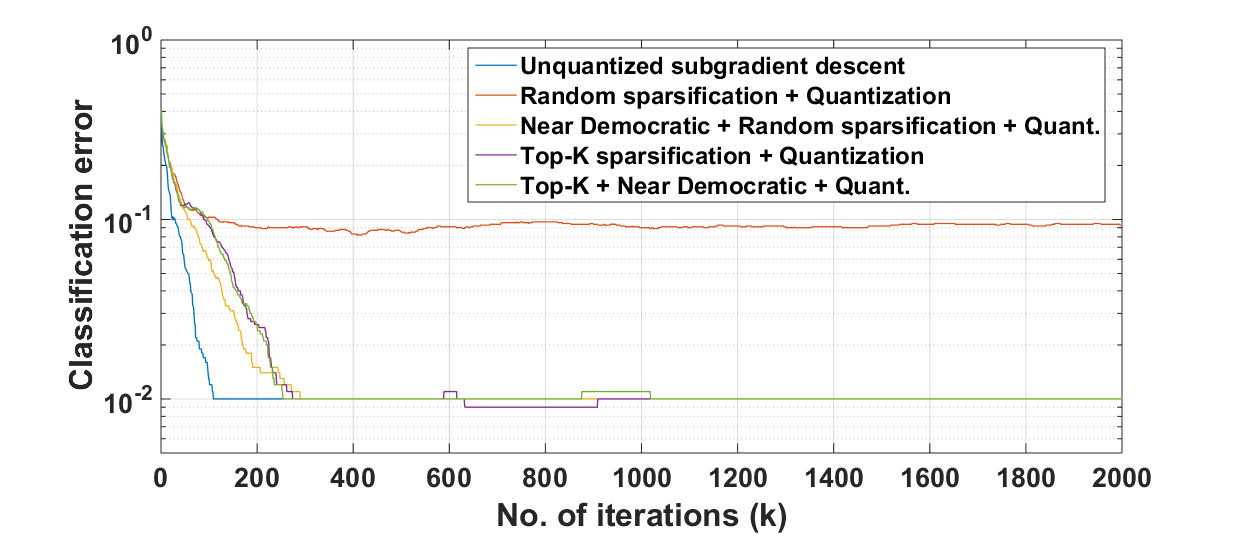}
    \caption{Classification error: SVM Gaussian data}
    \label{fig:SVM_classification_error}
  \end{subfigure}
  
  \begin{subfigure}[t]{.5\textwidth}
    \centering
    \vspace{6mm}
    \includegraphics[width=\linewidth]{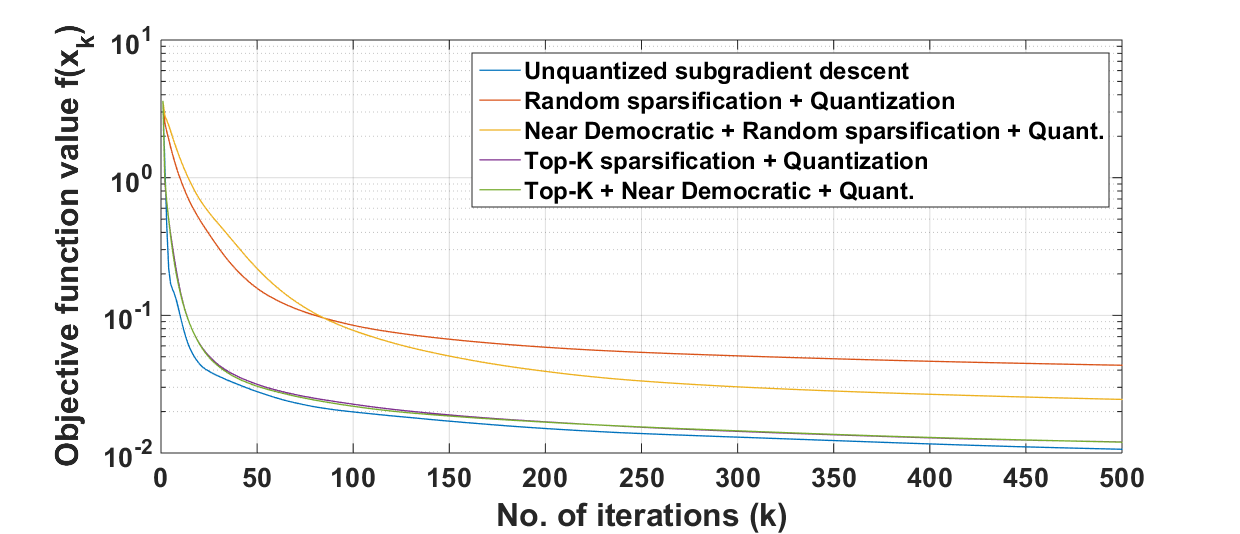}
    \caption{Objective function value: SVM MNIST $0$ vs $1$}
    \label{fig:MNIST_SVM_obj_fnval}
  \end{subfigure}
  \hfill
  \begin{subfigure}[t]{.49\textwidth}
    \centering
    \vspace{6mm}
    \includegraphics[width=\linewidth]{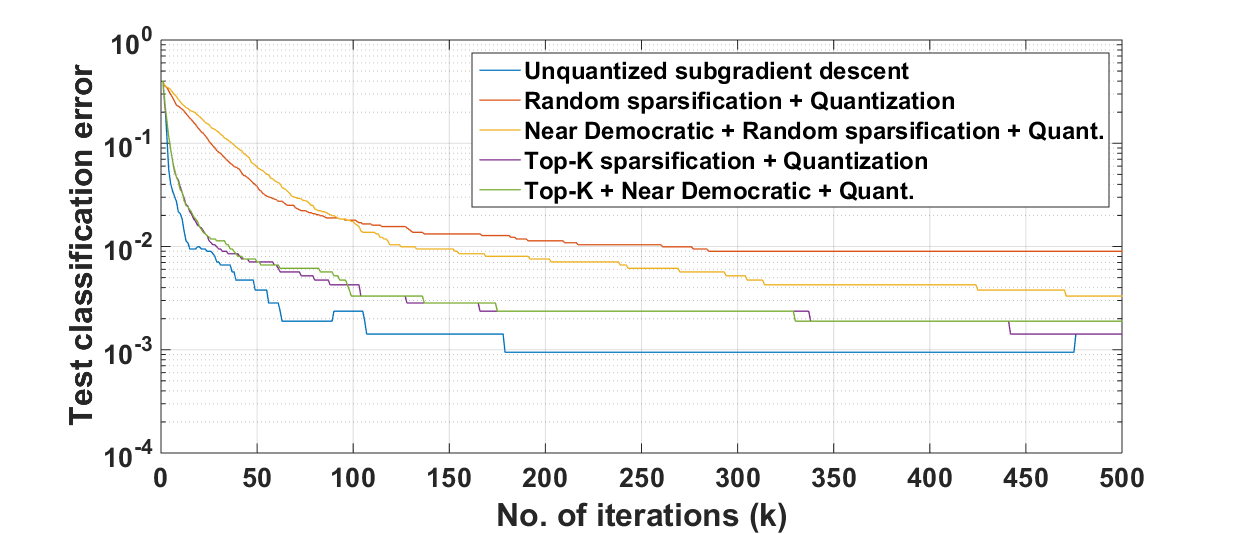}
    \caption{Classification error: SVM MNIST $0$ vs $1$}
    \label{fig:MNIST_SVM_test_classification_error}
  \end{subfigure}
  \vspace{5mm}
  \caption{General convex and non-smooth: Training an SVM}
\end{figure}

For general convex \& non-smooth objectives, we train a support vector machine where the subgradients are quantized using $R$-bits per dimension.
Each worker has $m$ datapoints $\{(\av_i,b_i)\} \in \Real^n \times \{-1,+1\}$ for $i = 1, \ldots, m$.
We want to solve the following optimization problem in which our aim is to minimize the hinge loss: $\minimize_{\xv \in \Real^n}\frac{1}{m}\sum_{i=1}^{m} \max\Paren{0, 1 - b_i\cdot\xv^{\top}\av_i}$.
We compare the performance of our proposed \textbf{\textsc{DQ-PSGD}} algorithm with naive scalar quantization as well as unquantized projected stochastic subgradient descent.
The stochasticity in the subgradient oracle evaluation arises from randomly subsampling the dataset to compute the subgradient at every iteration.
We consider the number of datapoints $m = 100$, and dimension of the problem $n = 30$.
For Figs. \ref{fig:SVM_suboptimality_gap} and \ref{fig:SVM_classification_error}, the data corresponding to each class is generated independently from Gaussian distributions with different means.
We consider random orthonormal frames for computing \textbf{NDE}'s.
Fig \ref{fig:SVM_suboptimality_gap} plots the suboptimality gap (averaged over $10$ different realizations) versus the number of iterations. 
The optimal value $f_*$ for computing the suboptimality gap in Fig. \ref{fig:SVM_suboptimality_gap} is obtained by using the interior point solver provided by CVX \cite{cvx}.
We effectively have $R = 0.5$, i.e., less than one bit per dimension.
In other words, since we only have a total of $nR = 15$ bits available, the subgradient is randomly sparsified by making certain coordinates zero, and the remaining vector is quantized using $1$-bit per dimension.
There is a significant difference in performance when randomly sparsifying $50\%$ of the coordinates with and without \textbf{NDE}'s.
We also consider top-K sparsification \cite{stich_2018_sparsifiedSGD} with and without \textbf{NDE}'s.
We choose $K=3$, i.e., we decide to retain only the top $10\%$ of the coordinates.
When we employ both sparsification and quantization techniques simultaneously to compress the gradient but only have a fixed total number of bits available, there arises a tradeoff between how many coordinates we want to retain and the number of bits allotted for quantizing each retained coordinate.
Smaller $K$ means more bits per coordinate i.e., better resolution for scalar quantization of the retained coordinates, and vice versa.
In random sparsification, we retain $15$ coordinates with $1$-bit allotted for each.
For top-K, we retain $3$ coordinates and allot $5$ bits for quantizing each of them.
Although top-K is expected to perform better than random sparsification, choosing the value of $K$ heuristically may yield poorer performance (despite the additional computation required for determining the top K coordinates) as in this case.
We also plot the classification error, that is the percentage of misclassified samples in the training set at every iteration in Fig. \ref{fig:SVM_classification_error} and observe a similar trend for the different sparsification and quantization schemes.

Figs. \ref{fig:MNIST_SVM_obj_fnval} and \ref{fig:MNIST_SVM_test_classification_error} consider the MNIST dataset \cite{lecun_mnist}, and the problem of training an SVM to distinguish the digit $0$ from digit $1$.
Fig. \ref{fig:MNIST_SVM_obj_fnval} shows how the objective function value decreases with the number of iterations.
Fig. \ref{fig:MNIST_SVM_test_classification_error} plots the classification error on the hold-out test set for each iteration.
We consider only $1$ realization for this setting and let $R = 0.1$.
For top-K, we retain the top $10\%$ coordinates, while ensuring that the total bit-budget remains same for all the schemes, i.e., a total of $\lfloor nR \rfloor = \lfloor 784 \times 0.1 \rfloor = 78$ bits.
For random sparsification with and without \textbf{NDE}'s, $78$ coordinates are chosen randomly from the gradient which $\in \Real^{784}$ and $1$ bit is allotted to each of them. 
For top-$10\%$, we now choose to retain the top $78$ coordinates of maximum magnitude, allot $1$ bit to quantize each of them, and make the rest zero.
Since the number of retained coordinates is the same for both random and top-K sparsification in this setting i.e., $78$, top-K performs better as expected.
For this set of simulations, we have chosen a nominal step-size $\alpha = 1$ empirically, and kept it constant for a fair comparison of different algorithms.

\begin{figure}[!ht]
\begin{subfigure}[!ht]{.5\textwidth}
    \centering
    \includegraphics[width=\linewidth]{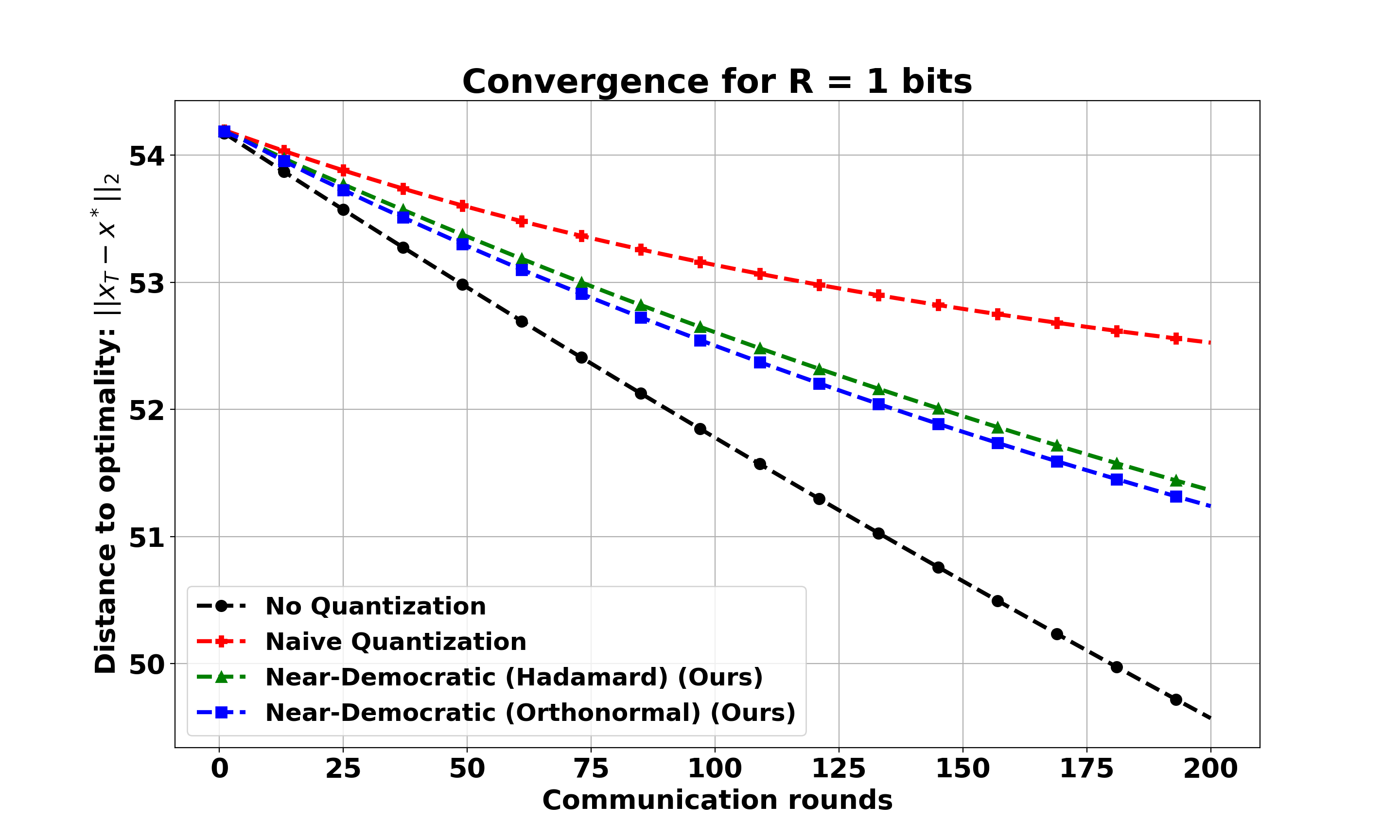}
    \caption{Linear Regression over $m = 10$ nodes}
    \label{fig:multi_worker_regression}
  \end{subfigure}
  \hfill
  \begin{subfigure}[!ht]{.5\textwidth}
    \centering
    \includegraphics[width=\linewidth]{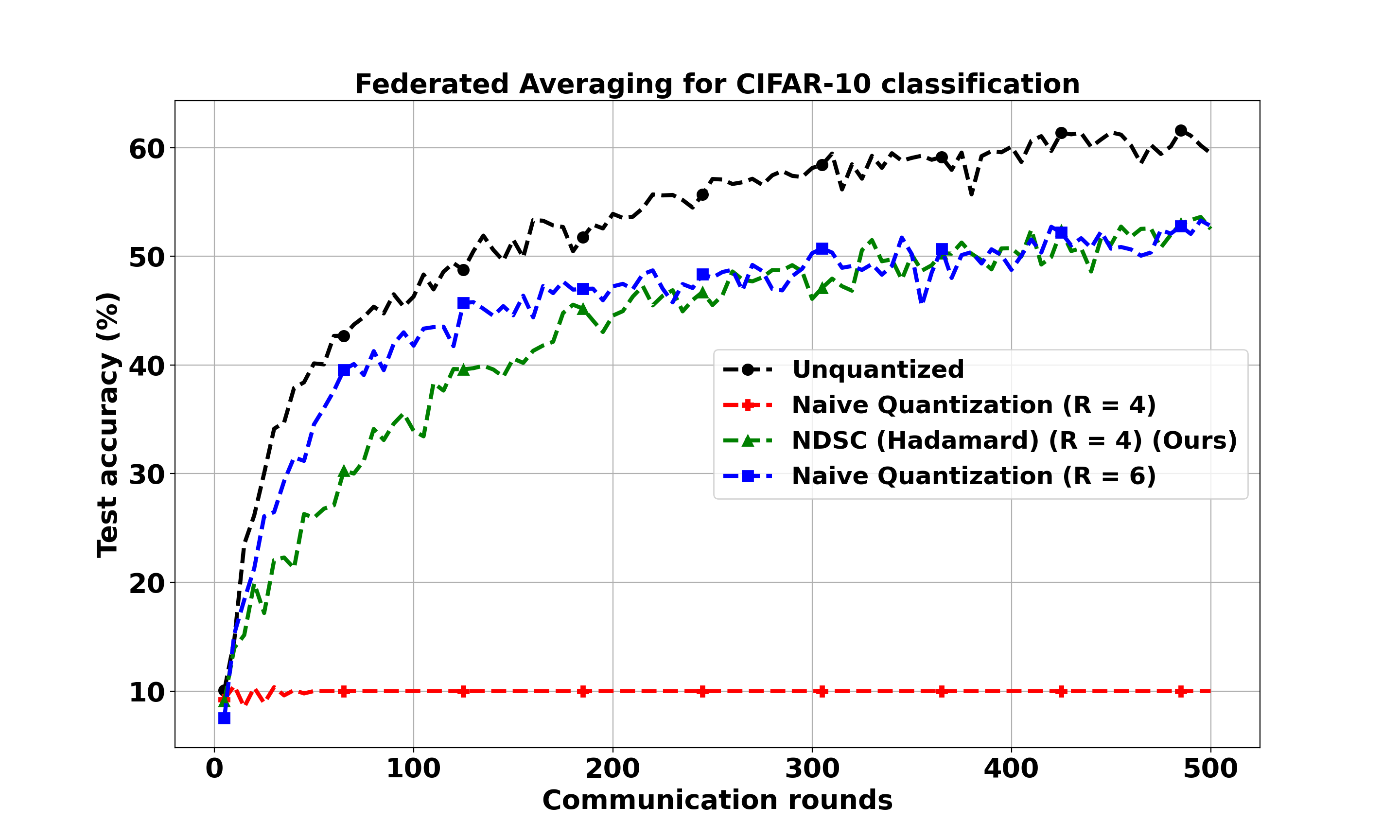}
    \caption{Training a CNN over $m = 10$ nodes}
    \label{fig:multi_worker_cnn}
  \end{subfigure}
  \vspace{6mm}
  \caption{Parameter-Server with multiple workers}
  \label{fig:multi_worker_simulations}
\end{figure}

\textbf{Multi-worker simulations}.
We consider two problems.
Fig. \ref{fig:multi_worker_regression} considers a regression model that solves: $\xv^* \equiv \argminimize_{\xv \in \Real^n} \frac{1}{m}\sum_{i=1}^{m}\Paren{\frac{1}{s}\sum_{j=1}^{s}\frac{1}{2}\Paren{b_{ij} - \av_{ij}^{\top}\xv}^2}$.
Here, $\{\av_{ij}, b_{ij}\}_{j=1}^{s}$ denotes the local dataset at node $i$, for $i \in [m]$.
The dimension of the problem is $n = 30$, $m = 10$ workers with each worker having $s = 10$ local datapoints.
The dataset is generated synthetically from a model $\xv^*$ according to the noisy planted model $\bv = \Av\xv^*$, where $\bv \in \Real^{ms}$ is the regression output and the rows of $\Av \in \Real^{ms \times n}$, i.e. $\{\av^{\top}_1, \ldots, \av^{\top}_{ms}\}$ are the data vectors.
We let $\xv^* \sim \text{Student-t (df = 1)}$ and the entries of the data matrix $\Av \stackrel{iid}{\sim} \Ncal(0,1)$.

Although our theoretical analysis is for convex functions, we also provide simulations for non-convex settings.
In Fig. \ref{fig:multi_worker_cnn}, we train a convolutional neural network (CNN) over $m = 10$ workers to do multi-class classification on the CIFAR-10 \cite{cifar10} image classification dataset that contains $50,000$ training and $10,000$ test images from $10$ classes.
The entire dataset is distributed across these workers in a non-i.i.d. fashion, so that each worker has images from at most $2$ out of the $10$ classes.
As can be seen from the plot, with a bit-budget of $R = 4$ bits per dimension per worker, our proposed near-democratic source coding (NDSC) scheme with randomized Hadamard frame outperforms na\"ive quantization with the same bit-budget ($R = 4$) which fails to even converge.
As a matter of fact, na\"ive quantization requires a higher budget of $R = 6$ bits per dimension per worker to achieve a performance comparable to that of NDSC.
Further detailed simulations are provided in Appendix \ref{supp_sec:extension_to_multiple_workers}.

\section{Conclusions}
\label{sec:conclusions}

In this work, we show that democratic embeddings can yield minimax optimal distributed optimization algorithms under communication constraints when employed in source coding schemes.
For smooth \& strongly convex objectives, we propose \textbf{\textsc{DGD-DEF}}, which employs error feedback to achieve linear convergence.
For the case of general convex \& non-smooth objectives, when the output of the stochastic subgradient oracle is quantized using a democratic source coding scheme, \textbf{\textsc{DQ-PSGD}} attains minimax optimal convergence rate.
We note that although \textbf{DSC} theoretically attains minimax optimal performance to within constant factors, computing democratic embeddings can be computation and memory intensive.
We also propose a randomized Hadamard construction for fast near-democratic embeddings.
Finally, we extend the analysis and simulate our algorithms for multi-worker setups.
A potential limitation of the proposed optimization approaches is that the curvature information is not utilized, which we leave as future work.

\clearpage
\bibliographystyle{IEEEtran}
\bibliography{refs}

\clearpage

\appendices

\begin{center}
    \section*{\Huge Appendices}
\end{center}

\section{Proof of Lemma \ref{lemma:near_democratic_dynamic_range_random_orthonormal}: Near-Democratic Embeddings with Random Orthonormal Frames}
\label{app:proof_NDE_Random_Orthonormal}

Let $\{\sv_i\}_{i=1}^{N} \in \Real^n$ and $\{\stv_i\}_{i=1}^{N} \in \Real^N$ denote the columns of $\Sv$ and $\Stv$ respectively, where $\Sv$, $\Stv$ are defined in \S \ref{subsec:near_democratic_embeddings}.
For $i \in [N]$, since $\Stv^{\top}\Stv = \Iv_N$, we have,
\[\norm{\sv_i}_2 \; \leq \; \norm{\Pv\stv_i}_2 \; \leq \; \norm{\stv_i}_2 \; = \; 1.\]
For any $\yv \in \Real^n$, let $\yhv = \yv/\norm{\yv}_2$.
Then,
\begin{align*}
    \norm{\Sv^{\top}\yv}_{\infty} = \max_{i \in [N]}|\sv_i^{\top}\yv| = \norm{\yv}_2\max_{i \in [N]}\norm{\sv_i}_2|\shv_i^{\top}\yhv| \leq \norm{\yv}_2 \max_{i \in [N]}|\shv_i^{\top}\yhv|,
\end{align*}
where $\shv_i = \sv_i/\norm{\sv_i}_2$.
Note that $\shv_i \in \Real^n$ is uniformly random on the unit sphere in $\Real^n$, i.e., $\shv_i$ has identical distribution as $\gv/\norm{\gv}_2$ where $\gv \sim \Ncal(\mathbf{0}, \Iv_n)$.
Due to rotational invariance of Gaussian distribution, for any fixed $\yhv \in \Real^n$ such that $\norm{\yhv}_2 = 1$, $\shv_i^{\top}\yhv$ has identical distribution as $\shv_i^{\top}\ev_1$, where $\ev_1$ is the first canonical basis vector.
From concentration of measure for uniform distribution on the unit sphere,
\[\Prob\left[|\shv_i^{\top}\yhv| \geq t\right] = \Prob\left[|s_1| \geq t\right] \leq 2e^{-nt^2/2}.\]
Using a union bound, 
\[\Prob\left[\max_{i \in [N]}|\shv_i^{\top}\yhv| \geq t\right]\leq 2Ne^{-nt^2/2}.\]
Setting $t = 2\sqrt{\frac{\log(2N)}{n}}$ yields,
\[\Prob\left[\norm{\Sv^{\top}\yv}_{\infty} \geq 2\sqrt{\frac{\lambda \log(2N)}{N}}\norm{\yv}_2\right] \leq \frac{1}{2N},\] 
which completes the proof.

\section{Proof of Lemma \ref{lemma:near_democratic_dynamic_range_randomized_Hadamard}: Near-Democratic Embeddings with Randomized Hadamard Frames}
\label{app:proof_NDE_Randomized_Hadamard}

Denote $\zv = \Pv^{\top}\yv = [z_1, \ldots, z_N]^{\top} \in \Real^N$, and let $\uv = \Hv\Dv\zv = [u_1, \ldots,u_N]^{\top}$.
Here, $u_j$ is of the form $\sum_{i=1}^{N}a_i z_i$, with each $a_i = \pm\frac{1}{\sqrt{N}}$ chosen i.i.d.
For any $t \in \Real$ and $\lambda > 0$, a Chernoff-type argument gives,
\begin{align*}
    \Prob[u_j > t] \; = \; \Prob\left[e^{\lambda u_j} > e^{\lambda t}\right] \; \leq \; e^{-\lambda t}\prod_{i=1}^{N}\mathbb{E}\left[e^{\lambda a_i z_i}\right].
\end{align*}
Now,
\begin{align*}
    \mathbb{E}\left[e^{\lambda a_i z_i}\right] = \frac{1}{2}e^{\frac{\lambda}{\sqrt{N}}z_i} + \frac{1}{2}e^{-\frac{\lambda z_i}{\sqrt{N}}} = \cosh\left(\frac{\lambda}{\sqrt{N}z_i}\right) \leq e^{\lambda^2z_i^2/(2N)},
\end{align*}
where the last inequality follows from a bound on hyperbolic cosine.
This gives us $\Prob\left[u_j > t\right] \leq e^{\frac{\lambda^2}{2N}\norm{\zv}_2^2 - \lambda t}$.
Setting $\lambda = tN/\norm{\zv}_2^2$ gives the tightest bound,
\[\Prob\left[u_j > t\right] \leq e^{-t^2N/(2\norm{\zv}_2^2)}.\]
Similarly, one can show that,
\[\Prob\left[u_j < -t\right] \leq e^{-t^2N/(2\norm{\zv}_2^2)}.\]
Since $\norm{\uv}_{\infty} = \max_{j \in [N], s \in \{\pm 1\}}s u_j$, union bound gives,
\[\Prob\left[\norm{\Hv\Dv\zv}_{\infty} > t\right] \leq e^{-\frac{t^2N}{2\norm{\zv}_2^2} + \log(2N)}.\]
Setting $t = 2\norm{\zv}_2\sqrt{\frac{\log(2N)}{N}}$ yields,
\[\Prob\left[\norm{\Hv\Dv\zv}_{\infty} \leq 2\norm{\zv}_2\sqrt{\frac{\log(2N)}{N}}\right] \geq 1 - \frac{1}{2N}.\]
Since $\zv = \Pv^{\top}\yv \implies \norm{\zv}_2 = \norm{\yv}_2$, this completes the proof.

\section{Proof of Thm. \ref{prop:quantization_error_DSC_and_NDSC}: Quantization Error: (N)-DSC}
\label{app:proof_of_quantization_error_DSC_and_NDSC}

Let $\Esf$ denote either $\Esf_{d}$ or $\Esf_{nd}$.
Then, given an input $\yv \in \Real^n$ to $\Esf(\cdot)$, let $\xv \in \Real^N$ be the (near) democratic representation, $\xtv = \xv/\norm{\xv}_{\infty}$ be the normalized input to $\Qsf(\cdot)$, $\xv' \in \Real^N$ be the encoder output, and $\yv' = \Dsf(\xv') \in \Real^n$ be the decoder output.
The error incurred after encoding and subsequent decoding is 
\begin{align*}
    \norm{\yv' - \yv}_2 \leq \norm{\xv}_{\infty}\norm{\Sv\left(\xv' - \xtv\right)}_2 \leq \norm{\xv}_{\infty}\norm{\xv' - \xtv}_2
\end{align*}
The last inequality follows since,
\[\norm{\Sv\left(\xv' - \xtv\right)}_2 \leq \norm{\Sv}_{2}\norm{\xv - \xtv}_2.\]
Since $\Sv$ is a Parseval frame, and non-zero eigenvalues of $\Sv^{\top}\Sv$ are the same as those of $\Sv\Sv^{\top} = \Iv_n$, we have $\norm{\Sv}_2 = 1$.
To upper bound the quantization error $\norm{\xv' - \xtv}_2$, note that if we originally had a total budget of $nR$-bits, the number of bits per dimension to uniformly quantize $\xtv \in \Real^N$ is now $nR/N$, i.e., $2^{nR/N}$ quantization points per dimension.
From \eqref{eq:uniform_quantizer_error},
\[\norm{\xv' - \xtv}_2 \leq 2^{1-nR/N}\sqrt{N}.\]
So, if we use $\Qsf_d$ from Lemma \ref{lem:democratic_embedding_property},
\[\norm{\yv' - \yv}_2 \leq \frac{K_u}{\sqrt{N}}\norm{\yv}_2 2^{1 - \frac{nR}{N}}\sqrt{N} = 2^{\left(1 - \frac{R}{\lambda}\right)}K_u\norm{\yv}_2.\]
Similarly, for $\Qsf_{nd}$, using Lemma \ref{lemma:near_democratic_dynamic_range_randomized_Hadamard}, 
\[\norm{\yv' - \yv}_2 \leq 2\sqrt{\frac{\log(2N)}{N}}\norm{\yv}_22^{\left(1 -\frac{nR}{N}\right)}\sqrt{N}.\]
This completes the proof.

\section{Proof of Thm. \ref{prop:DGD_DEF_convergence_rate}: Convergence Rate of \textbf{\textsc{DGD-DEF}}}
\label{app:proof_DGDDEF_convergence_guarantee}

For an \textbf{$R$-bit Quantized Gradient Descent (QGD)} algorithm (defined in \S \ref{sec:proposed_optimization_algorithms} and \cite[Def. IV.1]{lin2020achieving}), the minimax convergence rate \eqref{eq:minimax_defn_smooth} over the function class $\Fcal_{\mu, L, D}$ is lower bounded as $C(R) \geq \max\{\sigma, 2^{-R}\}$, where $\sigma \triangleq \frac{L - \mu}{L + \mu}$.
The convergence analysis makes use of a recursive invariant satisfied by the trajectory of \textbf{\textsc{DGD-DEF}}.
Consider the two descent trajectories: \textbf{\textsc{DGD-DEF}} and \textbf{unquantized GD} with the same step size $\alpha$, starting at the same location $\xhv_0 = \xv_0$.
Then using triangle inequality, at each iteration $i \in \mathbb{N}$ we have, 
\begin{align*}
    \xhv_t = \xv_t - \alpha \ehv_{t-1}
    \implies \norm{\xhv_t - \xv^*}_2 \leq \norm{\xv_t - \xv^*}_2 + \alpha\norm{\ehv_{t-1}}_2,
\end{align*}
where $\ehv_{-1} = \mathbf{0}$.
From algorithm pseudocode \ref{alg:DGD-DEF}, note that $\zv_t = \xv_t$, i.e., \textbf{\textsc{DGD-DEF}} computes the gradient at the unquantized trajectory $\{\xv_t\}_{t=0}^{\infty}$.
Decay of the first term $\norm{\xv_t - \xv^*}_2$ is given by the convergence guarantee of unquantized GD \cite{nesterov_2014}, which states that $\norm{\xv_T - \xv^*}_2 \leq \nu^T \norm{\xv_0 - \xv^*}_2$, where $\nu \triangleq (1 - (\alpha^* L \mu)\alpha)^{1/2}$ is the convergence rate for unquantized GD with step size $\alpha$.
An upper bound to the second term $\norm{\ehv_{t-1}}_2$ is obtained from our quantization scheme, as per the following auxiliary lemma.
\begin{lemma}
\label{lem:upper-bound_input_to_quantize}
For $f \in \Fcal_{\mu, L, D}$, at any iteration $t$, the quantizer input satisfies $\norm{\uv_t}_2 \leq r_t$, where the sequence $\{r_t\}$ is given by $r_t = LD\sum_{j=0}^{t}\nu^j \beta^{t-j}$. 
Here $\beta \triangleq 2^{(1 - R/\lambda)}K_u$ if democratic embeddings are used, and $\beta \triangleq 2^{(2 - R/\lambda)}\sqrt{\log(2N)}$ if near-democratic embeddings are used.
\end{lemma}
\begin{proof}
This is proved using induction.
For $t = 0$, we have $\uv_0 = \nabla f(\xv_0) - \ev_{-1}$.
Since $\ev_{-1} = \mathbf{0}$, recalling that $\nabla f(\xv^*) = \mathbf{0}$ and $f$ is $L$-smooth, we have, 
\begin{align*}
    \norm{\uv_0}_2 = \norm{\nabla f(\xv_0)}_2 = \norm{\nabla f(\xv_0) - \nabla f(\xv^*)}_2 \leq L\norm{\xv_0 - \xv^*}_2 \leq LD,
\end{align*}
and so the lemma holds true for $t = 0$.
From triangle inequality, 
\[\uv_t = \nabla f(\xv_t) - \ev_{t-1} \implies \norm{\uv_t}_2 \leq \norm{\nabla f(\xv_t)}_2 + \norm{\ev_{t-1}}_2.\]
The first term can be upper bounded as,
\begin{align*}
    \norm{\nabla f(\xv_t)}_2 = \norm{\nabla f(\xv_t) - \nabla f(\xv^*)}_2 \leq L\norm{\xv_t - \xv^*}_2 \leq L\nu^t \norm{\xv_0 - \xv^*}_2 \leq L\nu^t D.
\end{align*}
Furthermore, from Thm. \ref{prop:quantization_error_DSC_and_NDSC},
\begin{align*}
    \norm{\ev_{t-1}}_2 &= \norm{\uv_{t-1} - \Dsf\left(\Esf(\uv_{t-1})\right)}_2\\
    &\leq \beta \norm{\uv_{t-1}}_2 \leq \beta L D \sum_{j=0}^{t-1}\nu^j\beta^{t-1-j}\\
    &= LD\sum_{j=0}^{t-1}\nu^j\beta^{t-j},
\end{align*}
where $\beta$ depends on whether we choose democratic or near-democratic embeddings for our source coding scheme, and the second inequality is the induction hypothesis.
Using these, we can upper bound the magnitude of the quantizer input as,
\[\norm{\uv_t}_2 \leq LD \left( \nu^t + \sum_{j=0}^{t-1}\nu^j\beta^{t-j} \right) = LD \sum_{j=0}^{t}\nu^j\beta^{t-j}.\]
This completes the proof.
\end{proof}
The proof of Thm. \ref{prop:DGD_DEF_convergence_rate} is similar to the \cite[Thm. III.1]{lin2020achieving} except for appropriate modifications due to our proposed source coding schemes: \textbf{DSC} and \textbf{NDSC}.
Since for $t \in [N]$,
\[\norm{\xhv_t - \xv^*}_2 \leq \norm{\xv_t - \xv^*}_2 + \alpha\norm{\ehv_{t-1}}_2,\]
the first term can be upper bounded from the descent guarantee of unquantized GD trajectory as,
\[\norm{\xv_T - \xv^*}_2 \leq \nu^T\norm{\xv_0 - \xv^*}_2 \leq \nu^T D.\]
The second term can be upper bounded as
\[\norm{\ev_{T-1}}_2 \leq \beta \norm{\uv_{T-1}}_2 \leq \beta r_{T-1} = \beta L D \sum_{j=0}^{T-1}\nu^j \beta^{T-1-j},\] 
where the inequalities follow from the definition of $\beta$ and Lemma \ref{lem:upper-bound_input_to_quantize}.
So we have, $\norm{\xhv_T - \xv^*}_2 \leq bD, \hspace{2mm} \text{where} \hspace{2mm} b = \nu^T + \beta \alpha L \sum_{j=0}^{T-1}\nu^j \beta^{T-1-j}$.
There are now three possibilities:
\begin{enumerate}[(i)]
    \item $\nu > \beta$: The geometric sum is computed as,
    \[b = \nu^T + \beta\alpha L \nu^{T-1}\frac{1 - (\beta/\nu)^T}{1 - \beta/\nu} \leq \nu^T\left(1 + \beta\frac{\alpha L}{\nu - \beta}\right).\]
    
    \item $\nu = \beta$: In this case, 
    \[b = \nu^T + \alpha L \nu \cdot \nu^{T-1} T = \nu^T\left(1 + \alpha L T\right).\]
    
    \item $\nu < \beta$: The case parallels the first case by interchanging the role of $\nu$ and $\beta$, and we get, 
    \begin{align*}
         b &= \nu^T + \beta \alpha L \beta^{T-1}\sum_{j=0}^{T-1}\left(\frac{\nu}{\beta}\right)^j \\
         &= \nu^T + \alpha L \beta^T\frac{1 - (\nu/\beta)^T}{1 - \nu/\beta} \leq \beta^T\left(1 + \beta\frac{\alpha L}{\beta - \nu}\right)
    \end{align*}
\end{enumerate}
the proof is complete by concisely expressing the above three cases as:

\begin{equation}
    \norm{\xhv_T - \xv^*}_2 \leq 
    \begin{cases}
    \max\left\{\nu, \beta\right\}^T\left(1 + \beta\frac{\alpha L}{|\beta - \nu|}\right)D, \text{ if } \nu \neq \beta,\\
    \nu^T\left(1 + \alpha L T\right)D \;\; \text{ otherwise}.
    \end{cases}
\end{equation}

\section{Proof of Thm. \ref{prop:DQ_PSGD_convergence_rate}: Convergence Rate of \textbf{\textsc{DQ-PSGD}}}
\label{app:proof_DQPSGD_convergence_guarantee}

Consider the optimization problem: $\minimize_{\xv \in \Xcal} f(\xv)$, where, $\Xcal \subseteq \Real^n$ is a convex set, and $f$ is convex, but not necessarily smooth.
$\Xcal$ satisfies $\sup_{\xv, \yv \in \Xcal} \norm{\xv - \yv}_2 \leq D$ for some known $D \geq 0$.
We assume oracle access to noisy subgradients of $f$.
The oracle output $\ghv(\xv)$ for any input query point $\xv \in \Xcal$ to be \textit{unbiased}, i.e., $\mathbb{E}[\ghv(\xv) \vert \xv] \in \partial f(\xv)$, where, $\partial f(\xv)$ denotes the subdifferential of $f$ at the point $\xv$, and \textit{uniformly bounded}, i.e., $\norm{\ghv(\xv)}_2 \leq B$ for all $\xv$, for some $B > 0$.
An \textbf{$\mathbf{R}$-bit quantizer} is defined to be a (possibly randomized) pair of mappings $(\Qsf^e, \Qsf^d)$, with the \textbf{encoder} mapping $\Qsf^e:\Real^n \to \{0,1\}^{nR}$ and the \textbf{decoder mapping} $\Qsf^d:\{0,1\}^{nR} \to \Real^n$.
Let $\Qcal_R$ denote the set of all such $R$-bit quantizers.
For any pair $(f, \Ocal)$ of objective function $f$ and oracle $\Ocal$, and an $R$-bit quantizer $\Qsf$, let $\Qsf \circ \Ocal$ denote the composition oracle that outputs $\Qsf(\ghv(\xv))$ for each query $\xv \in \Xcal$.
Let $\pi \in \Pi_{T,R}$ be an optimization protocol as defined in \S \ref{sec:introduction}.
We consider the minimax expected suboptimality gap \eqref{eq:minimax_defn_non_smooth}.
The convergence rate of \textbf{\textsc{DQ-PSGD}} depends on the quantizer design of $\Qsf \in \Qcal_R$.

The performance of any quantizer is determined by the following two quantities:
The \textbf{worst-case second moment}, i.e.,
\[\Acal_{\Qsf} \triangleq \sup_{\yv \in \Real^n : \norm{\yv}_2 \leq B} \sqrt{\mathbb{E}[\norm{\Qsf(\yv)}_2^2]},\]
and the \textbf{worst-case bias}, i.e.,
\[\Bcal_{\Qsf} \triangleq \sup_{\yv \in \Real^n : \norm{\yv}_2 \leq B} \norm{\mathbb{E}[\yv - \Qsf(\yv)]}_2.\]
For any such quantizer $\Qsf$, from \cite[Thm. 2.4]{mayekar_2020}, the worst-case expected suboptimality gap of quantized projected subgradient algorithm after $T$ iterations, with step-size $\alpha = \frac{D}{\Acal_{\Qsf}\sqrt{T}}$ is,
\begin{equation}
\label{eq:QPSGD_convergence_rate}
    \sup_{(f,\Ocal)} \mathbb{E} f(\xv) - f(\xv^*) \leq D\left(\frac{\Acal_{\Qsf}}{\sqrt{T}} + \Bcal_{\Qsf}\right).
\end{equation}
We design $\Qsf$ so that for any input $\yv \in \Real^n$, $\Qsf$ encodes the magnitude (gain) of the input $\norm{\yv}_2$ and the direction (shape) $\yv_S = \frac{\yv}{\norm{\yv}_2}$ of $\yv$ separately, and forms the estimate of $\yv$ by multiplying the estimates for the magnitude and direction.
In other words, 
\[\Qsf(\yv) \triangleq \Qsf_G(\norm{\yv}_2)\cdot \Qsf_S\left(\frac{\yv}{\norm{\yv}_2}\right),\]
where $\Qsf_G:\Real \to \Real$, and $\Qsf_S:\Real^n \to \Real^n$.

It is assumed that given $\yv$, $\Qsf_G$ and $\Qsf_S$ are independent of each other.
From \cite[Thm. 4.2]{mayekar_2020}, if $\Qsf_S$ is unbiased, i.e., $\mathbb{E}[\Qsf_S(\yv_S)] = \yv_S$ for all $\yv_S$ satisfying $\norm{\yv_S}_2 \leq 1$, then,
\[\Bcal_{\Qsf} \leq \sup_{\yv \in \Real^n : \norm{\yv}_2 \leq B} \left\vert\mathbb{E}[\Qsf_G\left(\norm{\yv}_2\right) - \norm{\yv}_2]\right\vert.\]

The \textbf{uniformly dithered quantizer} for $\Qsf_G$ is described next.
Let the \textbf{dynamic range} of $\Qsf_G$ be the known uniform upper bound, $B$.
Consider $m$ quantization points $\{u_1, \ldots, u_m\}$ uniformly spaced along the interval $[0,B]$ and let $u_0 = 0$ and $u_{m+1} = B$.
For any input $v \in [u_j, u_{j+1}) \subseteq [0,B]$, the output of the gain quantizer $\Qsf_G(v)$ is defined to be:
\vspace{-2mm}
\begin{equation}
\label{eq:dithered_quantization}
    \Qsf_G(v) = 
    \begin{cases}
    u_j \hspace{6mm} \text{with probability} \hspace{2mm} $r$, \\
    u_{j+1} \hspace{2mm} \text{with probability} \hspace{2mm} $1-r$,
    \end{cases}
\end{equation}
where, $r \triangleq \frac{u_{j+1} - v}{(B/(m+1))}$.
If a fixed number of $b = \log_2 m$ bits (typically $32$) are used, it can be easily shown that $\Qsf_G$ is unbiased.
For designing $\Qsf_S$, we consider two separate cases: The \textbf{high-budget regime} ($R > 1$) and the \textbf{sub-linear budget regime} ($R < 1$).
The proof Thm. \ref{prop:DQ_PSGD_convergence_rate} is completed after combining the results of \S \ref{subsec:high_budget_regime} and \S \ref{subsec:sublinear_budget_regime}.

\vspace{-4mm}
\subsection{High-budget regime}
\label{subsec:high_budget_regime}
Let the input to $\Qsf_S$ be $\yv$ such that $\norm{\yv}_2 \leq 1$.
For $\Sv \in \Real^{n \times N}$, if $\xv_d$ denotes the democratic embedding of $\yv$ with respect to $\Sv \in \Real^{n \times N}$, then $\norm{\xv_d}_{\infty} \leq \frac{K_u}{\sqrt{N}}\norm{\yv}_2 = \frac{K_u}{\sqrt{N}}$.
Let the \textbf{coordinate-wise uniformly dithered quantizer} ($\Qsf_{CUQ}$) be as in \cite{mayekar_2020}, in which we do dithered quantization \eqref{eq:dithered_quantization} of each coordinate of $\xv_d$ independently, using $R$ bits per dimension and a dynamic range of $\left[-\frac{K_u}{\sqrt{N}},+\frac{K_u}{\sqrt{N}}\right]$.
The output is $\Qsf_S(\yv) = \Sv\Qsf_{CUQ}(\xv_d)$.
Since $\Qsf_{CUQ}$ is unbiased, the output of $\Qsf_S$ is also unbiased, i.e., $\mathbb{E}[\Qsf_S(\yv)] = \mathbb{E}[\Sv \Qsf_{CUQ}(\xv_d)] = \Sv \mathbb{E}[\Qsf_{CUQ}(\xv_d)] = \Sv\xv_d = \yv$.
Since both $\Qsf_G$ and $\Qsf_S$ are conditionally independent of each other, the bias of $\Qsf(\cdot) = \Qsf_G(\cdot)\cdot\Qsf_S(\cdot) = 0$, i.e., the worst-case bias $\Bcal_{\Qsf} = 0$.
To evaluate the worst-case second moment of $\Qsf$, from \cite[Thm. 4.2]{mayekar_2020}, $\Acal_{\Qsf} \leq \Acal_{\Qsf_G}\Acal_{\Qsf_S}$.
Since the dynamic range of $\Qsf_G$ is $B$, the worst-case bias $\Acal_{\Qsf_G} \leq B$.
For $\Qsf_S$ and any $\yv$ such that $\norm{\yv}_2 \leq 1$, we have $\norm{\Qsf_S(\yv)}_2^2 = \norm{\Sv\Qsf_{CUQ}(\xv_d)}_2^2 \leq \sigma_{max}^2(\Sv)\cdot\norm{\Qsf_{CUQ}(\xv_d)}_2^2 \leq \norm{\Qsf_{CUQ}(\xv_d)}_2^2$.
The final inequality follows since $\Sv\Sv^{\top}=\Iv_n$.
Moreover, since the dynamic range for $\Qsf$ is $\frac{K_u}{\sqrt{N}}$, we have $\norm{\xv_d}_{\infty} \leq \frac{K_u}{\sqrt{N}} \implies \norm{\Qsf_{CUQ}(\xv_d)}_2^2 \leq K_u^2$ for all $\yv$ $\implies \mathbb{E}\norm{\Qsf_{CUQ}(\xv_d)}_2 \leq K_u^2 \implies \Acal_{\Qsf_S} \leq K_u$.
So, the worst-case second moment is $\Acal_{\Qsf} = \sqrt{\mathbb{E}\norm{\Qsf_S(\yv)}_2^2} \leq BK_u$.
Substituting these values for $\Acal_{\Qsf}$ and $\Bcal_{\Qsf}$ in \eqref{eq:QPSGD_convergence_rate}, we get the result.
A similar result with a $O(\sqrt{\log n})$ dependence on dimension can be proved for \textbf{NDSC}.

\vspace{-4mm}
\subsection{Sub-Linear budget regime}
\label{subsec:sublinear_budget_regime}

When $R < 1$, the total bit-budget is $r = nR \leq n$, i.e., we have less than $1$-bit per coordinate.
For brevity, let $N = n$.
The statements here can be easily generalized to the case of $N > n$.
To allocate $r = nR < n$ bits to each coordinate of a vector in $\mathbb{R}^n$ so that on an average $R$-bits per dimension is utilized, we choose $r = nR$ coordinates uniformly at random; allocate $1$-bit to each of these coordinates, and allocate $0$-bit to the remaining coordinates.
This essentially subsamples the vector in $\mathbb{R}^n$ to a smaller dimensional vector in $\mathbb{R}^{nR}$, and subsequently doing a $1$-bit quantization of the vector in $\mathbb{R}^{nR}$, while decoding the other coordinates as $0$.
Since and we are subsampling and want $\Qsf_S$ to be unbiased, we need to scale the quantized output by a factor of $\frac{1}{R}$.
So, the democratic representation + subsampling + $1$-bit quantization scheme as is $\mathsf{Q}_S(\mathbf{y}) = \frac{1}{R}\mathbf{S}\sum_{i \in \mathcal{S}}\mathsf{Q}_{CUQ}(\mathbf{x}_d)\mathbf{e}_i = \frac{1}{R}\mathbf{S}\sum_{i = 1}^n \mathsf{Q}_{CUQ}(\mathbf{x}_d)\mathbf{e}_i \mathbf{1}_{i \in \mathcal{S}}$.
Here, $\mathbf{e}_i \in \mathbb{R}^n$ is the $i^{th}$ canonical basis vector, $\mathcal{S}$ with $|\mathcal{S}| = nR$ denotes the random $nR$ indices chosen in the subsampling step, and $\mathbf{1}_{(\cdot)}$ denotes the indicator function, since the coordinates not selected in the subsampling step are decoded as $0$.
Unbiasedness ensures that $\Bcal_{\Qsf} = 0$.
Moreover, 
\begin{align*}
    \mathbb{E}\lVert\mathsf{Q}_S(\mathbf{y})\rVert_2^2 &\leq \frac{1}{R^2}\mathbb{E}\left[\left\lVert\sum_{i \in \mathcal{S}}\mathsf{Q}_{CUQ}(\mathbf{x}_d)\mathbf{e}_i\right\rVert_2^2\right]\\
    &= \frac{1}{R^2}\sum_{i=1}^{n}\mathbb{E}[\mathsf{Q}_{CUQ}(\mathbf{x}_d)_i^2]\cdot \mathbb{E}[\mathbf{1}_{i \in \mathcal{S}}] \\
    &= \frac{1}{R}\mathbb{E}\lVert\mathsf{Q}_{CUQ}(\mathbf{x}_d)\rVert_2^2
\end{align*}
The last equality follows as sampling $nR$ coordinates uniformly at random from $n$ coordinates implies $\mathbb{E}[\mathbf{1}_{i \in \mathcal{S}}] = R$.
Since $\mathbb{E}\lVert\mathsf{Q}_{CUQ}(\mathbf{x}_d)\rVert_2^2 \leq K_u^2$, the worst-case second moment is $\Acal_{\Qsf} \leq \sqrt{\mathbb{E}\lVert\mathsf{Q}(\mathbf{y})\rVert_2^2} \leq \frac{BK_u}{\sqrt{R}}$.
Substituting these values of $\Acal_{\Qsf}$ and $\Bcal_{\Qsf}$ in \eqref{eq:QPSGD_convergence_rate}, the convergence rate when $R < 1$ is $\frac{K_uDB}{\sqrt{RT}}$.
This completes the proof.

\section{Quantizing the \texorpdfstring{$\ell_{\infty}$}{linfty} norm}
\label{app:quantizing_the_linf_norm}

For simplicity, the statement of Thm. \ref{prop:DGD_DEF_convergence_rate} assumes that the $\ell_{\infty}$ norm of the input to the quantizer can be transmitted perfectly without lossy quantization.
If we use a constant number $O(1)$ of bits (typically, $32$ bits depending on the machine precision) to uniformly quantize $\norm{\xv}_{\infty}$, the total number of bits required to quantize the vector is $nR + O(1) \implies$ $R + \frac{O(1)}{n} \to R$ bits per dimension as $n \to \infty$.
Hence, the bit-budget is respected asymptotically and the additional constant number of bits is negligible for high dimensional problems.
To take account the error due to quantizing $\norm{\xv}_{\infty}$, note that if $\yv_S = \frac{\yv}{\norm{\yv}_{\infty}}$, 
\begin{align*}
    \norm{\Qsf(\yv) - \yv}_2 &= \norm{\Qsf_G(\norm{\yv}_{\infty})\Qsf_S\left(\yv_S\right) - \norm{\yv}_{\infty}\yv_S}_2 \\
    &= \left\lVert\Paren{\norm{\yv}_{\infty} + \epsilon}\Qsf_S(\yv_S) - \norm{\yv}_{\infty}\yv_S\right\rVert_2\\
    &\leq \norm{\yv}_{\infty}\norm{\Qsf_S(\yv_S) - \yv_S}_2 + \epsilon\norm{\Qsf_S(\yv_S)}_2
\end{align*}
The last inequality follows from triangle inequality.
Since we employ \textbf{DSC} for the shape quantizer, the first term can be upper bounded using Thm. \ref{prop:quantization_error_DSC_and_NDSC}.
The second term can be upper bounded as,
\begin{align*}
    \epsilon\norm{\Qsf_S(\yv_S)}_2 \leq \epsilon \sqrt{N} \norm{\Qsf(\yv_S)}_{\infty} \leq \epsilon \sqrt{N} \frac{K_u}{\sqrt{N}} \leq \epsilon K_u,
\end{align*}
which is an additional constant error.
The whole convergence analysis follows through with this constant additive error too.

\section{Closed Form for Near-Democratic Embeddings (NDE)}
\label{app:near_democratic_closed_form}

For any $\yv \in \Real^n$, its \textbf{NDE}, $\xv_{nd} \in \Real^N$ with respect to a frame $\Sv \in \Real^{n \times N}$ (for $N \geq n$) is defined to be the solution of the $\ell_2$ minimization problem \textcolor{blue}{$(7)$}.
Its Lagrangian is,
\[L(\xv, \nuv) = \xv^{\top}\xv + \nuv^{\top}\left(\Sv \xv - \yv\right)\]
$\nuv \in \Real^n$.
This gives,
\[\nabla_{\xv}L(\xv, \nuv) = 2\xv + \Sv^{\top}\nuv = 0 \implies \xv_{nd} = -\frac{1}{2}\Sv^{\top}\nuv.\]
So, $\nuv = -2\left(\Sv\Sv^{\top}\right)^{-1}\yv$ and,
\[\xv_{nd} = -\frac{1}{2}\Sv^{\top}\left(-2\left(\Sv\Sv^{\top}\right)^{-1}\yv\right) = \Sv^{\top}\left(\Sv\Sv^{\top}\right)^{-1}\yv = \Sv^{\top}\yv.\]
The last equality follows from the fact that we choose our frames $\Sv \in \Real^{n \times N}$ to be \textbf{Parseval}, i.e., they satisfy $\Sv\Sv^{\top} = \Iv^{n \times n}$.
Random Haar orthonormal or Random Hadamard frames are Parseval frames and hence \textbf{NDE}s can be computed very efficiently.

\section{Extension to General Compression Schemes}
\label{app:extension_to_general_compression_schemes}

For any general (possibly stochastic) lossy compression scheme like sparsification, standard dithering, etc, instead of compressing the input $\yv \in \Real^n$ directly, compressing its (near) democratic embedding $\xv \in \Real^N$, will consistently improve robustness by ensuring that the error due to compression has independent/weak-logarithmic dependence on dimension $n$.
\begin{theorem}
\label{prop:general_compressors}
For any unbiased compression operator $\Ccal : \Real^N \to \Real^N$ that satisfies $0 \leq \Ccal(\xv)\text{sign}(\xv) \leq \norm{\xv}_{\infty}$ for all $\xv \in \Real^N$, \textbf{DSC}/\textbf{NDSC} with corresponding encoding and decoding functions defined for any $\yv \in \Real^n$ as $\Esf\left(\yv\right) = \Ccal\left(\xv\right)$ and $\Dsf\left(\xv\right) = \Sv \xv$, satisfies $\Expect [\norm{\Dsf(\Esf(\yv)) - \yv}_2^2] \leq \gamma^2\norm{\yv}_2^2$, where $\gamma = K_u$ if $\xv = \xv_d$, and $\gamma = 2\sqrt{\log(2N)}$ if $\xv = \xv_{nd}$.
\end{theorem} 
\begin{proof}
For a given frame $\Sv \in \Real^{n \times N}$, from linearity of expectations, $\Expect\left[\Dsf\left(\Esf(\yv)\right)\right] = \Expect\left[\Dsf\left(\Ccal(\xv)\right)\right] \\= \Expect\left[\Sv \Ccal(\xv)\right] = \Sv \Expect\left[\Ccal(\xv)\right] = \Sv \xv = \yv$.
Here, $\xv$ can represent either the \textbf{DE} ($\xv_d$) or the \textbf{NDE} ($\xv_{nd}$).
The approximation error can be bounded uniformly (without the expectation) as,
\vspace{-3mm}
\begin{align}
\label{eq:proof_prop5_eq1}
        \hspace{-2mm}\norm{\Dsf\left(\Esf(\yv)\right) - \yv}_2^2 = \norm{\Sv \Ccal(\xv) - \Sv \xv}_2^2 \stackrel{(i)}{\leq} \norm{\Ccal(\xv) - \xv}_2^2
        \leq N\max_{1 \leq i \leq N}\left(\Ccal(\xv)_i - \xv_i\right)^2 \leq N \norm{\xv}_{\infty}^2.
\end{align}
$(i)$ follows since $\norm{\Sv}_2 = 1$. 
If $\xv = \xv_d$, then from Lemma \textcolor{blue}{$1$}, $\norm{\xv}_{\infty} \leq \frac{K_u}{\sqrt{N}}\norm{\yv}_2$.
Substituting this in \eqref{eq:proof_prop5_eq1}, we get the result for \textbf{DE} with $\gamma = K_u$.
Similarly, if $\xv = \xv_{nd}$, using Lemma \textcolor{blue}{$3$}, $\norm{\xv}_{\infty} \leq 2\sqrt{\frac{\log(2N)}{N}}\norm{\yv}_2$.
Substituting this once again in \eqref{eq:proof_prop5_eq1}, we get the result for \textbf{NDE} with $\gamma = 2\sqrt{\log(2N)}$.
This completes the proof.
\end{proof}

\section{Extension to multiple workers}
\label{supp_sec:extension_to_multiple_workers}

\begin{figure}[h!]
      \centering
      \includegraphics[width=0.45\linewidth]{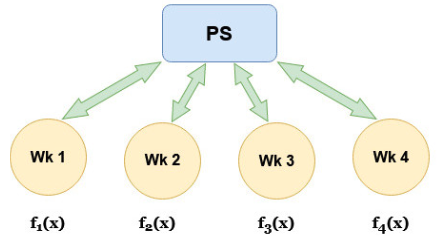}  
      \caption{Parameter server and Workers}
      \label{fig:PS_setup}
\end{figure}

\vspace{4mm}
To extend our algorithm to a setup with multiple workers, consider the following optimization problem over $m$ workers and a parameter-server (PS):
\begin{equation}
    \xv^* \triangleq \argminimize_{\xv \in \Xcal}f(\xv) \equiv \argminimize_{\xv \in \Xcal}\frac{1}{m}\sum_{i=1}^{m}f_i(\xv).
\end{equation}
As shown in Fig. \ref{fig:PS_setup}, the objective $f(\xv):\Xcal \to \Real$ is the sum of multiple $f_i(\xv):\Xcal \to \Real$, each known privately to a corresponding node $i$.
Node $i$ can compute the gradient $\nabla f_i(\xv)$ (or a subgradient, $\gv^i(\xv) \in \partial f_i(\xv)$) for any $\xv \in \Xcal$, and communicate it to the PS.
For general convex, non-smooth functions with stochastic subgradient oracle, Alg. \textcolor{blue}{$2$} can be extended to multiple workers by incorporating an additional consensus step at the PS.
The pseudocode is provided below in Alg. \ref{alg:dqpsgd_multiple_workers}.

Convergence analysis of SGD (from any standard text such as \cite{bottou_2018}) tells us that the upper bound on the expected suboptimality gap scales proportionally with the variance of the stochastic subgradient, i.e. $\mathbb{E}\lVert \qv_t - \gv_t \rVert_2^2$, where $\gv_t = \frac{1}{m}\sum_{i=1}^{m}\gv^i_t$, $\gv^i_t \in \partial f_i(\xv_t)$ for all $i = 1, \ldots, m$ and $t = 0, \ldots, T-1$.
Succinctly stated, similar to Thm. $3$ of the main paper, we have,
\begin{equation*}
    \sup_{f, \Ocal} \mathbb{E} f(\xv_T) - f(\xv^*) \lesssim O\Paren{\frac{\mathbb{E}\lVert \qv_T - \gv_T \rVert_2}{\sqrt{T}}}.
\end{equation*}
Let $\widehat{\gv}_t = \frac{1}{m}\sum_{i=1}^{m}\widehat{\gv}^i_t$ be the global stochastic subgradient at the PS if there were no quantization constraints.
Young's inequality states that for any $\av, \bv \in \Real^n$ and any $u > 0$, we have $\lVert \av + \bv \rVert_2^2 \leq (1 + u)\lVert \av \rVert_2^2 + (1 + u^{-1})\lVert \bv \rVert_2^2$.
Using this with $u = 1$, in the presence of quantization, the variance can be upper bounded as,
\begin{align*}
    \mathbb{E}\lVert \qv_t - \gv_t \rVert_2^2 &\leq 2\mathbb{E}\lVert \qv_t - \widehat{\gv}_t \rVert_2^2 + 2\mathbb{E}\lVert \widehat{\gv}_t - \gv_t \rVert_2^2 \\
    &\leq 2\mathbb{E}\left\lVert \frac{1}{m}\sum_{i=1}^{m}\Paren{\qv^i_t - \widehat{\gv}^i_t} \right\rVert_2^2 + 2\mathbb{E}\left\lVert \frac{1}{m}\sum_{i=1}^{m}\Paren{\widehat{\gv}^i_t - \gv^i_t} \right\rVert_2^2 \\
    &\stackrel{(i)}{=} \frac{2}{m^2}\sum_{i=1}^{n}\mathbb{E}\lVert \qv^i_t - \widehat{\gv}^i_t \rVert_2^2
    + \frac{2}{m^2}\sum_{i=1}^{n}\mathbb{E}\left\lVert \widehat{\gv}^i_t - \gv^i_t \right\rVert_2^2 \stackrel{(ii)}{\leq} \frac{2}{m}\Paren{\sigma_q^2 + \sigma_o^2}.
\end{align*}
Here, the equality $(i)$ follows from the fact that we assume the gradient quantization and the stochastic subgradient oracle at each node to be independent of each other (so that the cross-terms vanish).
Also, $\sigma_o > 0$ is a known upper bound on the inherent stochasticity of the subgradient oracle, i.e. $\mathbb{E}\lVert \widehat{\gv}^i_t - \gv^i_t \rVert_2^2 \leq \sigma_o^2$ for all $i \in [m], t = 0, \ldots, T-1$.
$\sigma_q^2$ is the additional variance introduced due to quantization at each node, and it depends on the source coding scheme we use at each node.
We next consider expressions for $\sigma_q$ for the na\"ive and our proposed quantization strategies.

Consider the \textbf{stochastic uniform quantizer} as defined below.
This is similar to the quantizer defined in App. \textcolor{blue}{E} of the main paper (with slight differences in scaling that can be easily accounted for with a more detailed analysis).
For an input $\xv \in \Real^n$, let $x(j)$ denote the value of its $j^{th}$ coordinate. 
The stochastic uniform quantizer quantizes each coordinate of $\xv$ independently, i.e. the quantizer output is given by $Q(\xv) = [\Qsf(x(1)), \ldots, \Qsf(x(d))]$, where the scalar quantizer $\Qsf(\cdot)$ is defined next.
Along each dimension, consider the $2^R$ quantization points $\{u_i\}$, $0 \leq i \leq 2^R - 1$ given by,
\begin{equation*}
    u_i \triangleq -\norm{\xv}_{\infty} + i\cdot\frac{2\norm{\xv}_{\infty}}{2^R - 1}
\end{equation*}
For any coordinate $j$, $1 \leq j \leq n$, suppose $x(j) \in [u_{r}, u_{r+1})$ for some $r$.
Then the output of $\Qsf$ is,
\begin{align*}
    \Qsf(x(j)) = 
    \begin{cases}
    u_{r+1} \text{ with probability } \frac{x(j) - u_r}{u_{r+1} - u_r} \\
    u_r \hspace{4mm}\text{ otherwise.} 
    \end{cases}
\end{align*}
Note that this quantizer satisfies $\mathbb{E}[\xv] = \xv$ for any $\xv \in \Real^n$, and its variance can be upper bounded as,
\begin{align*}
    \mathbb{E}\Br{\norm{Q(\xv) - \xv}_2^2} &= \sum_{j = 1}^{n}\mathbb{E}\Br{\Paren{\Qsf(x(j)) - x(j)}^2} \\
    &= \sum_{j = 1}^{n}\Paren{u_{r+1} - x(j)}^2\cdot \frac{x(j) - u_r}{u_{r+1} - u_r} + \Paren{x(j) - u(r)}^2\cdot \frac{u_{r+1} - x(j)}{u_{r+1} - u_r} \\
    &= \sum_{j = 1}^{n}\Paren{u_{r+1} - x(j)}\Paren{x(j) - u_r} \leq \sum_{j=1}^{n}\frac{\Paren{u_{r+1} - u_r}^2}{4} = \frac{n\norm{\xv}_{\infty}^2}{\Paren{2^R - 1}^2}.
\end{align*}
As can be seen from above, for the na\"ive stochastic uniform quantizer with an input that satisfies $\lVert \cdot \rVert_2 \leq B$, this gives us a worst-case upper bound on the quantization variance as:
\begin{equation*}
    \sigma_{q, naive}^2 \triangleq \frac{nB^2}{(2^R - 1)^2},
\end{equation*}
since $\lVert \xv \rVert_{\infty} = \lVert \xv \rVert_2 \leq B$ gives the worst-case input $\xv$.
With an application of Jensen's inequality, this gives us a convergence rate of \textbf{DQ-PSGD} that scales as follows:
\begin{equation}
    \sup_{f, \Ocal} \mathbb{E}f(\xv_T) - f(\xv^*) \lesssim O\Paren{\frac{1}{\sqrt{mT}}\cdot\frac{\sqrt{n}B}{\Paren{2^R - 1}}},
\end{equation}
We ignore the dependence on $\sigma_o^2$ in the above expression since it does not depend on the quantizer design.
The dependence of this convergence rate on the dimension $n$ can be detrimental for high dimensional problems.
We show that using our proposed source coding schemes, we can get rid of this linear dependence on $n$.

Our proposed algorithm \textbf{DSC}, first computes the democratic embedding, quantizes it the embedding using a stochastic uniform quantizer in the encoding process, and finally computes the inverse embedding while decoding.
Democratic embeddings have the property that $\lVert \xv_d \rVert_{\infty} \leq \frac{K_u}{\sqrt{n}}$ for some constant $K_u$ with high probability.
Substituting this $\ell_{\infty}$-norm in the expression of the quantizer variance yields,
\begin{equation*}
    \mathbb{E}\Br{\lVert Q(\xv_d) - \xv_d \rVert_2^2} = \frac{n\lVert \xv_d \rVert_{\infty}^2}{(2^R - 1)^2} \leq \frac{K_u^2}{(2^R - 1)^2}.
\end{equation*}
This gives us a dimension-independent bound on the quantizer variance, and consequently on the convergence rate, which now scales as,
\begin{equation}
    \sup_{f, \Ocal} \mathbb{E}f(\xv_{T}) - f(\xv^*) \lesssim O\Paren{\frac{1}{\sqrt{mT}}\cdot\frac{K_u}{(2^R - 1)}}.
\end{equation}
Similarly, if we use near-democratic embeddings $(\xv_{nd})$ (using a random orthonormal or a randomized Hadamard frame), we have, with high probability, $\lVert \xv_{nd} \rVert_{\infty} \leq O\left(\frac{\log n}{n}\right)$, which yields an upper bound on the convergence rate of \textbf{DQ-PSGD} as,
\begin{equation}
    \sup_{f, \Ocal}\mathbb{E}f(\xv_T) - f(\xv^*) \lesssim O\Paren{\frac{1}{\sqrt{mT}}\cdot \frac{\sqrt{\log n}}{(2^R - 1)}},
\end{equation}
giving us a mild logarithmic scaling instead of a (worse) $O(\sqrt{n})$ scaling as in the na\"ive case.
In what follows, we do some additional simulations for training models over multiple workers.
We consider a setup of multi-worker linear regression, wherein our goal is the solve the optimization problem:
\begin{equation}
    \xv^* \equiv \argminimize_{\xv \in \Real^n} \frac{1}{m}\sum_{i=1}^{m}\Paren{\frac{1}{s}\sum_{j=1}^{s}\frac{1}{2}\Paren{b_{ij} - \av_{ij}^{\top}\xv}^2}.
\end{equation}
Here, $\{\av_{ij}, b_{ij}\}_{j=1}^{s}$ denotes the local dataset at node $i$, for $i \in [m]$.
The dimension of the problem is $n = 30$, $m = 10$ workers with each worker having $s = 10$ local datapoints.
The learning rate of SGD is taken to be a constant $\alpha = 0.1$.
The dataset is generated synthetically from a model $\xv^*$ according to the planted model $\bv = \Av\xv^*$, where $\bv \in \Real^{ms}$ is the regression output and the rows of $\Av \in \Real^{ms \times n}$, i.e. $\{\av^{\top}_1, \ldots, \av^{\top}_{ms}\}$ are the data vectors.
In Fig. \ref{fig:dqpsgd_model_data_gaussian3}, the entries of $\xv^*$ and $\Av$ are generated by first drawing i.i.d. samples from $\Ncal(0,1)$ and raising them to the third power, i.e. $\Ncal(0,1)^3$.
All plots are averaged over $5$ independent trials and the synthetic data is also generated independently each time.
The two figures depict plots for two different values of $R = 0.5$ and $R = 1$ bits per dimension per user.
We also simulate for another heavy-tailed distribution of $\xv^* \sim \text{Student-t (df = 1)}$ and the entries of the data matrix $\Av \stackrel{iid}{\sim} \Ncal(0,1)$ in Fig. \ref{fig:dqpsgd_model_data_student-t}.

\begin{figure}[h!]
\begin{subfigure}[h!]{.48\textwidth}
    \centering
    \includegraphics[width=\linewidth]{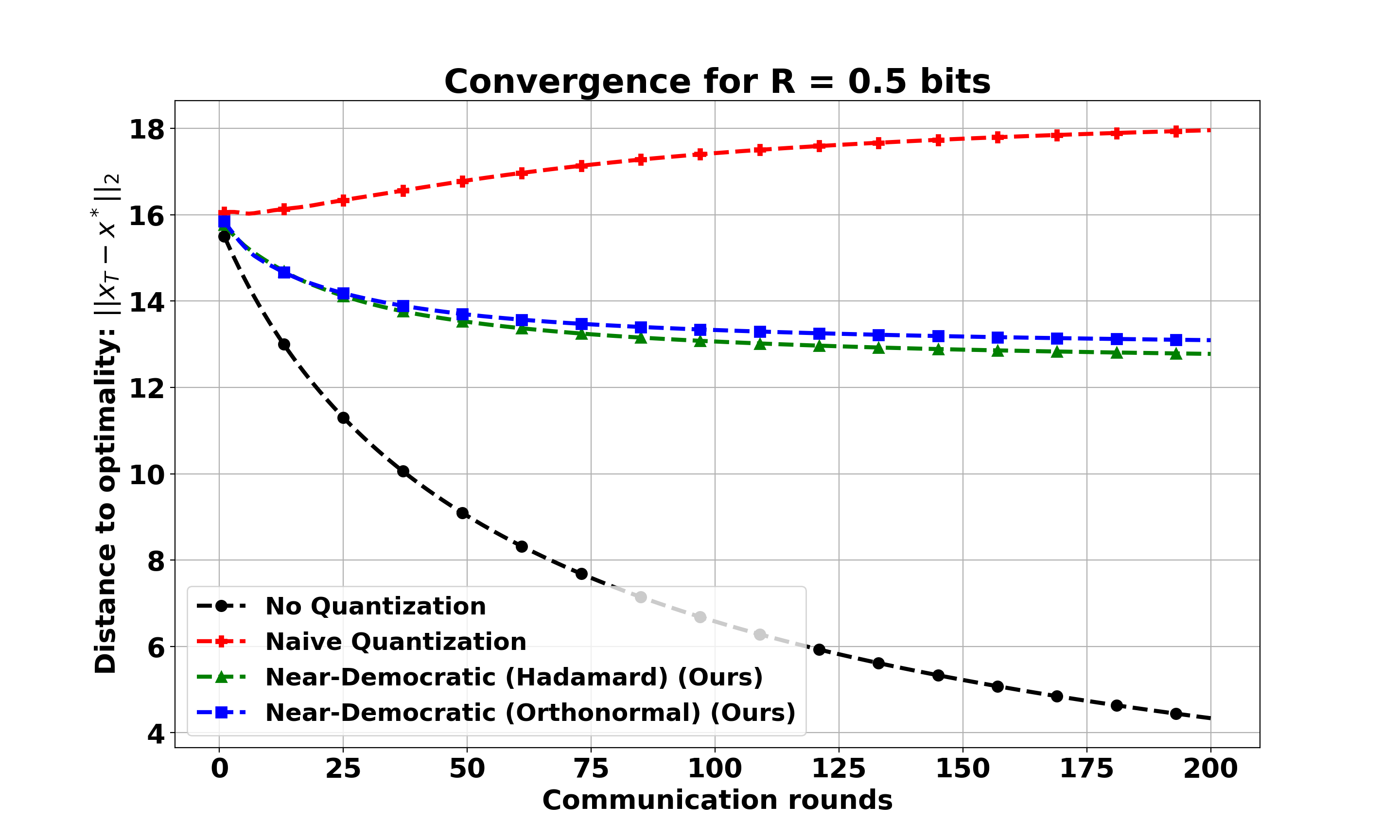}
  \end{subfigure}
  \hfill
  \begin{subfigure}[h!]{.48\textwidth}
    \centering
    \includegraphics[width=\linewidth]{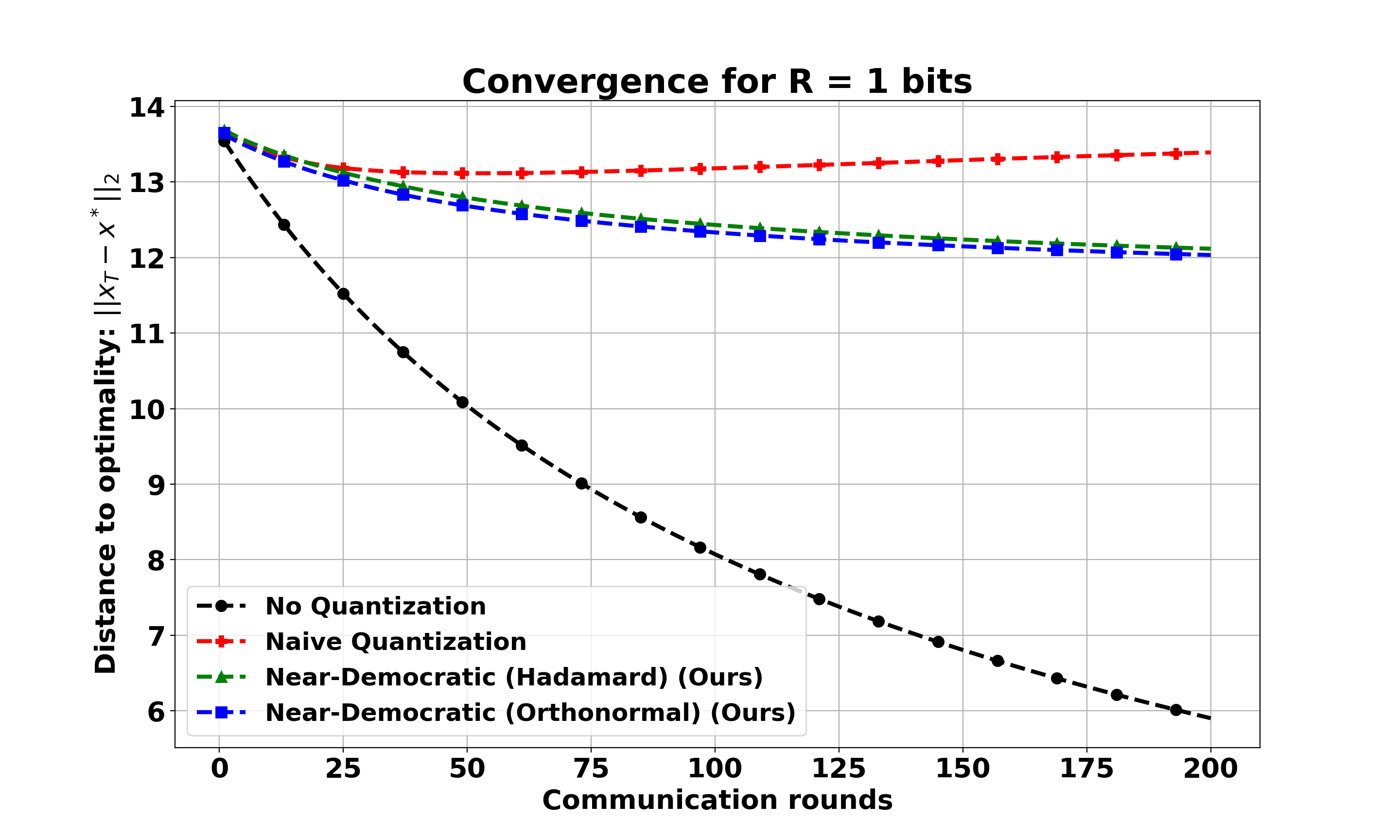}
  \end{subfigure}
  \caption{Multi-worker linear regression with $\xv^* \stackrel{iid}{\sim} \Ncal(0,1)^3$ and $\Av \stackrel{iid}{\sim} \Ncal(0,1)^3$.}
  \label{fig:dqpsgd_model_data_gaussian3}
\end{figure}

\begin{figure}[h!]
\begin{subfigure}[h!]{.48\textwidth}
    \centering
    \includegraphics[width=\linewidth]{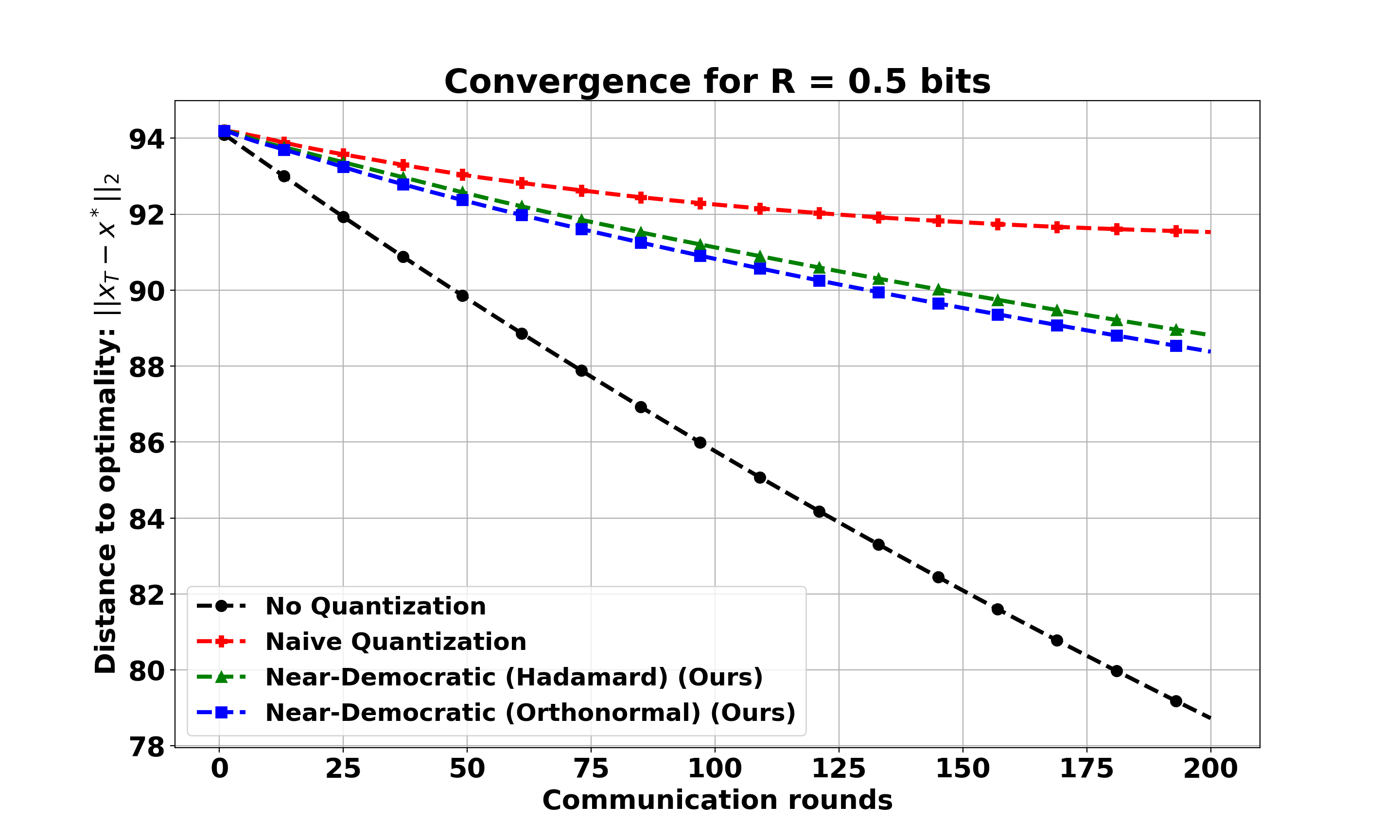}
  \end{subfigure}
  \hfill
  \begin{subfigure}[h!]{.48\textwidth}
    \centering
    \includegraphics[width=\linewidth]{figures/LinReg_Student-t_R1.png}
  \end{subfigure}
  \caption{Multi-worker linear regression with $\xv^* \stackrel{iid}{\sim} Student-t$ and $\Av \stackrel{iid}{\sim} \Ncal(0,1)$.}
  \label{fig:dqpsgd_model_data_student-t}
  \vspace{4mm}
\end{figure}

In addition to this, it is also possible to consider an extension of \textbf{DGD-DEF} to optimizing smooth and strongly convex objective functions over multiple workers and obtain faster convergence rates, by considering the extension in \cite[Sec. 5]{lin2020achieving} in which each worker does error feedback independently.
However, a complete characterization and extension of the idea of error feedback to multiple workers (in which the error feedback term compensates for errors from all workers) is still an open problem (as noted in \cite[Sec. V-C]{lin2020achieving}) and is beyond the current scope of our work.

We also train a simple neural network in the Federated Learning setup to do multi-class classification on the CIFAR-10 \cite{cifar10} image classification dataset that contains $50,000$ training and $10,000$ test images from $10$ classes.
We consider $m = 10$ workers and the entire dataset is distributed across these workers in a non-i.i.d. fashion, so that each worker has images from at most $2$ out of the $10$ classes.

Each worker locally trains a simple CNN model described next.
The CNN has a \textit{Convolution $\to$ Batch Normalization} layer, followed by a \textit{$2 \times 2$ max pooling layer}, followed by another \textit{Convolution $\to$ Batch Normalization} layer.
This is followed by three fully-connected layers.
We use \textbf{SGD} with global momentum at the parameter server to the train the model.
The learning rate is taken to be $0.05$, and the momentum parameter as $0.9$.
A coefficient of $1e-4$ is considered for $\ell_2$-regularization (or weight-decay regularization) to prevent overfitting.
The batch-size is taken to be $64$.
All simulations were carried out on NVIDIA GeForce GTX $1080$ Ti with a CUDA Version $11.4$.

\begin{figure}[h!]
    \centering
    \includegraphics[width=0.7\linewidth]{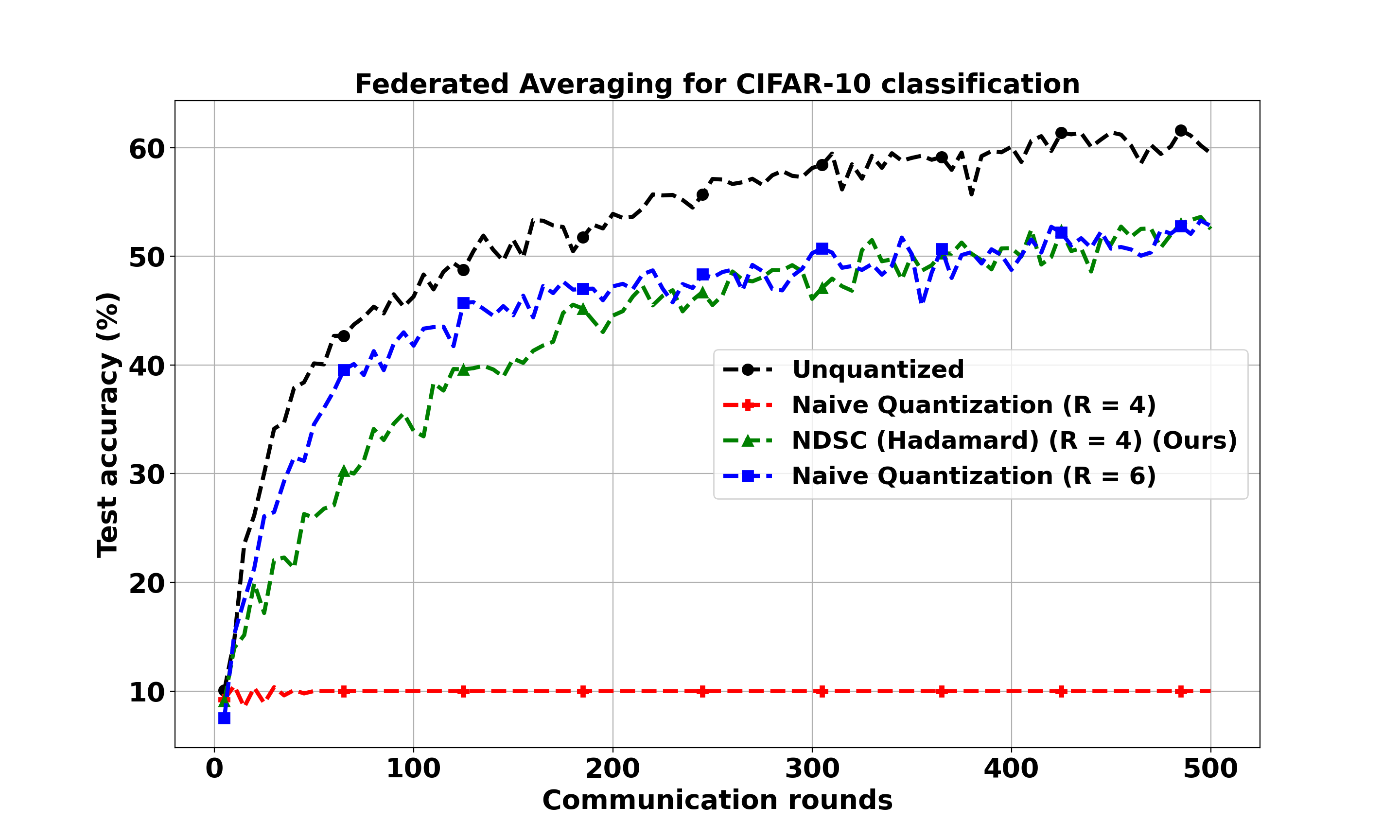}
    \caption{Multi-worker Distributed Optimization over CIFAR-10}
    \label{fig:simplecnn_cifar10}
    \vspace{4mm}
\end{figure}

The neural network simulation results are presented in Fig. \ref{fig:simplecnn_cifar10}.
Here, we plot the test-accuracy versus the number of communication rounds between the workers and the PS.
As can be seen from the plot, with a bit-budget of $R = 4$ bits per dimension per worker, our proposed near-democratic source coding (NDSC) scheme with randomized Hadamard frame outperforms na\"ive quantization with the same bit-budget ($R = 4$) which fails to even converge.
As a matter of fact, na\"ive quantization requires a higher-budget of $R = 6$ bits per dimension per worker to achieve a performance comparable to that of near-democratic source coding (NDSC).

\section{Comparison of Different Classes of Randomized Frames for Computing Democratic Embeddings}
\label{app:frames_comparison}

The \textit{Uncertainty Principle (UP)} for random matrices (\textbf{Definition} \textcolor{blue}{$2$} in the main paper), is closely related to the \textbf{Restricted Isometry Property (RIP)} for random matrices introduced in \cite{candes_tao_RIP}.
\begin{definition}(\textbf{Restricted Isometry Property (RIP)})
\label{def:RIP}
A matrix $\Sv \in \Real^{n \times N}$ is said to satisfy the \textit{Restricted Isometry Property of order $k$} with constant $\epsilon > 0$, if for every $k$-sparse vector $\xv \in \Real^N$ (i.e. a vector with at most $k$ non-zeros entries), the following holds true:
\begin{equation}
    (1 - \epsilon)\norm{\xv}_2^2 \leq \norm{\Sv \xv}_2^2 \leq (1 + \epsilon)\norm{\xv}_2^2
\end{equation}
\end{definition}
Comparing Definition \ref{def:RIP} above with the definition of UP (Definition \textcolor{blue}{$2$} in the main paper), it can be seen that any matrix that satisfies RIP with parameters $(k, \epsilon)$, satisfies UP with parameters $\left(\sqrt{1 + \epsilon}, \frac{k}{N}\right)$.
In other words, randomly constructed frames which satisfy the UP with high probability, are strongly related to sensing matrices with small restricted isometry constants.
Since the \textit{upper Kashin constant} $K_u$ is given by $K_u = \frac{\eta}{\left(A - \eta\sqrt{B}\right)\sqrt{\delta}}$, in order to achieve embeddings having small $\ell_{\infty}$ norm, one is therefore interested in finding frames satisfying the UP with small $\eta$ and large $\delta$.

\subsection{Sub-Gaussian Random Matrices}
\label{subsec:subGassian_random_matrices}

\begin{definition}(\textbf{Sub-Gaussian random variable})
A random variable $X$ is called sub-Gaussian with parameter $\beta$ if
\begin{equation}
    \Pr\left[|X| > u\right] \leq \exp\left(1 - \frac{u^2}{\beta^2}\right) \hspace{2mm} \text{for all} \hspace{2mm} u > 0.
\end{equation}
\end{definition}

From \cite{lyubarskii_2010}, we have the following result regarding the uncertainty principle parameters of random sub-Gaussian matrices:

\begin{theorem}\textbf{\cite[Theorem 5]{lyubarskii_2010}: (Uncertainty principle for sub-Gaussian matrices)}
Let the entries of $\Stv \in \Real^{n \times N}$ be i.i.d. zero-mean sub-Gaussian random variables with parameter $\beta$.
Let $\lambda = N/n$ for some $\lambda \geq 2$.
Then with probability at least $1 - \lambda^{-n}$, the random matrix $\Sv = \frac{1}{\sqrt{N}}\Stv$ satisfies the UP with parameters:
\begin{equation}
    \eta = C_0 \beta\sqrt{\frac{\log\left(\lambda\right)}{\lambda}}, \hspace{2mm} \text{and} \hspace{2mm} \delta = \frac{C_1}{\lambda},
\end{equation}
where, $C_0, C_1 > 0$ are absolute constants.
\end{theorem}

\textbf{Note}: Random sub-Gaussian matrices are \textit{NOT} tight frames.
However, \cite[Lemma 4.8]{lyubarskii_2010} and \cite[Corollary 4.9]{lyubarskii_2010} show that $\Sv$ (as constructed above) is an approximate Parseval frame with high probability, i.e. $\Sv$ has frame bounds $A = 1 - \xi$ and $B = 1 + \xi$ for some small $\xi \in (0,1)$.
Substituting in the values to obtain the upper Kashin constant, we have,
\begin{align}
    K_u = \frac{\eta}{\left(A - \eta\sqrt{B}\right)\sqrt{\delta}} &= C_0\beta\sqrt{\frac{\log(\lambda)}{\lambda}}\cdot\sqrt{\frac{\lambda}{C_1}}\cdot \frac{1}{1 - \xi - \left(C_0\beta\sqrt{\frac{\log(\lambda)}{\lambda}}\right)\sqrt{1 + \xi}} \nonumber \\
    &= \frac{C_0\beta}{\sqrt{C_1}}\cdot\frac{\log(\lambda)}{1 - \xi - \left(C_0\beta\sqrt{\frac{\log(\lambda)}{\lambda}}\right)\sqrt{1 + \xi}}
\end{align}

Note that the upper Kashin constant $K_u$ has a logarithmic dependence on the aspect ratio of the frame, $\lambda$.
Since $\lambda$ is essentially treated as a constant, $K_u$ is a constant.
However, the $32$-bit floating point entries of a sub-Gaussian matrix need to be stored in the memory and furthermore, multiplying the dense matrix $\Sv \in \Real^{n \times N}$ by a vector $\xv \in \Real^N$ can take up to a worst case of $O\left(nN\right)$ time.
Hence, choosing sub-Gaussian frames can be intensively computation as well as memory demanding.

\subsection{Random Orthonormal Matrices}
\label{subsec:random_orthonormal_matrices}

Random orthonormal matrices refer to $n \times N$ matrices whose rows are orthonormal. 
Such matrices can be obtained by randomly selecting the first $n$ rows of a randomly generated $N \times N$ orthonormal matrix.
We state a result from \cite{lyubarskii_2010}.
If $\Scal(N)$ denotes the space of all orthogonal $N \times N$ matrices with the normalized Haar measure, then $\Scal(n \times N) = \left\{ \Pv_n\Vv; \Vv \in \Scal(N)\right\}$, where $\Pv_n:\Real^N \to \Real^n$ is the orthogonal projection on the first $n$ coordinates, denotes the space of all $n \times N$ orthonormal matrices.
The probability measure on $\Scal(n \times N)$ is induced by the Haar measure on $\Scal(N)$.

\begin{theorem}\textbf{\cite[Theorem 4.1]{lyubarskii_2010}: (UP for Random Orthonormal Matrices)}
Let $\mu > 0$ and $N = (1 + \mu)n$.
Then with probability at least $1 - 2\exp\left(-c\mu^2n\right)$, a random orthonormal $n \times N$ matrix $\Sv$ satisfies the uncertainty principle with parameters
\begin{equation}
    \eta = 1 - \frac{\mu}{4} \hspace{2mm} \text{and}, \hspace{2mm} \delta = \frac{c\mu^2}{\log(1/\mu)}
\end{equation}
where $c > 0$ is an absolute constant.
\end{theorem}

Note that random orthonormal matrices are Parseval frames, i.e. $A = B = 1$.
So, the upper Kashin bound is given by:
\begin{align}
\label{eq:Ku_random_orthonormal}
    K_u = \frac{1 - \frac{\mu}{4}}{\mu/4}\cdot\frac{\sqrt{\log(1/\mu)}}{\mu\sqrt{c}} = \frac{4}{\mu^2\sqrt{c}}\left(1 - \frac{\mu}{4}\right)\sqrt{\log\left(\frac{1}{\mu}\right)}
\end{align}

This once again shows that $K_u$ is a constant.

\subsection{Random Matrices with Fast Transforms}
\label{subsec:random_matrices_fast_transforms}

\textit{Matrices with fast transforms} refer to structured matrices like Fourier transform matrix, etc. for which fast algorithms exist to compute the forward transform efficiently.
To analyze the RIP properties of our proposed randomized Hadamard matrix construction, $\Sv = \Pv \Dv \Hv \in \Real^{n \times N}$, we recall the following theorem from \cite{haviv_reged}.

\begin{theorem}\textbf{\cite[Theorem 4.5]{haviv_reged}}
\label{thm:fast_matrices_RIP}
For sufficiently large $N$ and $k$, a unitary matrix $\Mv \in \Real^{N \times N}$ with entries of absolute value $O(1/\sqrt{N})$, and for sufficiently small $\epsilon > 0$, the following holds:
For some $n = O\left(\log^2(1/\epsilon)\epsilon^{-2}\cdot k \cdot \log^2\left(k / \epsilon\right)\cdot \log N\right)$, let $\Stv \in \Real^{n \times N}$ be a matrix whose $n$ rows are chosen uniformly at random and independently from the rows of $\Mv$, multiplied by $\sqrt{N/n}$.
Then, with probability $1 - 2^{\Omega\left(\log N \cdot \log\left(k/\epsilon\right)\right)}$, the matrix $\Stv$ satisfies the restricted isometry property of order $k$ with constant $\epsilon$.
\end{theorem}

The upper Kashin constant for such a matrix can be computed approximately as follows:
\begin{equation}
\label{eq:haviv_reged_RIP_1}
    n = O\left(\log^2\left(\frac{1}{\epsilon}\right)\epsilon^{-2} k \log^2\left(\frac{k}{\epsilon}\right)\log N \right) \implies k \log^2\left(\frac{k}{\epsilon}\right) = \Omega\left(\frac{n}{\log^2(1/\epsilon)\epsilon^{-2}\log N}\right).
\end{equation}

Furthermore, since the definition of UP requires $\norm{\Sv\xv}_2 \leq \eta\norm{\xv}_2$, it is apparent that if we scale the entries of the matrix $\Sv$ by a constant factor, the parameter $\eta$ of UP also scales by the same factor.
So, in Thm. \ref{thm:fast_matrices_RIP}, if we multiply the matrix $\Stv$ by $\sqrt{\frac{n}{N}}$ to negate the effect of multiplying by $\sqrt{\frac{N}{n}}$ as in the statement of the theorem in order to get $\Sv = \sqrt{\frac{n}{N}}\cdot \Stv$, then the corresponding UP parameters of $\Sv$ are:
\begin{equation}
    \eta = \sqrt{\frac{n}{N}\left(1  + \epsilon\right)} \hspace{2mm} \text{and}, \hspace{2mm} \delta = \frac{k}{N},
\end{equation}
where $k$ is dictated by \eqref{eq:haviv_reged_RIP_1}.
\begin{align}
    K_u &= \frac{\sqrt{\frac{n}{N}(1 + \epsilon)}}{1 - \sqrt{\frac{n}{N}(1 + \epsilon)}}\cdot \sqrt{\frac{N}{k}} \nonumber \\
    &= \frac{\sqrt{\frac{n}{N}(1 + \epsilon)}}{1 - \sqrt{\frac{n}{N}(1 + \epsilon)}}\cdot \frac{\sqrt{N}\log\left(k/\epsilon\right)}{\sqrt{k \log^2\left(k / \epsilon\right)}} \nonumber \\
    &= O\left(\frac{\sqrt{\frac{n}{N}(1 + \epsilon)}}{1 - \sqrt{\frac{n}{N}(1 + \epsilon)}} \cdot\log\left(\frac{k}{\epsilon}\right)\cdot \sqrt{\frac{N}{n}} \cdot \log\left(\frac{1}{\epsilon}\right)\cdot \frac{1}{\epsilon}\cdot \sqrt{\log N} \right) \nonumber \\
    &= O\left(\frac{\sqrt{1 + \epsilon}}{1 - \sqrt{\frac{n}{N}(1 + \epsilon)}}\cdot\log\left(\frac{k}{\epsilon}\right)\cdot \log\left(\frac{1}{\epsilon}\right)\cdot \sqrt{\log N}\right) = O\left(\sqrt{\log N}\right).
\end{align}

We see that our proposed randomized Hadamard construction $\Sv = \Pv \Dv \Hv$, which provides computational and memory savings, as seen in the main paper falls in this class of \textit{random matrices with fast transforms}.
However, for such matrices, since $K_u$ is not a constant and scales as $O\left(\sqrt{\log N}\right)$, when $n$ is large, computing democratic embedding does not yield much benefit as compared to the \textit{near-democratic embedding}.
From Lemma \textcolor{blue}{$3$} of the main paper, the $\ell_{\infty}$-norm of the \textit{near-democratic embedding} $\overline{\xv} \in \Real^N$ of $\yv \in \Real^n$ also scales as $\norm{\overline{\xv}}_{\infty} \leq \frac{O\left(\sqrt{\log N}\right)}{\sqrt{N}}\norm{\yv}_2$.

In other words, for random matrices with fast transforms, asympotically (i.e. for large $n$), near-democratic embeddings have $\ell_{\infty}$-norm as good as democratic embeddings.
However, unlike sub-Gaussian and random orthonormal frames, since the entries of $\Sv = \Pv \Dv \Hv $ are $\pm\frac{1}{\sqrt{N}}$, only the signs need to be stored (which amounts to $1$-bit per entry, hence saving memory requirements), and $\overline{\xv} = \Sv^{\top}\yv$ can effectively be obtained in $O(n \log n)$ additions, which is much less computationally demanding.

\section{Proof of Lemma \texorpdfstring{\textcolor{blue}{$\mathbf{4}$}}{four}: Covering efficiency: DSC \& NDSC}
\label{app:proof_covering_efficiency}

Suppose the input to the encoders satisfy $\norm{\yv}_2 \leq r$.
From Thm. \textcolor{blue}{$1$}, $\norm{\yv - \Qsf_d(\yv)}_2 \leq 2^{\left(1 - R/\lambda\right)}r$.
Since an $R$-bit quantizer has a range of finite cardinality $|\Rcal'| \leq 2^{nR}$, using eq. \textcolor{blue}{$(15)$}, we have $ \rho_d = \left(2^{nR}\right)^{\frac{1}{n}}\frac{2^{(1 - R/\lambda)}K_ur}{r} = 2^{1+R\left(1 - \frac{1}{\lambda}\right)}K_u$.
Similarly, the expression for $\rho_{nd}$ follows.

\section{Additional Discussions and Remarks}
\label{sec:additional_discussions_and_remarks}

\subsection{Discussions on Other Related Works}

A succinct comparison of \textbf{DSC} and \textbf{NDSC} with some of the existing quantization/sparsification schemes is provided in Table \textcolor{blue}{I}.
We characterize each compression scheme by its bit requirement and the normalized error.
For any input vector $\xv \in \Real^n$ with $\norm{\xv}_2 = 1$, let the output of a compression scheme be denoted as $\Ccal(\xv)$.
Then the normalized error is defined to be $\norm{\Ccal(\xv) - \xv}_2$.\\

\textbf{Sign Quantization \cite{bernstein_icml_2018, karimireddy_2019_error_feedback}}. Sign quantization strategies quantize each coordinate of the quantizer input $\xv$ to either $+1$ or $-1$ depending on its sign.
With proper magnitude scaling, they are effective in quantizing $\xv$ using only $1$-bit per dimension.
To get an expression for the normalized error, if we additionally assume that each coordinate of $\xv$ is bounded, i.e. $\norm{\xv}_{\infty} \leq B$, then the error incurred can be upper bounded as $(B-1)^2n \sim O(n)$, where $n$ is the dimension of $\xv$.
A similar situation arises with \textbf{QSGD} when more than $1$-bit are used to quantize each coordinate of $\xv$.
The situation is similar for \textbf{TernGrad} \cite{terngrad_neurips_2017}, that uses three numerical levels $\{-1, 0, +1\}$ to quantize the gradients for distributed training of neural networks.\\

\textbf{QSGD \cite{alistarh_NeurIPS_2017_qsgd}}. QSGD uses $s$ quantization levels for each coordinate.
If we were to use $R$-bits per dimension, then we have $s = 2^R$.
\cite[Thm. 3.2]{alistarh_NeurIPS_2017_qsgd} specifies the bit requirement and normalized error.
This result provides an upper bound on the \textbf{expected} bit requirement, since it is a \textbf{Variable-length code}.
Our work considers \textbf{fixed-length quantizers} wherein the bit-budget per dimension ($R$) is pre-specified as a constraint.
Moreover, we study optimality of fixed-length quantization schemes for all values of $R \in (0, \infty)$.
In contrast to this, QSGD only considers $R \approx 2.8$.
Furthermore, for worst-case inputs $\xv$, the performance of QSGD is far from optimal, whereas our quantizers are designed to guarantee minimax optimal performance.\\

\textbf{Vector Quantized SGD \cite{gandikota2020vqsgd}}. \textit{vqSGD} constructs a convex hull using a finite number of points and quantize their vector to one of the points randomly.
They also consider both communication and privacy parameters of their coding scheme simultaneously.
From \cite[Table 2]{gandikota2020vqsgd}, the error of \textit{vqSGD} scales as $O(n)$ with the dimension $n$, as compared to $O(2^{-2R/\lambda})$ and $O(2^{-2R/\lambda}\log n)$ for \textbf{DSC}/\textbf{NDSC} respectively.
Here, $\lambda$ is a constant.\\

\textbf{Top-$k$ sparsification} \cite{stich_2018_sparsifiedSGD} is a common gradient sparsification technique.
Here, $k$ denotes the number of retained coordinates.
The bit-requirement arises from the need to communicate the values of the $k$ chosen coordinates, and the $\log_2\binom{n}{k}$ comes from the number of the ways $k$ out of $n$ indices can be chosen, which also needs to be conveyed.
This source coding scheme requires a computation of $O(k + (n-k)\log_2 k)$ due to the fact that the coordinates must be sorted in order to be able to choose the top $k$ out of them.
\textbf{Random sparsification} \cite{wangni_2018} is another simple strategy which randomly chooses the $k$ out of $n$ indices.
Gradient sparsification and quantization are complementary to each other in reducing the communication requirements.
We show \textbf{DE}/\textbf{NDE} can be used in conjunction with existing gradient sparsification and quantization techniques to further improve the performance of the latter.\\

\textbf{SimQ+ \cite{mayekar_SimQ+}}. SimQ+ addresses the problem of designing quantizers for optimization over general $\ell_p$-spaces.
SimQ+ relies on a more fundamental quantizer, SimQ that exploits the fact that any point inside the unit $\ell_1$ ball of $\Real^n$ can be represented as a convex combination of at most $2n$ points.
SimQ then samples one of these points.
In this sense, it is similar to vqSGD.
For Euclidean spaces, i.e. $p = 2$ (considered in our work), \cite[Thm. 3.1]{mayekar_SimQ+} shows that SimQ+ uses $\sim 3n$ bits to achieve the optimal suboptimality gap scaling of $O\Paren{\frac{1}{\sqrt{T}}}$ for general convex and non-smooth functions.
The computational complexity of SimQ+ is $O(n^2)$.
This is because SimQ+ requires $k$ repetitions of SimQ.
SimQ involves sampling from an $(n+1)$-dimensional probability distribution \cite[Alg. 2]{mayekar_SimQ+}.
Efficient methods for sampling from high-dimensional distributions (like Walker Alias method) entails $O(n)$ complexity.
Moreover for $\ell_2$-spaces, SimQ+ requires $k = n^{2/p} = n$ repetitions of SimQ.
This results in a total complexity of $O(n^2)$.
Consequently, for $R \approx 3$, the performance of our proposed quantizers is the same as SimQ+.
However, for any other value of $R$, SimQ+ is no longer optimal as noted in \cite[\S 7]{mayekar_SimQ+}.
In order to make SimQ+ operate under \textit{any arbitrary precision constraint} $R \neq 3$, requires us to modify it by incorporating heuristics like uniform sampling without replacement.
But doing so causes the SimQ+ quantizer design to lose optimality with undesirable logarithmic factors appearing between upper and lower bounds.
Although the gap is practically insignificant, our quantizer design using democratic embeddings resolves this open question for $\ell_2$ spaces.
This is especially true in the sublinear bit-budget regime, when $R < 1$.\\

\textbf{RATQ \cite{mayekar_2020}}. The \textbf{Rotated Adaptive Tetra-iterated Quantizer} (RATQ) involves Hadamard transforms and appropriately chosen quantizer dynamic ranges to achieve a performance which is $O(\log\log(\cdot))$ factor close to optimal for stochastic optimization, with a computational complexity of $O(n\log n)$.
They consider the optimization of general convex and non-smooth objective functions with access with a stochastic subgradient oracle, which corresponds to Thm. \textcolor{blue}{$3$} of our work.
Compared to this, our proposed scheme \textbf{DSC} (with random orthonormal frame) achieves a performance to within a constant factor of optimality with a complexity of $O(n^2)$, and \textbf{NDSC} achieves a performance $O(\log (\cdot))$ factor away from optimality while entailing a $O(n\log n)$ complexity.
Hence from a high-level perspective, RATQ has a performance somewhere in between \textbf{DSC} and \textbf{NDSC} for stochastic optimization.
Note that since RATQ makes use of standard dithering \cite[Alg. 5]{mayekar_2020}, it is an inherently stochastic quantizer.
Consequently, it is not straightforward to see how RATQ could be applied to settings which require deterministic quantizers such as the optimization of smooth, strongly-convex objective functions in our work.
For analyzing deterministic settings, we require a uniform upper bound on $\norm{\Qsf(\xv) - \xv}_2$ (i.e. without any expectation) for any quantizer $\Qsf(\cdot)$ and input $\xv$.
\textbf{DSC} and \textbf{NDSC} can be designed for deterministic quantizers as well (cf. Eq. \textcolor{blue}{$(11)$}), which does nearest-neighbor quantization.

\section{A note on randomized frame constructions}
\label{note_randomized_frame_constructions}

The definition of a frame as given in Def. \textcolor{blue}{$1$} of the main paper is quite general \cite{frames_book}.
A frame $\Sv \in \Real^{n \times N}$ is the same as a wide matrix with full row rank, i.e. $n$.
More specifically, a full rank matrix $\Sv$ with singular values $0 < \sigma_{min} \triangleq \sigma_0, \sigma_1, \ldots, \sigma_n \triangleq \sigma_{max} < \infty$ satisfies
\begin{equation*}
    \sigma_{min}^2\lVert \yv \rVert_2^2 \leq \lVert \Sv^{\top}\yv \rVert_2^2 \leq \sigma_{max}^2\lVert \yv\rVert_2^2,
\end{equation*}
for any $\yv \in \Real^n$.
That is, $A = \sigma_{min}^2$ and $B = \sigma_{max}^2$ are the upper and lower frame bounds respectively.
The concept of frames is a very widely used concept in linear algebra and not specific to our work.
For an inner-product space, a frame is a set of (possibly) linearly-dependent vectors that act as a generalization of a basis.
Because it consists of linearly dependent vectors, the representation of any vector in terms of a frame is \textit{not unique}.
Consequently, frames provide a robust way of representing a vector due to this redundancy.

Although $n \times N$ frames are abundant, what is more interesting to our work are frames that satisfy the uncertainty principle as in Def. \textcolor{blue}{$2$} in the main paper.
Our randomized frame constructions satisfy the uncertainty principle for \textit{appropriate values of parameters} $(\eta, \delta)$ such that the upper Kashin constant $K_u = \frac{\eta}{(A - \eta\sqrt{B})\sqrt{\delta}}$ (as defined in Lemma \textcolor{blue}{$1$}) is a small constant.

In what follows, we show a full rank matrix that satisfies all the properties of a frame, but does not satisfy uncertainty principle with desirable values of the parameters $(\eta, \delta)$.
As a consequence, the upper Kashin constant $K_u$ is not small and the inequality $(6)$ is Lemma $1$ is vacuous.
Consider the matrix:
\begin{equation*}
    \Sv = 
    \begin{bmatrix}
    1 & 0 & 0 & \ldots & 0 & 0\\
    0 & 1 & 0 & \ldots & 0 & 0\\
    \vdots & \vdots & \vdots & \vdots & \vdots & \vdots\\
    0 & 0 & \ldots & 1 & \ldots & 0
    \end{bmatrix},
\end{equation*} 
i.e. the first $n$ rows of the identity matrix $\Iv_N$.
The definition of uncertainty principle (Def. $2$) requires $\delta \in (0,1)$.
For any $\delta \in \Paren{\frac{n}{N}, 1}$, consider $\xv = [\hspace{-2mm}\underbrace{1 \ldots 1}_{\delta N \text{ indices}} 0 \ldots  0]^{\top}$.
Clearly, $\lVert \xv \rVert_2 = \sqrt{\delta N}$, and $\Sv\xv = [1 \hspace{2mm} 1 \hspace{2mm} \ldots \hspace{2mm} 1]^{\top} \in \Real^{n \times 1}$, so that $\lVert \Sv \xv \rVert_2 = \sqrt{n}$.
Since we require the parameter $\eta$ to be such that $\lVert \Sv \xv \rVert_2 \leq \eta \lVert \xv \rVert_2$, or $\sqrt{n} \leq \eta \sqrt{\delta N}$, this implies $\eta \geq \sqrt{\frac{n}{N}}$.
Also, $\eta \leq 1$ holds trivially from the fact that $\Sv$ is Parseval, i.e. all the singular values of $\Sv$ are $1$, meaning $\eta \in \left[\sqrt{\frac{n}{N}}, 1\right]$.
Furthermore, $\Sv$ is Parseval also implies, $\sigma_{min} = \sigma_{max} = 1$, i.e. we have $A = B = 1$.
When $N = n$, it follows that,
\begin{equation}
    A - \eta\sqrt{B} = 0 \implies K_u = \infty,
\end{equation}
which undesirably gives a vacuous upper bound in eq. $(6)$ of Lemma $1$ since $(A - \eta\sqrt{B})$ appears in the denominator of the expression for $K_u$.
When $N > n$, we have $N \geq n + 1 \implies \eta \in \left[\sqrt{\frac{n}{n+1}}, 1\right]$.
For very large $n$, $\eta$ gets squeezed to $1$, implying that as $n \to \infty$, $\eta \to 1 \implies K_u \to \infty$, which again yields a non-informative upper bound for high-dimensional problems.
Similarly, for any $\delta \in \Paren{0, \frac{n}{N}}$, if we let $\xv = [\hspace{-2mm}\underbrace{1 \ldots 1}_{\delta N \text{ indices}} 0 \ldots  0]^{\top}$, we have $\Sv\xv = [\hspace{-1mm}\underbrace{1 \ldots 1}_{\delta N \text{ indices}} 0 \ldots 0]^{\top} \implies \lVert \Sv\xv \rVert_2 = \sqrt{\delta N}$.
Since the definition of uncertainty principle requires $\lVert \Sv\xv \rVert_2 \leq \eta \lVert \xv \rVert_2 \implies \sqrt{\delta N} \leq \eta \sqrt{\delta N} \implies \eta \geq 1$.
Following the same arguments as before, $\eta = 1$ and $K_u = \infty$ giving us a vacuous upper bound for any value of $\delta \in (0,1)$.\\

Constructions for $\Sv \in \Real^{n \times N}$ like the ones described above are full rank (and also Parseval i.e. satisfy $\Sv^{\top}\Sv = \Iv_n$) and are valid frames according to Def. 1.
However, they do not serve any useful purpose for our proposed source coding schemes since they do not satisfy the uncertainty principle (Def. 2) for appropriate values of $(\eta, \delta)$ so as to yield a small value of the upper Kashin constant $K_u$.

\section{A note on why we set \texorpdfstring{$\lambda \triangleq \frac{N}{n}$}{lambdaexpression} as close to \texorpdfstring{$1$}{one} as possible}
\label{sec:choice_of_lambda}

From a high-level perspective, there is a tradeoff associated with the choice of $N$.
As a consequence of Lemmas \textcolor{blue}{$1$}, \textcolor{blue}{$2$} \& \textcolor{blue}{$3$}, a larger value of $N$ implies a smaller $\ell_{\infty}$-norm of the vector input to the quantizer.
However, since we are now quantizing a vector in $\Real^N$ instead of a vector in $\Real^n$ and $N \geq n$, the effective number of bits per dimension reduces from $R$ to $\frac{nR}{N}$.
The quantization error per dimension is proportional to $\frac{\lVert \cdot\rVert_{\infty}}{2^{nR/N}}$, where the numerator, $\lVert \cdot \rVert_{\infty}$ is the $\ell_{\infty}$-norm of the quantizer input which decreases with increasing $N$, and the denominator also decreases with an increasing $N$, implying an optimal choice of $N$ that minimizes this ratio.
We do a numerical study of this tradeoff.

The main idea of (near) democratic source coding is to compute a (near) democratic embedding and subsequently quantizing the embedding so as to reduce the $\ell_{\infty}$-norm of the input to a uniform scalar quantizer.
For instance, for \textit{near-democratic embeddings}, Lemmas \textcolor{blue}{$2$} \& \textcolor{blue}{$3$} tell us that larger the value of $N$, the smaller the $\ell_{\infty}$-norm should be, since $\sqrt{\frac{\log(2N)}{N}}$ is a decreasing function of $N$ and $\lVert \xv_{nd}\rVert_{\infty} \leq 2\sqrt{\frac{\log(2N)}{N}}\lVert \yv \rVert_2$ holds true with a probability exceeding $1 - \frac{1}{2N}$, which is an increasing function of $N$ -- both of which are desirable consequences of increasing $N$.
We verify this via numerical simulations in Figs. \ref{fig:linf_norm_gaussian3} and \ref{fig:linf_norm_student-t}.
For these plots, the vector $\yv \in \Real^n$ is drawn from $\text{Gaussian}^3$ and $\text{Student-t}$ distributions respectively, and the plots are averaged over $50$ realizations.
These heavy-tailed distributions ensure that the different coordinates of $\yv$ have widely varying magnitudes.
The embedding matrix is $\Sv = \Pv \Dv \Hv \in \Real^{n \times N}$ and we plot the $\ell_{\infty}$-norm of the near-democratic embedding $\xv_{nd}$ vs. the embedding dimension $N$.
The original dimension $n = 30$ is kept fixed and we choose $N$ to be powers of $2$, i.e. $N = 2^5, 2^6, \ldots, 2^{15}$ for which Hadamard matrix $\Hv \in \Real^{N \times N}$ can be constructed.
As we discussed already, $\lVert \xv_{nd} \rVert_{\infty}$ decreases with increasing $N$.
The original $\ell_{\infty}$-norm (averaged over $50$ realizations) is $\lVert  \yv \rVert_{\infty} = 13.57$ for Fig. \ref{fig:linf_norm_gaussian3} and $\lVert \yv \rVert_{\infty} = 52.96$ for Fig. \ref{fig:linf_norm_student-t}.\\

\begin{figure}[h!]
    \begin{subfigure}[h!]{.5\textwidth}
    \centering
    \includegraphics[width=\linewidth]{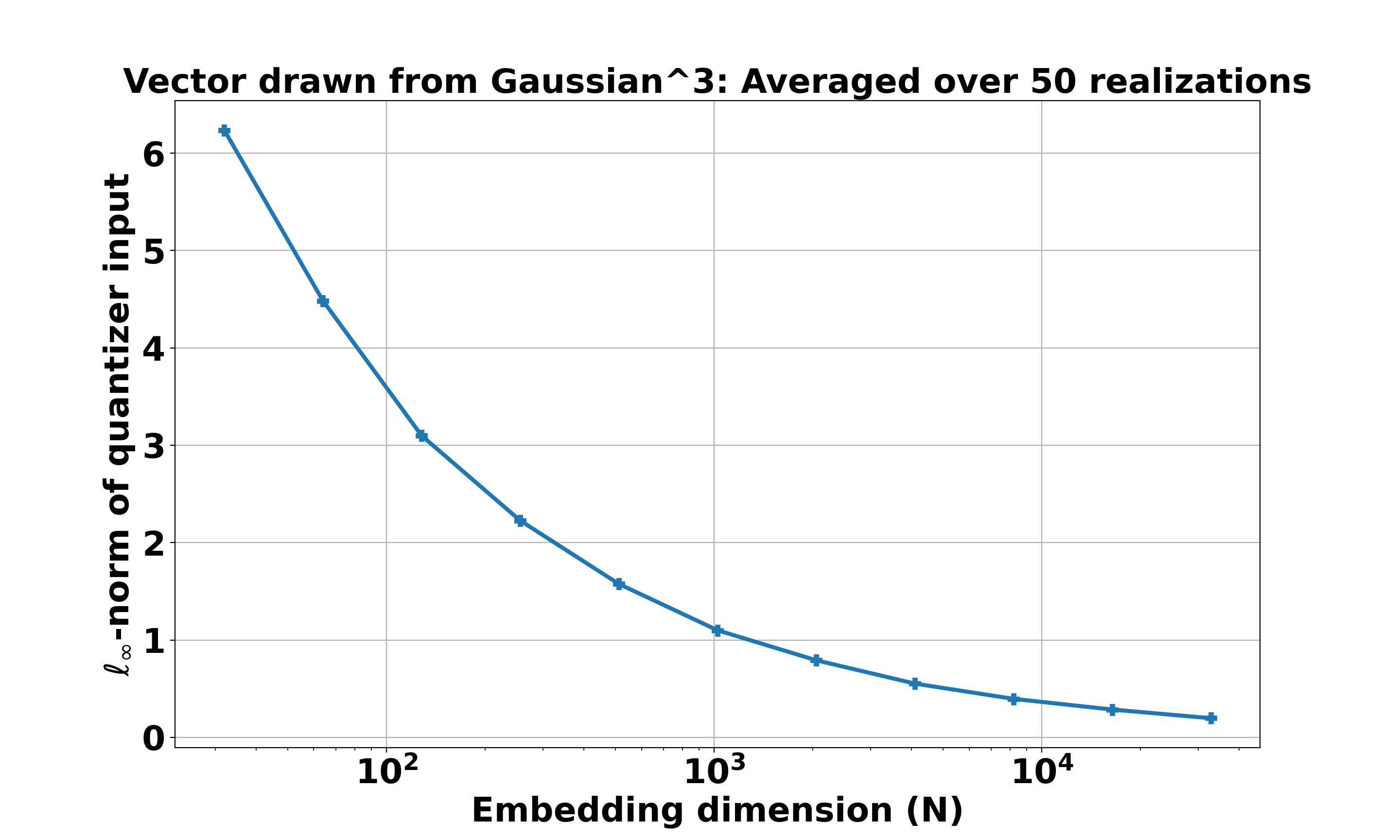}
    \caption{$\lVert \xv_{nd} \rVert_{\infty}$ for $\yv \sim \text{Gaussian}^3$}
    \label{fig:linf_norm_gaussian3}
  \end{subfigure}
  \hfill
  \begin{subfigure}[h!]{.5\textwidth}
    \centering
    \includegraphics[width=\linewidth]{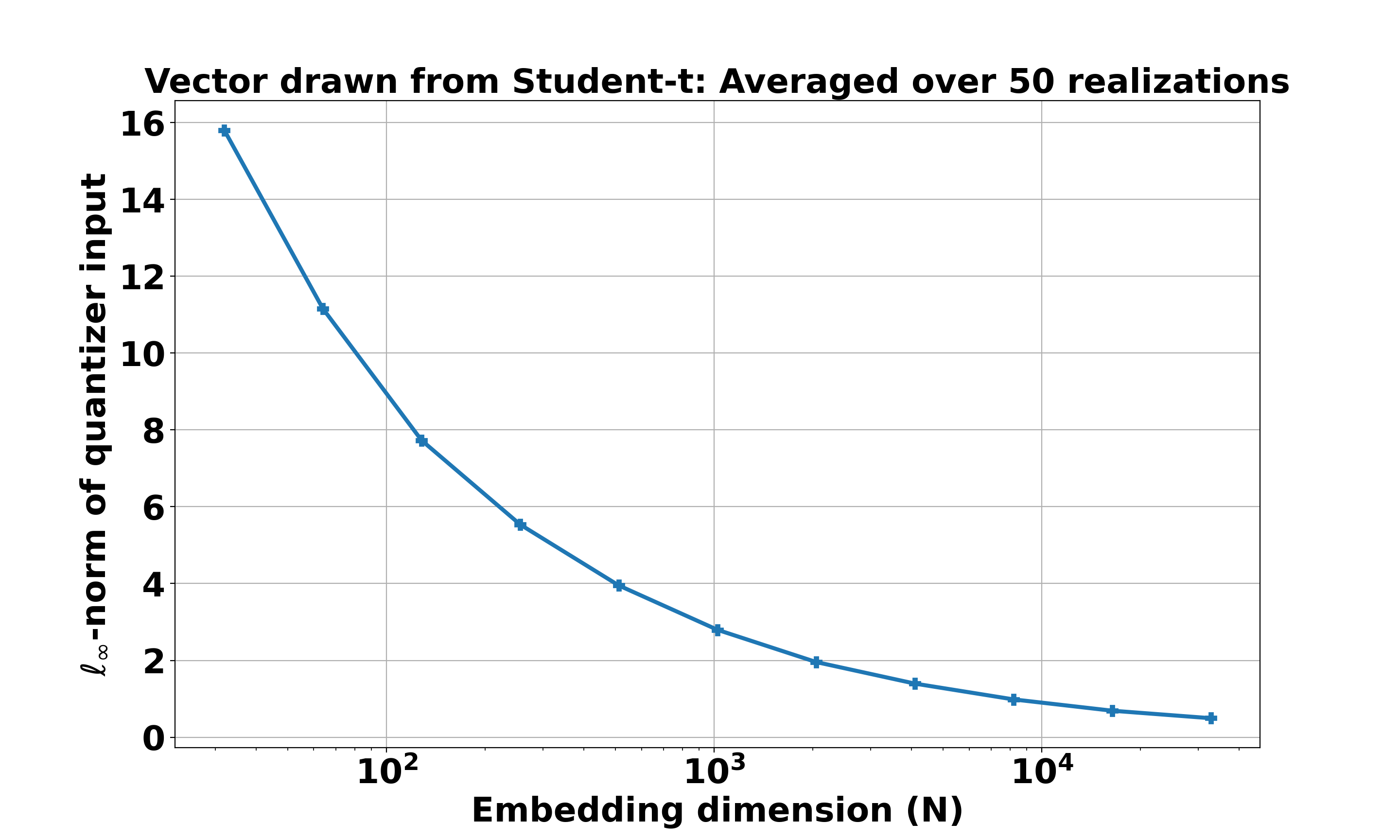}
    \caption{$\lVert \xv_{nd} \rVert_{\infty}$ for $\yv \sim \text{Student-t}$}
    \label{fig:linf_norm_student-t}
  \end{subfigure}
  \vspace{4mm}
\end{figure}

There is no tradeoff so far.
However, when we take into account the pre-specified quantization budget of $R$-bits per dimension, we have a total budget of $nR$ bits to quantize a vector in $\Real^n$.
When we choose to quantize the embedding in $\Real^N$ instead of the original vector using a uniform scalar quantizer, we effectively have $\frac{nR}{N}$-bits per dimension.
For a fixed dynamic range of a scalar quantizer, a lesser number of bits per dimension implies a poorer resolution while quantizing each scalar coordinate.
Moreover, now we are quantizing $N$ coordinates instead of $n$ coordinates, and the error in each of these $N$ coordinates adds up to contribute to the total $\ell_2$-quantization error.
This is an undesirable consequence of increasing $N$.
So much so, this undesirable consequence counteracts the desirable effect of decreasing $\ell_{\infty}$-norm due to increasing $N$.
Since the quantization error per dimension is proportional to $\lVert \text{Quantizer input}\rVert_{\infty}/2^{nR/N}$, we also plot $\lVert \xv_{nd} \rVert_{\infty}\sqrt{N}$ in Figs. \ref{fig:linf_norm_sqrt_N_gaussian3} and \ref{fig:linf_norm_sqrt_N_student-t} which shows that it is more or less constant (with increasing $N$).
That is, the desirable and undesirable effects of increasing $N$ counteract each other.

\begin{figure}[h!]
    \begin{subfigure}[h!]{.5\textwidth}
    \centering
    \includegraphics[width=\linewidth]{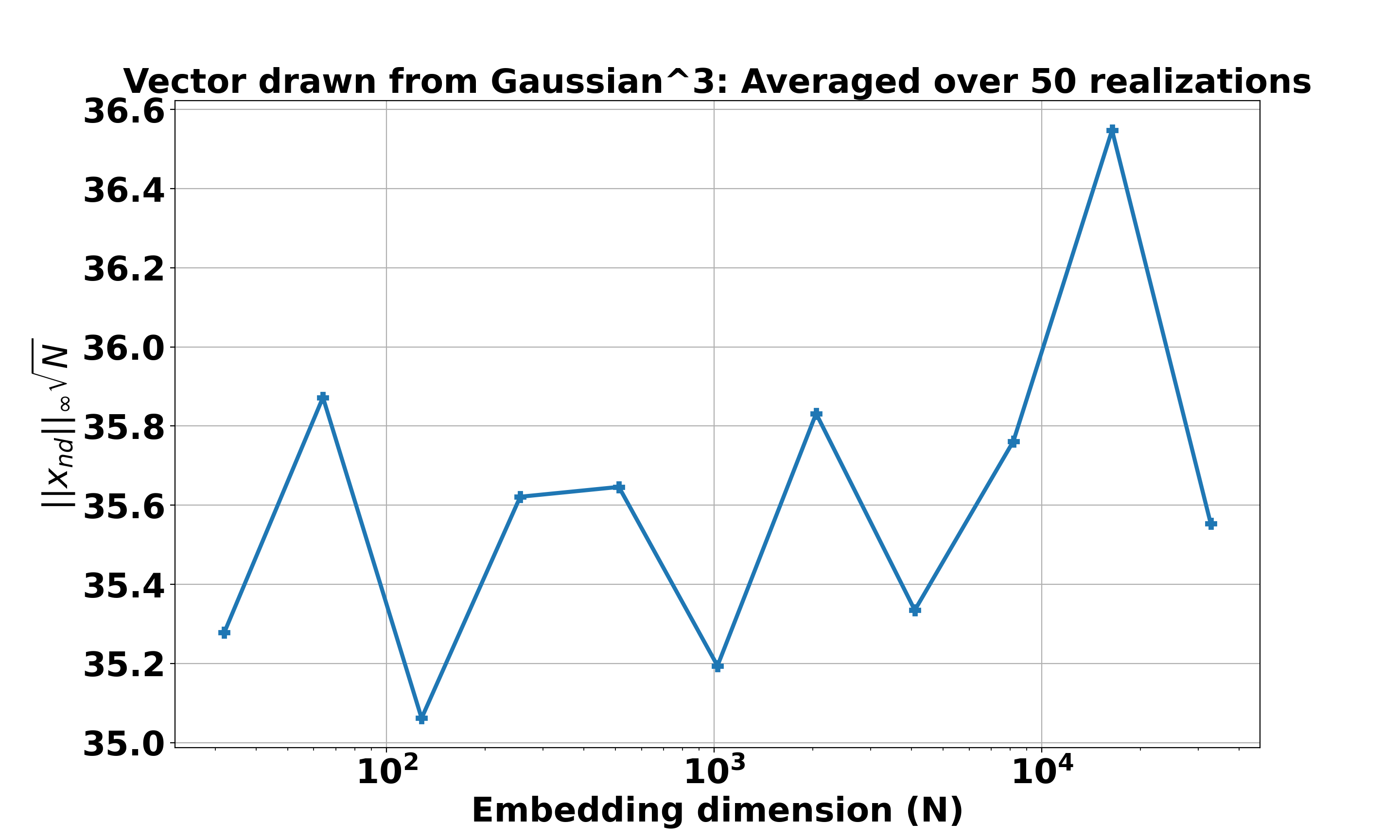}
    \caption{$\lVert \xv_{nd} \rVert_{\infty}\sqrt{N}$ for $\yv \sim \text{Gaussian}^3$}
    \label{fig:linf_norm_sqrt_N_gaussian3}
  \end{subfigure}
  \hfill
  \begin{subfigure}[h!]{.5\textwidth}
    \centering
    \includegraphics[width=\linewidth]{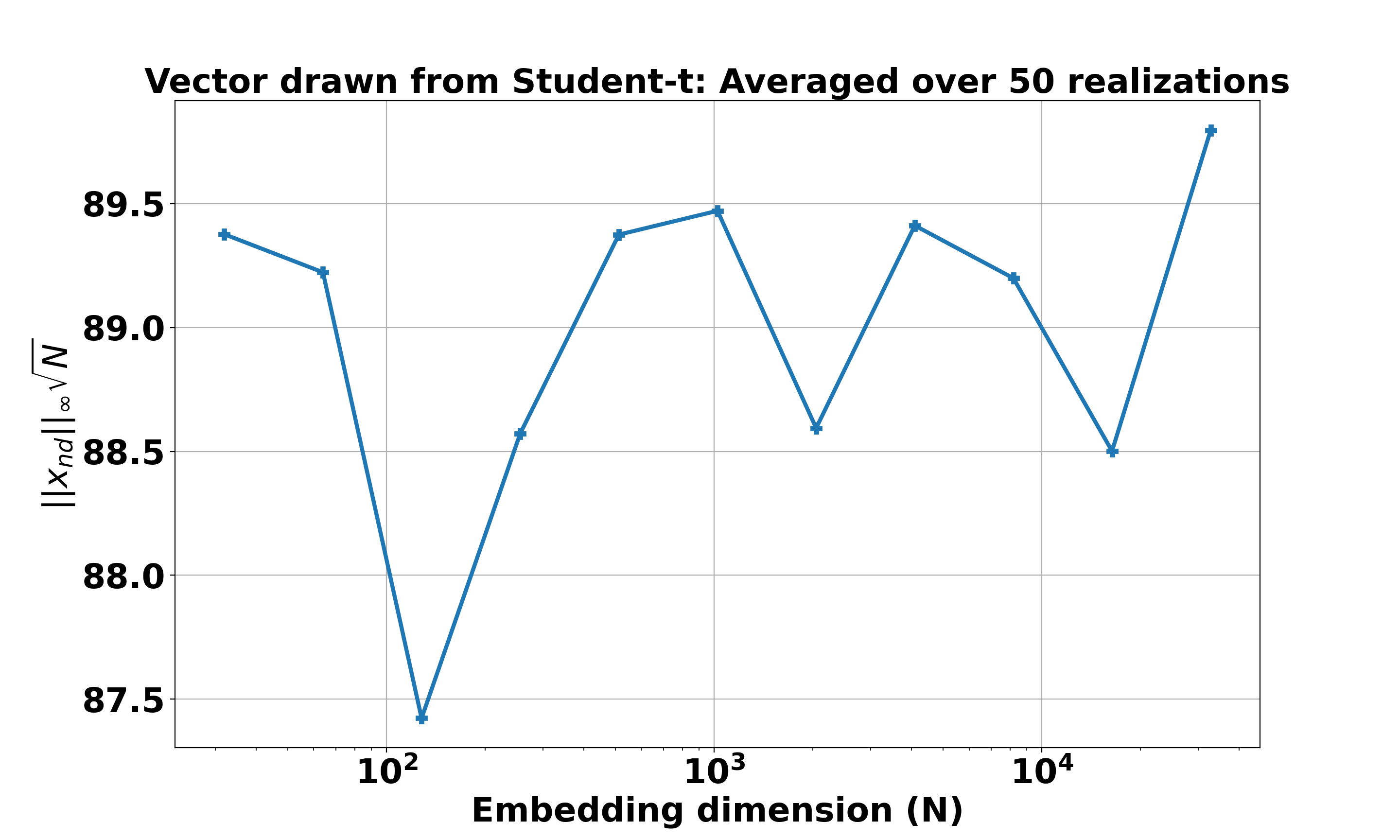}
    \caption{$\lVert \xv_{nd} \rVert_{\infty}\sqrt{N}$ for $\yv \sim \text{Student-t}$}
    \label{fig:linf_norm_sqrt_N_student-t}
  \end{subfigure}
  \vspace{4mm}
\end{figure}

Theoretically, this is made precise in Thm. $1$, i.e.,
\begin{equation}
\label{eq:quantization_error_upper_bound_NDSC}
    \lVert \yv - \mathsf{Q}_{nd}(\yv) \rVert_2 \leq 2^{\Paren{2 - \frac{nR}{N}}}\sqrt{\log(2N)}\lVert \yv \rVert_2,
\end{equation}
which says that we should expect a mild logarithmic increase.
Since the upper bound in \eqref{eq:quantization_error_upper_bound_NDSC} is an increasing function of $N$, this implies we want $N$ to be as small as possible while ensuring $N \geq n$.
So we set $N$ to be equal to the nearest power of 2 greater than $n$, so that $\Sv =\Pv \Dv \Hv$ can be constructed.
For Haar random orthonormal frames, we can choose $N = n$.

For the case of \textit{democratic embeddings}, the situation is slightly different.
When we use random orthonormal frames, from Thm. \textcolor{blue}{$2$} and eq. \eqref{eq:Ku_random_orthonormal} in Sec. \ref{subsec:random_orthonormal_matrices} of the Supplementary material, we can see that the upper Kashin constant $K_u$ satisfies,
\begin{equation}
\label{eq:Ku_random_orthonormal_matrix}
    K_u = \frac{4}{(\lambda - 1)^2\sqrt{c}}\Paren{\frac{5 - \lambda}{4}}\sqrt{\log\Paren{\frac{1}{\lambda - 1}}},
\end{equation}
for some constant $c$, with probability exceeding $1 - 2\exp(-c(\lambda - 1)^2n)$.
Here, $\lambda \triangleq N/n$ and for a fixed value of $n$, an optimal choice of $N$ is equivalent to an optimal choice of $\lambda$.
Note that \eqref{eq:Ku_random_orthonormal_matrix} requires $1 < \lambda < 5$ to retain the theoretical guarantees.
However, we can numerically compute the democratic embedding corresponding to any value of $N$ by solving the $\ell_{\infty}$-norm minimization problem, i.e. \textcolor{blue}{$(5)$} in the main paper.
We do this and plot $\lVert \xv_d \rVert_{\infty}$ of the democratic embedding vs. the embedding dimension $N$.
Figs. \ref{fig:linf_norm_sqrt_N_gaussian3_DSC} and \ref{fig:linf_norm_sqrt_N_student-t_DSC} correspond to $\text{Gaussian}^3$ and $\text{Student-t}$ distributions respectively.
Since (given a fixed $n$) random orthonormal frames $n \times N$ can be constructed for any value of $N$, we let $\lambda = [1.0, 1.1, 1.2, \ldots, 2.0, 2.5, 3.0, 4.0, 5.0, 10, 20, 50]$ and let $N = \lceil \lambda n \rceil$.
\begin{figure}[h!]
    \begin{subfigure}[h!]{.5\textwidth}
    \centering
    \includegraphics[width=\linewidth]{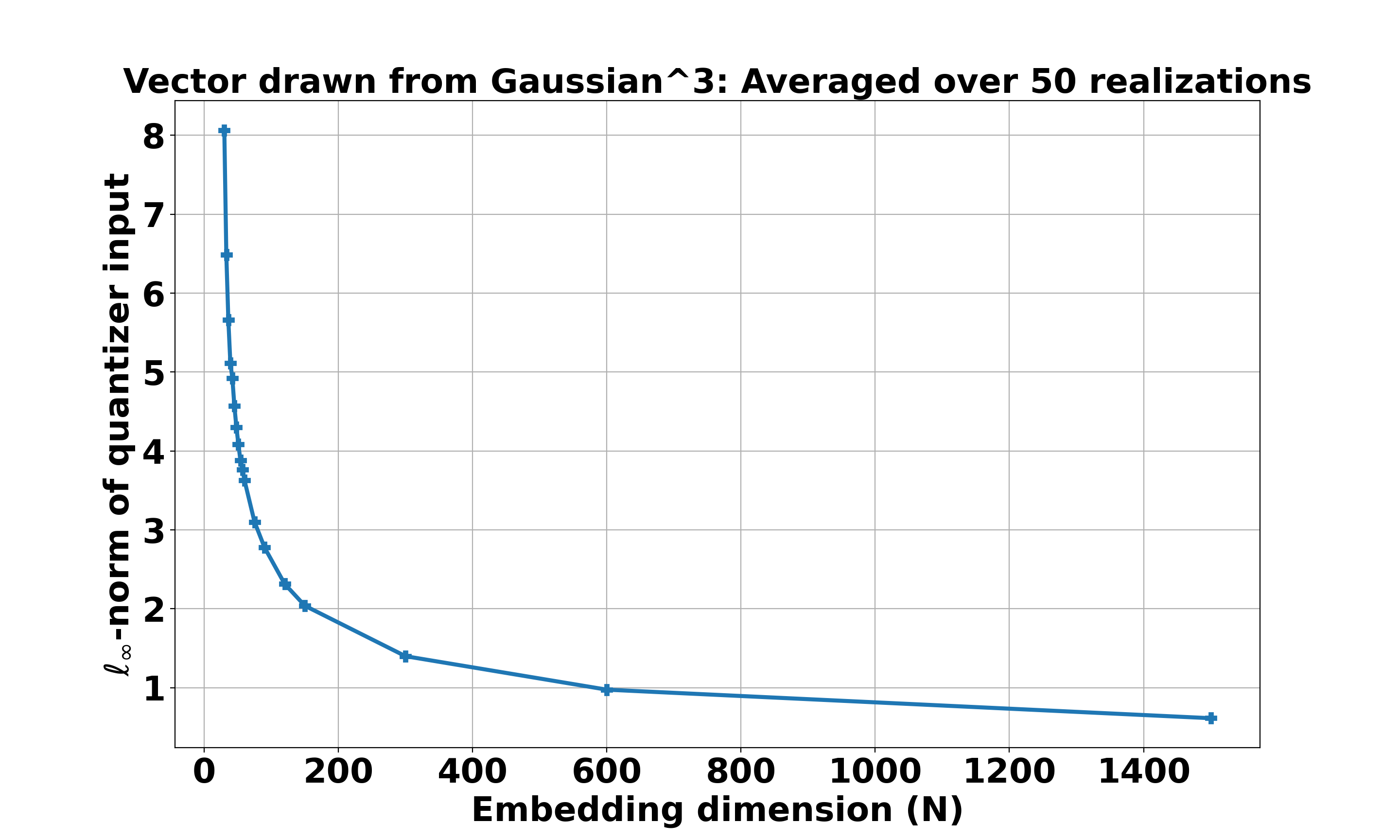}
    \caption{$\lVert \xv_{d} \rVert_{\infty}$ for $\yv \sim \text{Gaussian}^3$}
    \label{fig:linf_norm_gaussian3_DSC}
  \end{subfigure}
  \hfill
  \begin{subfigure}[h!]{.5\textwidth}
    \centering
    \includegraphics[width=\linewidth]{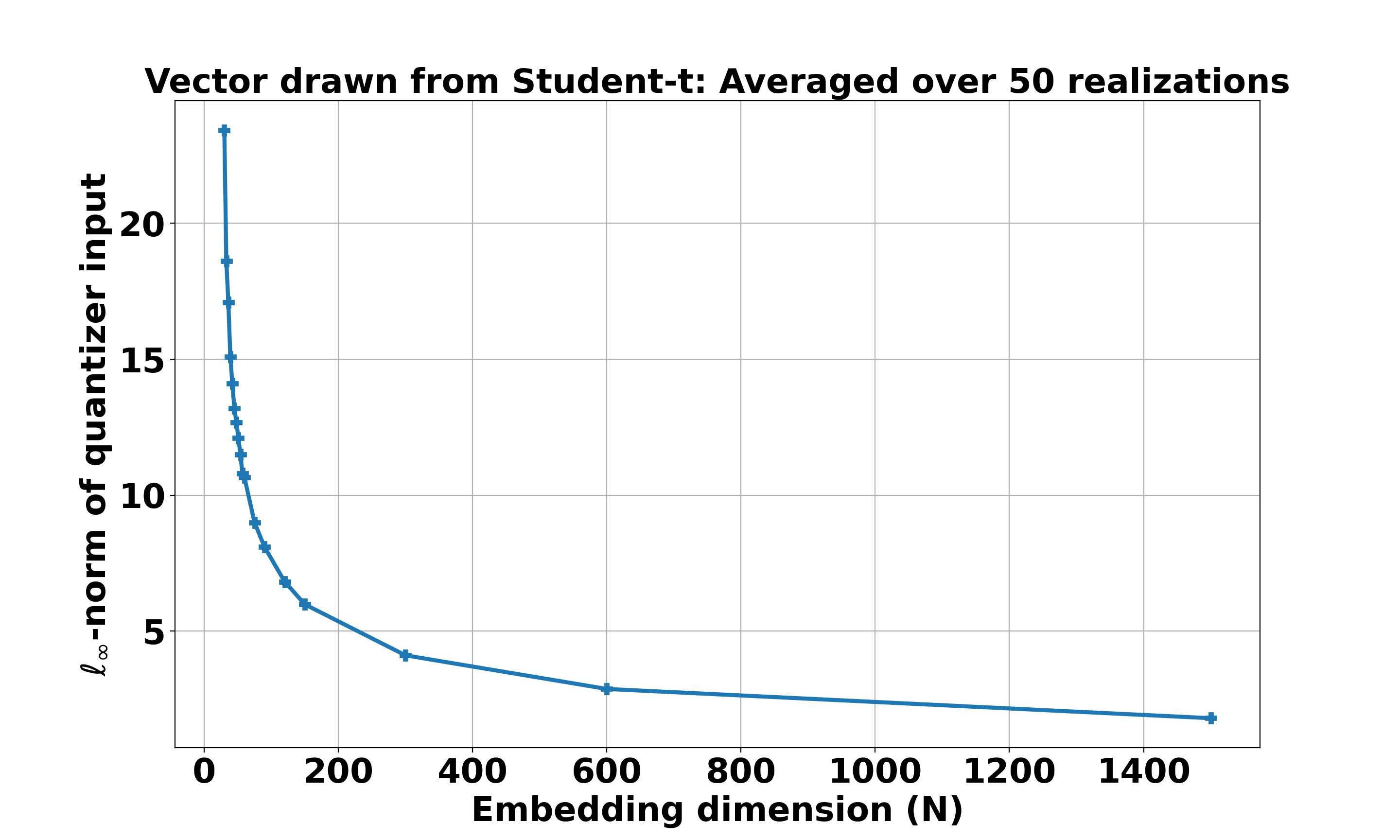}
    \caption{$\lVert \xv_{d} \rVert_{\infty}$ for $\yv \sim \text{Student-t}$}
    \label{fig:linf_norm_student-t_DSC}
  \end{subfigure}
  \vspace{4mm}
\end{figure}

Similar as before, we also plot $\lVert \xv_d \rVert_{\infty}\sqrt{N}$ vs. $N$ in Figs. \ref{fig:linf_norm_sqrt_N_gaussian3_DSC} and \ref{fig:linf_norm_sqrt_N_student-t_DSC}.
Surprisingly, for the case of democratic embeddings, these plots still decrease as $N$ increases.

\begin{figure}[h!]
    \begin{subfigure}[h!]{.5\textwidth}
    \centering
    \includegraphics[width=\linewidth]{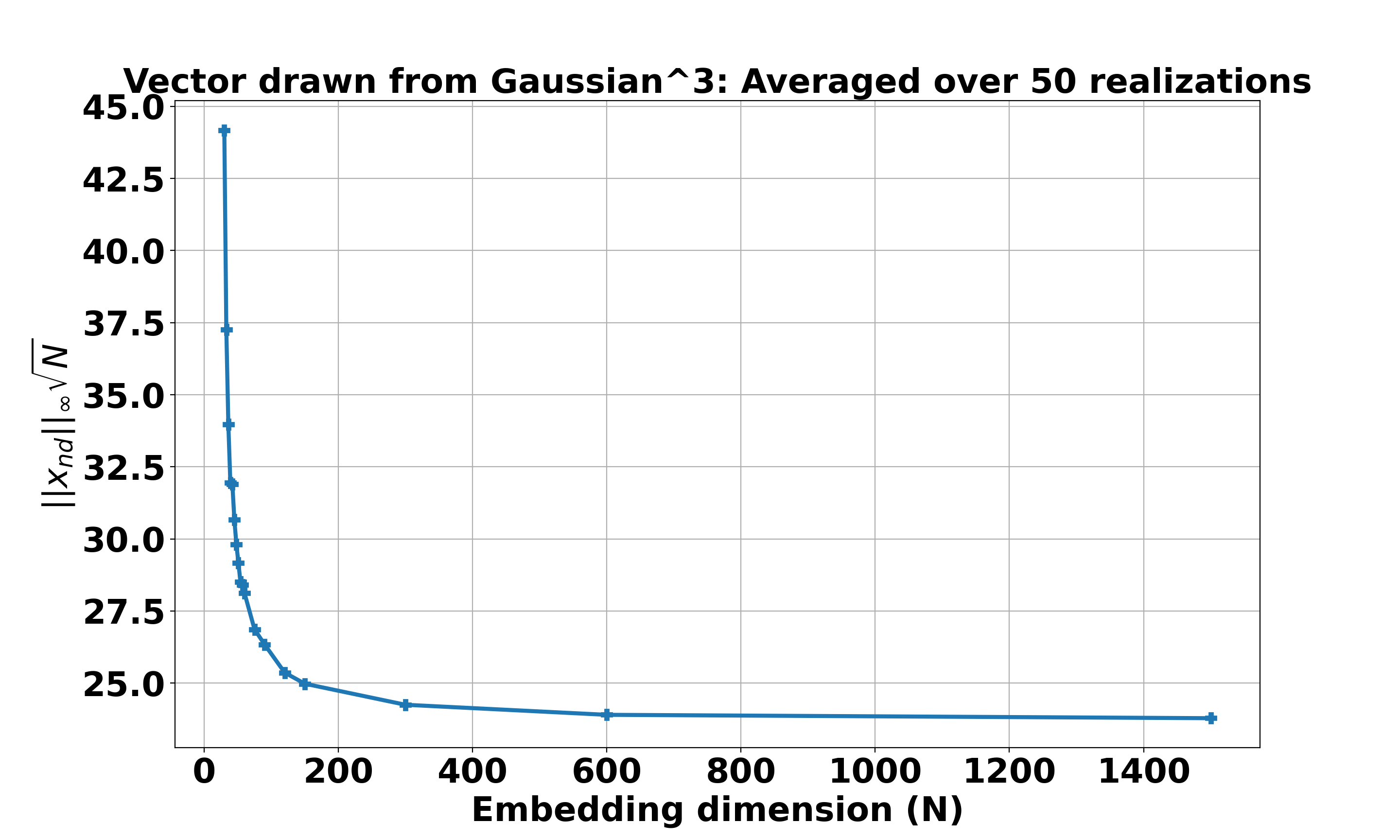}
    \caption{$\lVert \xv_{nd} \rVert_{\infty}\sqrt{N}$ for $\yv \sim \text{Gaussian}^3$}
    \label{fig:linf_norm_sqrt_N_gaussian3_DSC}
  \end{subfigure}
  \hfill
  \begin{subfigure}[h!]{.5\textwidth}
    \centering
    \includegraphics[width=\linewidth]{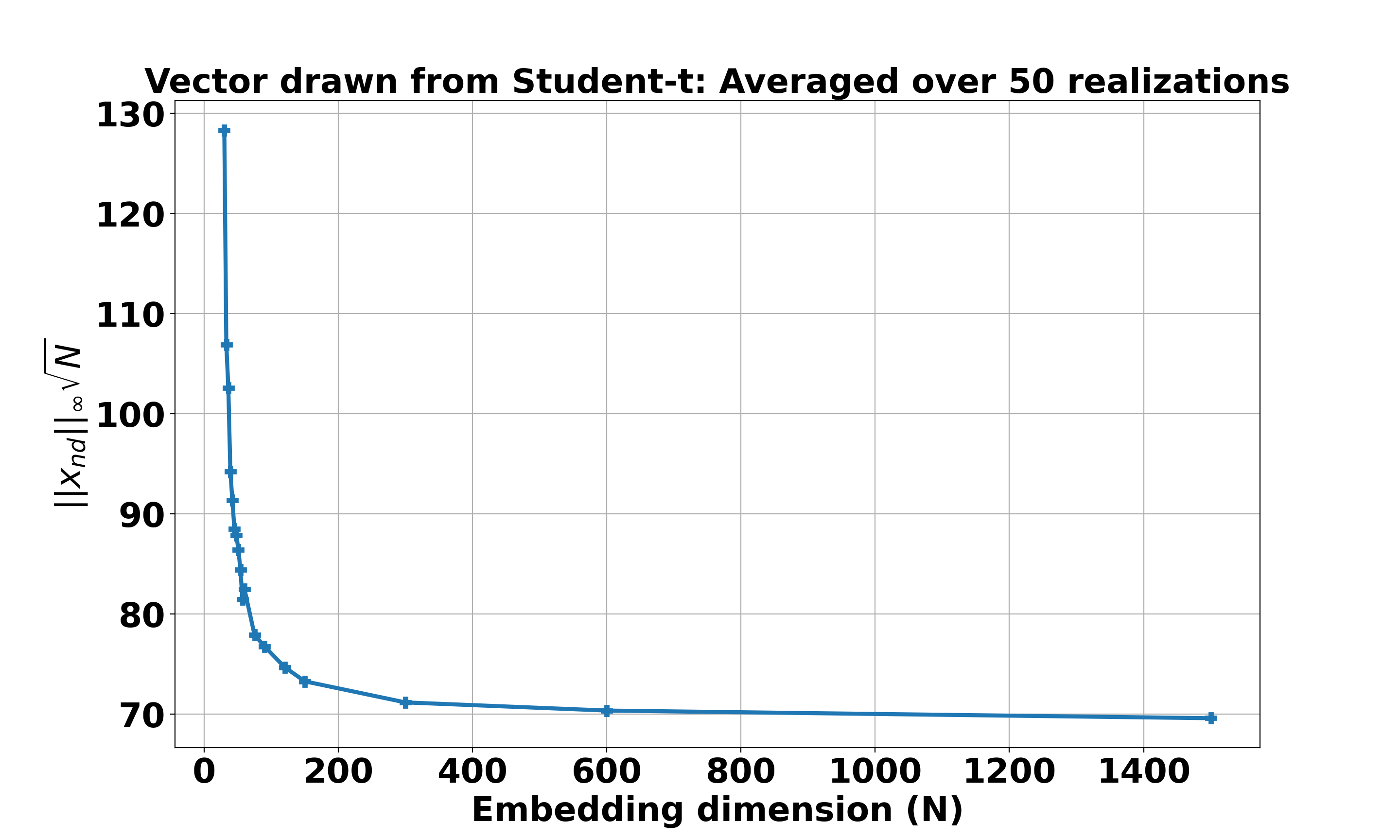}
    \caption{$\lVert \xv_{nd} \rVert_{\infty}\sqrt{N}$ for $\yv \sim \text{Student-t}$}
    \label{fig:linf_norm_sqrt_N_student-t_DSC}
  \end{subfigure}
  \vspace{4mm}
\end{figure}

However, this still does not definitively establish whether increasing $N$ is beneficial for decreasing the quantization error (unlike near-democratic embeddings).
In Figs. \ref{fig:quant_error_DSC_gaussian3} and \ref{fig:quant_error_DSC_student-t}, we plot the $\ell_2$-quantization error vs. the embedding dimension $N$.
We note that the quantization error increases with $N$, implying that the undesirable effect fewer bits per dimension available to quantize a vector in $\Real^N$) overwhelms the reduction in the dynamic range due to decrease in $\ell_{\infty}$-norm of the embedding.
More precisely, the DSC quantization error guarantee in Thm. \textcolor{blue}{$1$} tells us,
\[
\lVert \yv - \mathsf{Q}_{d}(\yv) \rVert_2 \leq 2^{\Paren{1 - \frac{nR}{N}}}K_u\lVert \yv \rVert_2
\]
A loose argument comparing upper bounds tells us that $2^{\Paren{1 - \frac{nR}{N}}}$ increases as $N$ increases, whereas $K_u$ decreases when you increase $N$, i.e. increase $\lambda$, but the undesirable effect of the former dominates.
\begin{figure}[h!]
    \begin{subfigure}[h!]{.5\textwidth}
    \centering
    \includegraphics[width=\linewidth]{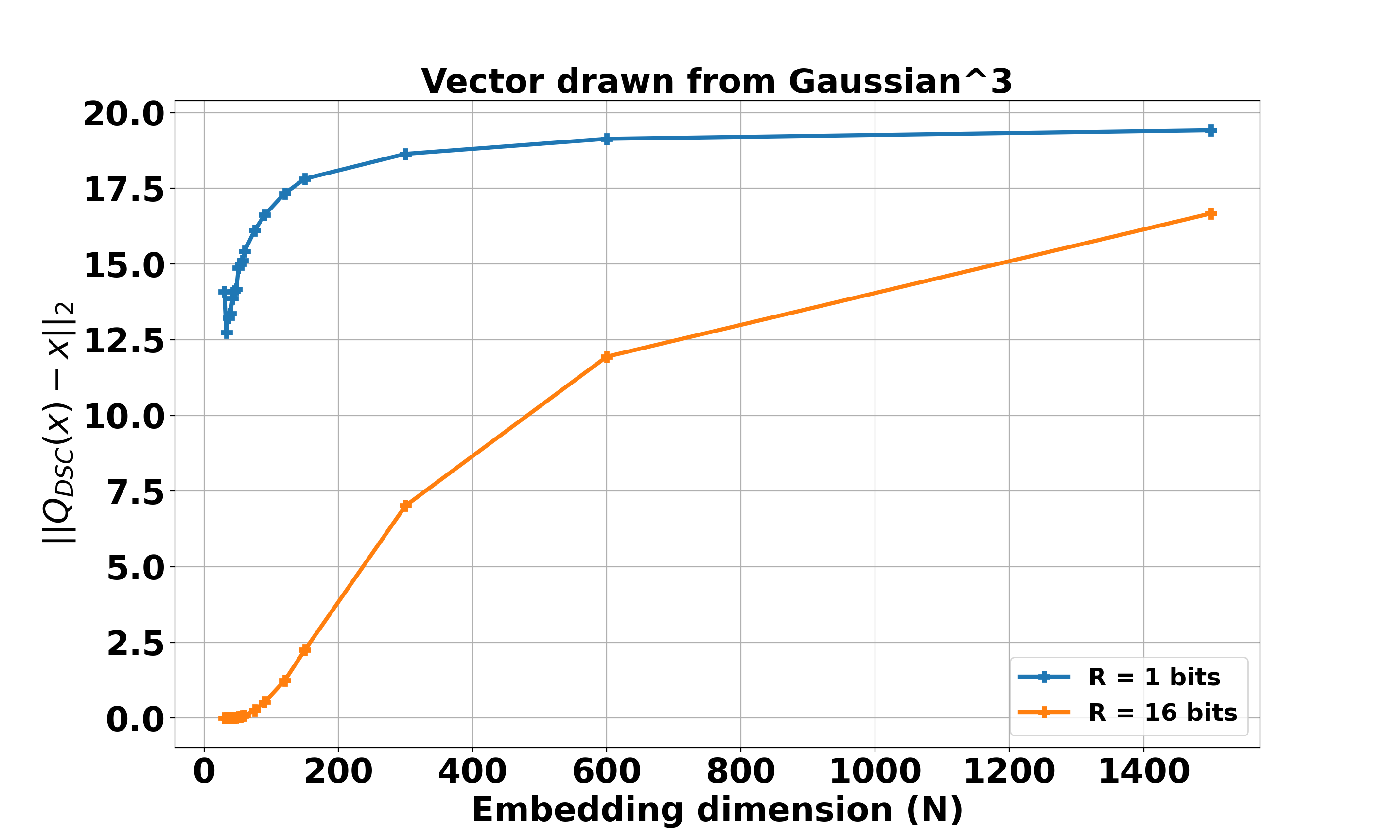}
    \caption{$\lVert \yv - \mathsf{Q}_d(\yv) \rVert_2$ for $\yv \sim \text{Gaussian}^3$}
    \label{fig:quant_error_DSC_gaussian3}
  \end{subfigure}
  \hfill
  \begin{subfigure}[h!]{.5\textwidth}
    \centering
    \includegraphics[width=\linewidth]{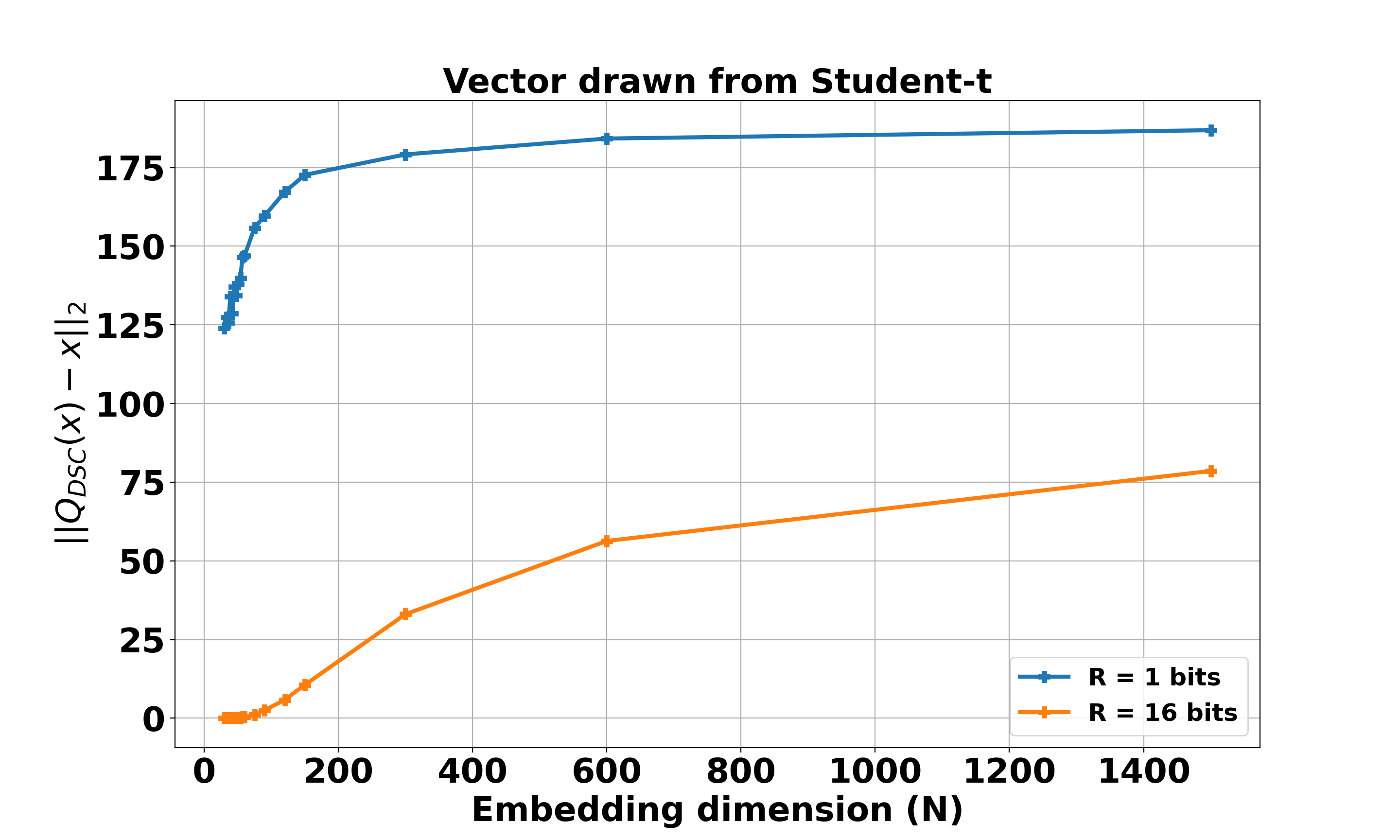}
    \caption{$\lVert \yv - \mathsf{Q}_d(\yv) \rVert_2$ for $\yv \sim \text{Student-t}$}
    \label{fig:quant_error_DSC_student-t}
  \end{subfigure}
  \vspace{4mm}
\end{figure}

Figs. \ref{fig:quant_error_DSC_gaussian3} and \ref{fig:quant_error_DSC_student-t} indicate that for random orthonormal frames,  it is beneficial to choose $N$ as close to $n$ as possible, i.e. $\lambda = \frac{N}{n}$ should be close to $1$.
The situation might be different for different classes of frames.
In any case, $\lambda$ can be considered to be a constant.
$N$ is NOT specified by the optimization problem.
In other words, $K_u = K_u(\lambda)$ can be treated as a constant that depends on the aspect ratio $\lambda = \frac{N}{n}$ (our choice) and is independent of the problem dimension $n$.

\end{document}